\pdfoutput=1
\documentclass[preprint,12pt]{elsarticle}
\makeatletter
\def\ps@pprintTitle{%
 \let\@oddhead\@empty
 \let\@evenhead\@empty
 \def\@oddfoot{\centerline{\thepage}}%
 \let\@evenfoot\@oddfoot}
\makeatother

\usepackage{graphicx,amsmath,amssymb}
\usepackage{amsthm}
\usepackage[mathscr]{eucal}
\usepackage{color,hyperref}
\usepackage[normalem]{ulem}
\usepackage{lineno}
\graphicspath{{figs/}{figs/circle/},{figs/1Ddata/},{figs/clustering/},
{figs/2Dclusters/},{figs/examples/},{figs/lines/},{figs/planes/},
{figs/parabolas/}}

\newtheorem{proposition}{Proposition}
\newtheorem{lemma}{Lemma}
\newtheorem{theorem}{Theorem}
\newtheorem{corollary}{Corollary}

\theoremstyle{definition}
\newtheorem{definition}{Definition}
\newtheorem{remark}{Remark}

\theoremstyle{remark}
\newtheorem{example}{Example}

\newcommand{\dwassone}{\ensuremath{W_1}}

\newcommand{\dwassp}{\ensuremath{W_p}}
\newcommand{\dwassinf}{\ensuremath{W_\infty}}
\newcommand{\real}{\ensuremath{\mathbb{R}}}
\newcommand{\sphere}{\ensuremath{\mathbb{S}}}
\newcommand{\expect}[1]{\ensuremath{\mathbb{E}\left[#1\right]}}
\newcommand{\borel}{\ensuremath{\mathcal{P}}}
\newcommand{\joint}{\ensuremath{\Gamma}}
\newcommand{\inner}[2]{\left\langle #1,#2 \right\rangle}
\newcommand{\volume}{\ensuremath{\nu}}
\newcommand{\area}{\ensuremath{\omega}}

\begin{document}

\begin{frontmatter}



\title{The Shape of Data and Probability Measures}


\author[ad1]{Diego H. D\'{i}az Mart\'{i}nez}
\author[ad2]{Facundo M\'{e}moli}
\author[ad1]{Washington Mio}

\address[ad1]{Department of Mathematics, Florida State University, Tallahassee,
FL 32306-4510 USA}
\address[ad2]{Department of Mathematics, Ohio State University, Columbus, OH
43210-1174 USA}

\begin{abstract} We introduce the notion of multiscale covariance tensor fields (CTF) associated with Euclidean random variables as a gateway to the shape of their distributions. Multiscale CTFs quantify variation of the data about every point in the data landscape at all spatial scales, unlike the usual covariance tensor that only quantifies global variation about the mean. Empirical forms of localized covariance  previously have been used in data analysis and visualization, for example, in local principal component analysis, but we develop a framework for the systematic treatment of theoretical questions and mathematical analysis of computational models. We prove strong stability theorems with respect to the Wasserstein distance between probability measures, obtain consistency results for estimators, as well as bounds on the rate of convergence of empirical CTFs. These results show that CTFs are robust to sampling, noise and outliers. We provide numerous illustrations of how CTFs let us extract shape from data and also apply CTFs to manifold clustering, the problem of categorizing data points according to their noisy membership in a collection of possibly intersecting smooth submanifolds of Euclidean space. We prove that the proposed manifold clustering method is stable and carry out several experiments to illustrate the method. \end{abstract}


\end{frontmatter}


\section{Introduction} 
\label{S:intro}

Probing, analyzing and visualizing the shape of complex data are
challenges that are magnified by the intricate dependence of their
structural properties, as basic as dimensionality, on location
and scale (cf.\,\cite{ljm09}). As such, resolving and integrating the
geometry and topology
of data across scales are problems of foremost importance. In this
paper, we develop the notion of multiscale covariance tensor
fields (CTF) associated with Euclidean random variables and
show that many properties of the shape of their distributions
become accessible through CTFs, which provide
stable representations that can be estimated reliably from data. 

For a random vector $y \in \real^d$, scale dependence is
controlled by a kernel function $K(x,y,\sigma) \geqslant 0$,
where $x,y \in \real^d$ and $\sigma > 0$ is the scale parameter.
The idea is that from the standpoint of $x$, at scale $\sigma > 0$,
the kernel masks the distribution by attributing weight
$K(x,y,\sigma)$ to data located at $y$, creating a windowing
effect. More simply put, $K (x,y,\sigma)$ quantifies how well an
observer at $x$ sees data at $y$ at scale $\sigma$. Covariation
of the weighted data is measured relative to every point
$x \in \real^d$, not just about the mean as is common practice,
thus giving rise to a multiscale covariance field.  Special cases of these
covariance fields were introduced in \cite{mmm13}, targeting
applications to such problems as detection of local scales
and feature rich points in shapes.  Here we present a more systematic
treatment that includes a broader formulation of multiscale
CTFs, stability theorems that ensure that properties of probability
measures derived from multiscale CTFs are robust, as well
as consistency results and convergence rates for empirical
CTFs. We prove stability of CTFs with respect to the Wasserstein
distance between probability measures, a metric that is finding
uses in an ever expanding landscape of problems and whose
origins are in optimal transport theory \cite{villani03,villani09}.
Since Wasserstein distance metrizes weak convergence of
probability measures, we obtain a strong stability result that
ensures that if two probability distributions are similar
in a weak sense, then their multiscale CTFs are uniformly
close over the entire domain. Convergence rates are derived
from the stability theorems and results by Fournier and Guillin
\cite{fournier14} and Garc\'{i}a-Trillos and Slep\v{c}ev \cite{gts15} on
convergence of empirical measures. The standard covariance
tensor of a random vector $y \in \real^d$ quantifies covariation
of $y$ about the mean, but may be extended to a full covariance
field by considering covariation about arbitrary points.
Nonetheless, this field provides no information about the
organization of the data other than that
already contained in the covariance about the mean.
Thus, a localized formulation is essential for gaining additional
insight into the shape of data. 

The trace of a multiscale CTF is a scalar field that gives a
multiscale analogue of the classical Fr\'{e}chet
function $V(x) = \expect{\|y-x\|^2}$ of a random variable $y$
with finite second moment. The Fr\'{e}chet function provides a
more geometric interpretation of the mean as the unique
minimizer of $V$; that is, the point $\mu \in \real^d$ with respect
to which the spread of $y$ is minimal. Similarly, the local extrema
and other properties of the multiscale Fr\'{e}chet function
provide a wealth of information about the distribution of $y$.
In fact, we show that the distribution of any random vector
may be fully recovered from the multiscale Fr\'{e}chet
function associated with the Gaussian kernel.

Several variants of empirical localized or weighted covariance
previously have been used in data analysis, but we develop a framework
for the formulation and systematic treatment of such problems.
Allard et al. have developed a computational model termed geometric
multi-resolution analysis for multiscale data analysis based on covariance
localized to hierarchies of dyadic cubes \cite{allardetal12}. In computer
graphics, local principal component analysis (PCA) is commonly
used in the estimation of normals to surfaces
from point-cloud data \cite{berk-caelli} for surface reconstruction;
see also \cite{rusu09} and references therein. In
computer vision, tensor voting by Medioni et al. \cite{medioni00} has been
applied to multiple image analysis and processing tasks.
Brox et al. have used empirical covariance
weighted by the isotropic Gaussian kernel in non-parametric
density estimation targeting applications in motion tracking
\cite{brcs07}.   In the literature dealing with clustering, especially clustering  of multiple possibly intersecting manifolds, local PCA ideas have been used in the works of Kushnir et al.  \cite{kushnir2006fast}, Goldberg et al. \cite{goldberg2009multi}, Gong et al. \cite{gong2012robust}, Wang et al. \cite{wang2011spectral}, and in a series of papers by Arias-Castro and collaborators \cite{arias2011clustering,arias2011spectral,arias2013spectral}. 

%-------------------------
\begin{figure}[h!]
\begin{center}
\begin{tabular}{ccc}
\begin{tabular}{c}
\includegraphics[width=0.22\linewidth]{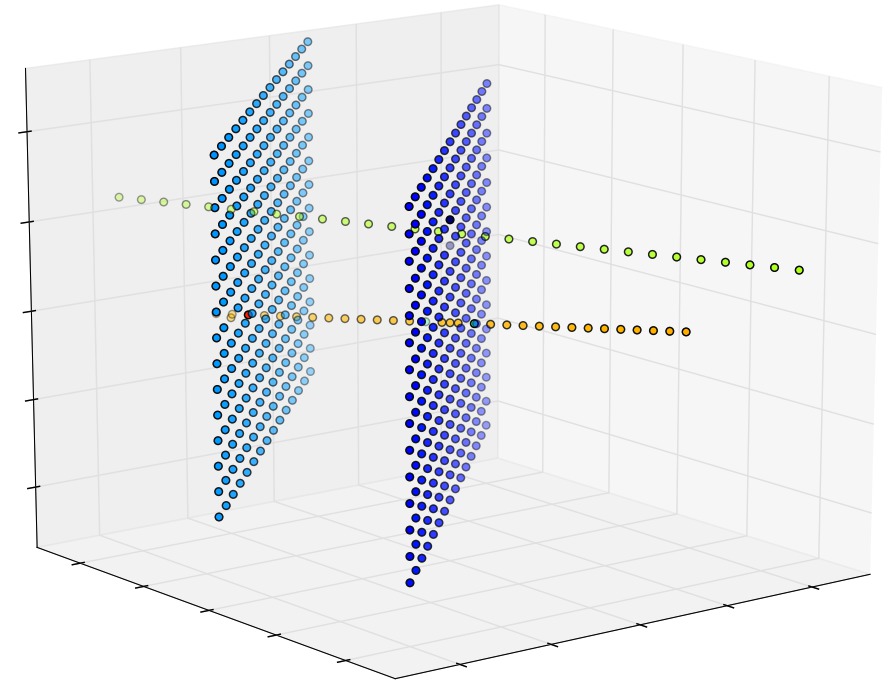} 
\end{tabular}
\quad & \quad
\begin{tabular}{c}
\includegraphics[width=0.2\linewidth]{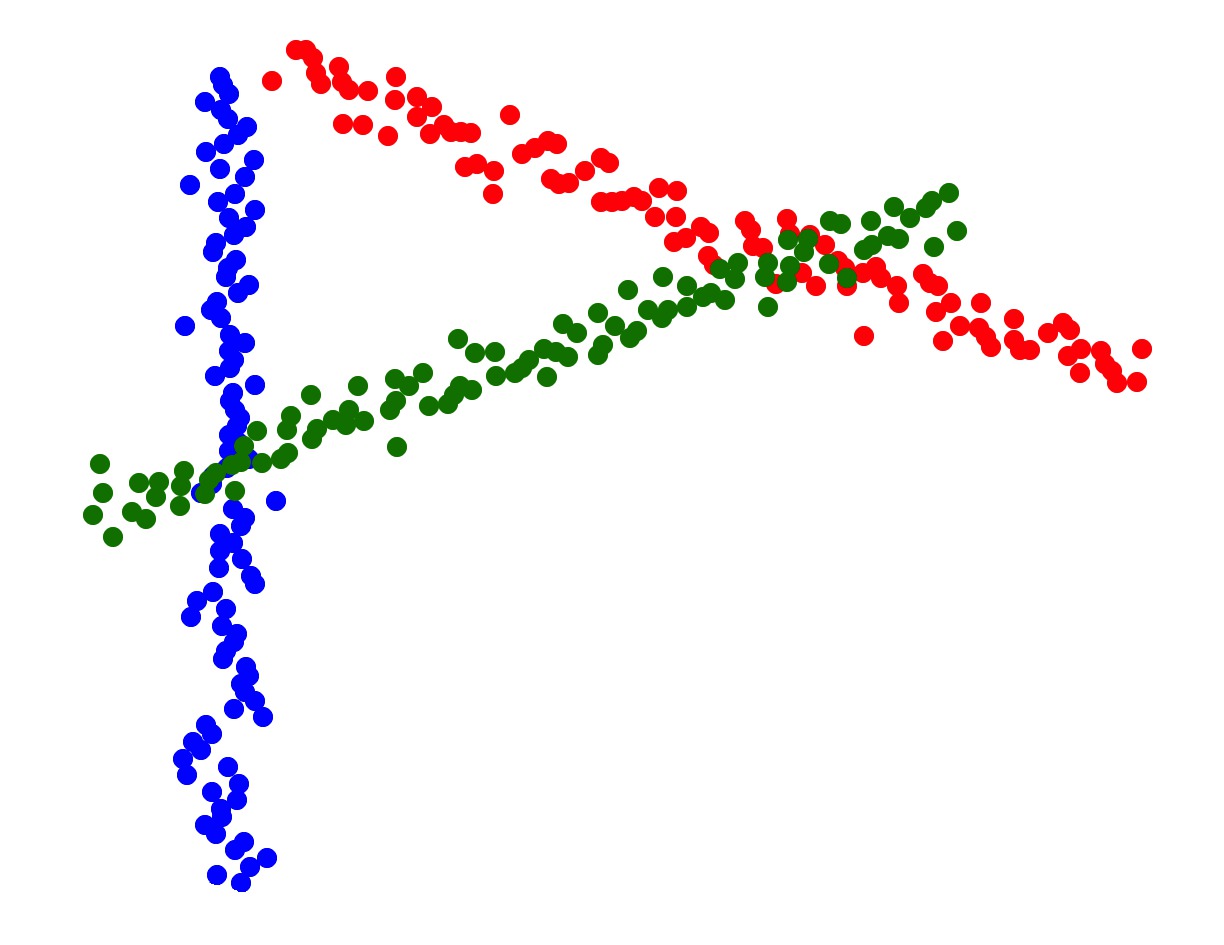}
\end{tabular}
\quad & \quad
\begin{tabular}{c}
\includegraphics[width=0.22\linewidth]{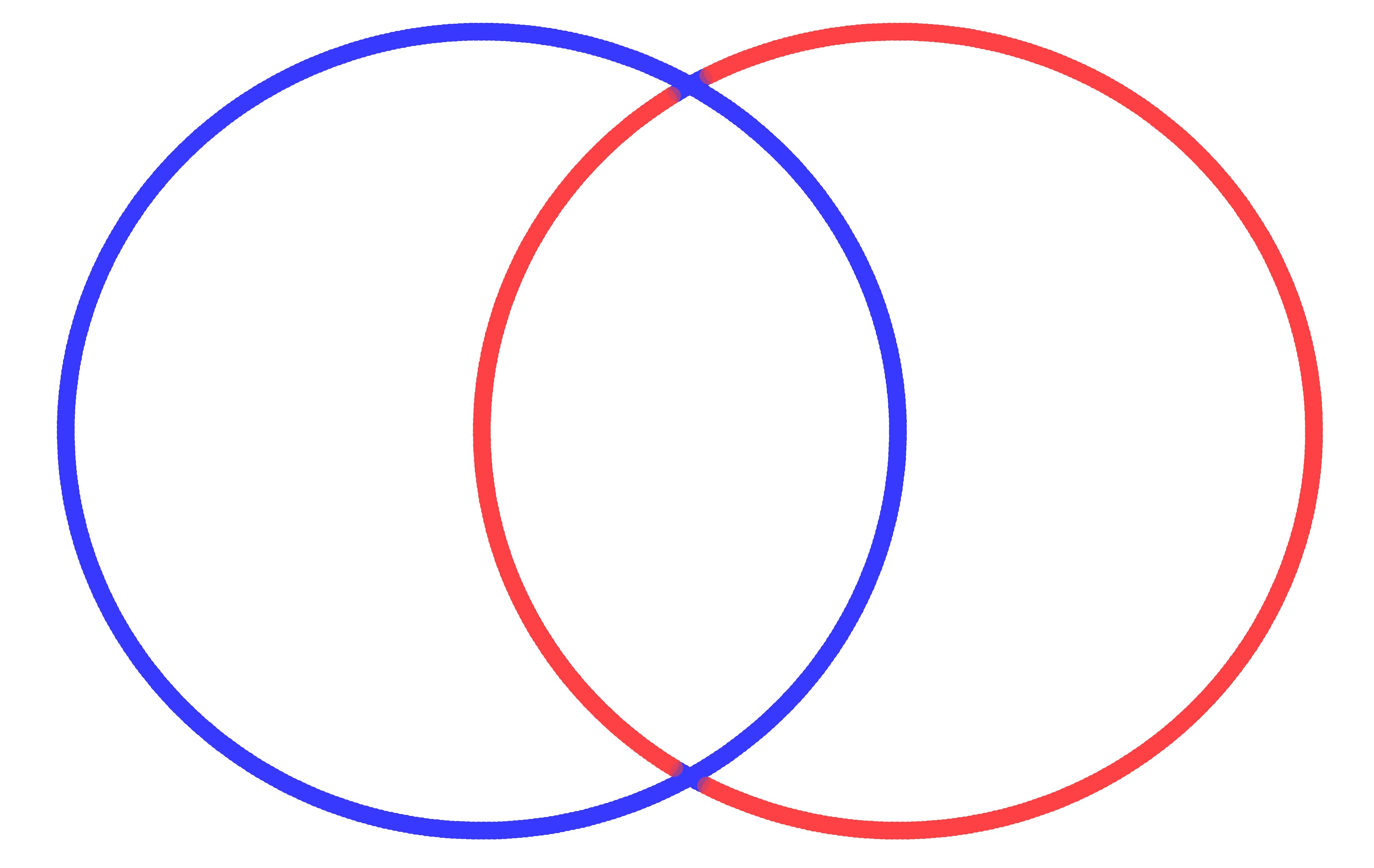} 
\end{tabular} \\
(a) \quad & \quad (b) \quad & \quad (c)
\end{tabular}
\end{center}
\caption{Examples of data clustered along intersecting
manifolds.}
\label{F:clusters}
\end{figure}
%-----------------------
The special case of affine linear subspaces, known as subspace
clustering, has been addressed in the machine learning and
computer vision literature by many authors using a variety of
techniques (cf.\,\cite{vidal05, govindu05, lerman09, lerman11,
lerman12, liu13, elhamifar13,soltanolkotabi2014robust}). More general manifold clustering has
been considered in \cite{polito2001grouping,gionis2005dimension,pless05,kushnir2006fast,haro2006stratification,goldberg2009multi,gong2012robust,wang2011spectral,arias2013spectral,wang2014riemannian}. In our approach, we exploit the fact that
localized covariance tensors encode rich information about the
tangential structure of the submanifolds that underlie the data.
Combined with information about the (relative) positions
of the data points, they yield an effective data representation for
manifold clustering.  Although several different clustering techniques
could be applied to the ``tensorized'' data, we use the single linkage
hierarchical method because it produces provenly stable
dendrograms. In conjunction with the stability and consistency
results for covariance fields, this ensures that the manifold clustering
method is stable at all steps. Dendrogram stability is analyzed
in the framework of \cite{memoli10}. 

%----------------------
%\facundo{Here I suggest introducing a section called ``contributions''}
\paragraph{Contributions and Organization of the paper}

%-------------------------
%-------------------------
The paper includes several
illustrations and applications of
CTFs to data analysis. For example, to illustrate how geometric
information can be extracted from CTFs, we show that the
curvature of plane curves and the principal curvatures of
surfaces in $\real^3$
can be calculated from the spectrum of multiscale CTFs. 
Thus, multiscale covariance tensors give a way of extending
these infinitesimal measures of geometric complexity to
all scales and general probability distributions, not just
those supported on smooth submanifolds. We also apply
multiscale CTFs to manifold clustering, the problem of
clustering Euclidean data that are organized along a finite
union of possibly intersecting smooth
submanifolds. Fig.\,\ref{F:clusters} shows three such
examples.

The main goals of the paper are: (i) to establish the foundations
for analysis, visualization and management of data with
methods based on multiscale covariance tensor fields,
and (ii) to describe applications that characterize the
usefulness of CTFs in data analysis. In Section  \ref{S:covariance},
we formulate the notion of multiscale CTFs for a broad class of
kernels and give examples that illustrate how CTFs reveal
the geometry of data. In Section \ref{S:geometry}, we show that
the curvature of a plane curve and the principal curvatures
of a surface in $\real^3$ can be recovered from small-scale
covariance. Section \ref{S:stability}
is devoted to the main theoretical developments. We prove
stability and consistency theorems for multiscale covariance
tensor fields under mild regularity assumptions on the kernel,
and also analyze rates of convergence that are
important for applications in data analysis. Since some
discontinuous kernels are of practical interest, we also
investigate convergence results for such kernels,
including a pointwise central limit theorem. Multiscale
Fr\'{e}chet functions are discussed in Section \ref{S:frechet}
and manifold clustering in Section \ref{S:clustering}. We
close with a summary and some discussion.

%----------------------

\section{Covariance Tensor Fields} \label{S:covariance}

\subsection{Preliminaries}

To define covariance tensor fields, we introduce some notation.
Elements of the tensor product $\real^d \otimes \real^d$ may be 
identified with bilinear forms $B \colon \real^d \times \real^d \to \real$
through the Euclidean inner product. More precisely,
a pure tensor $x \otimes y$ corresponds to the bilinear form
\begin{equation} \label{E:bilinear}
x \otimes y \, (u,v) = \inner{x}{u} \cdot \inner{y}{v},
\end{equation}
$\forall u, v \in \real^d$, where $\inner{\,}{}$ denotes Euclidean
inner product. Bilinear forms associated with more general
elements of $\real^d \otimes \real^d$ can be
described by linear extension. In Euclidean coordinates, we abuse
notation and also write the coordinate vectors of $x,y \in \real^d$ as
$x$ and $y$. With this convention, letting $A$ be the
$d \times d$ matrix $A = xy^T$, we have
\begin{equation}
x \otimes y \, (u,v) = \inner{u}{A v} ,
\end{equation}
where the superscript $T$ denotes transposition.
In this manner, using Euclidean coordinates, an element of
$\real^d \otimes \real^d$ also can be identified with a $d \times d$
matrix by linear extension of the correspondence
$x \otimes y \leftrightarrow A$. Through these identifications,
we refer to an element $\Sigma \in \real^d \otimes \real^d$
 interchangeably as a tensor, a bilinear form or a matrix. We equip
$\real^d \otimes \real^d$ with the inner product defined on
pure tensors by
\begin{equation} \label{E:inner}
\inner{x_1 \otimes y_1}{x_2 \otimes y_2} =
\inner{x_1}{x_2} \inner{y_1}{y_2}
\end{equation}
and extended linearly to $\real^d \otimes \real^d$. Thus, the
corresponding norm satisfies
\begin{equation} \label{E:norm}
\|x \otimes y\| = \|x\| \|y\| \,,
\end{equation}
for any $x, y \in \real^d$. In matrix representation,
this is the Frobenius norm.

Throughout the paper, we view $\real^d$ as a measurable space
equipped with the Borel $\sigma$-algebra for the
Euclidean metric. Let $y$ be an $\real^d$-valued random
variable distributed according to the probability measure $\alpha$.
Suppose that $y$ has expected value $\expect{y} = \mu \in \real^d$
and finite second moment. As a motivation for the definition of 
multiscale CTFs, recall that the covariance
tensor of $y$ is defined as
\begin{equation}
\Sigma_\alpha (\mu) = \expect{(y-\mu) \otimes (y-\mu)} =
\int_{\real^d} (y-\mu) \otimes (y-\mu) \, \alpha (dy)
\in \real^d \otimes \real^d \,.
\end{equation}
In matrix notation,
\begin{equation}
\Sigma_\alpha (\mu) = \int_{\real^d} (y-\mu) (y-\mu)^T \, \alpha (dy) \,.
\end{equation}
The bilinear form associated with $\Sigma_\alpha (\mu)$
clearly is symmetric and positive semi-definite.  

Covariation of $y$ may be measured with respect
to any $x \in \real^d$, not just $\mu$. Thus,
$\Sigma_\alpha (\mu)$ may be extended to a global
covariance tensor field
$\Sigma_\alpha \colon \real^d \to \real^d \otimes \real^d$
given by
\begin{equation} \label{E:field}
\Sigma_\alpha (x) = \int_{\real^d} (y-x) \otimes (y-x) \, \alpha (dy) \,. 
\end{equation}
Note, however, that
\begin{equation}
\Sigma_\alpha (x) = \Sigma_\alpha (\mu) + (\mu-x) \otimes (\mu-x) \,,
\end{equation}
for any $x \in \real^d$. Thus, for $x \ne \mu$, $\Sigma_\alpha (x)$
does not reveal any information about the distribution of $y$
other than that already contained in $\Sigma_\alpha (\mu)$. 
In contrast, as we shall see below, multiscale
analogues are rich in information about the
shape of $\alpha$.

%----------------------------

\subsection{Multiscale Covariance Tensor Fields}

We adopt the notation $\volume_d$ for the volume of the unit ball
in $\real^d$ and $\area_{d-1}$ for the ``surface area'' of the unit
sphere $\sphere^{d-1} \subset \real^d$, $d \geq 1$. Recall that 
$\area_{d-1} = 2 \pi^{d/2} / \Gamma(d/2)$, where $\Gamma(\cdot)$ is
the Gamma function, and $\area_{d-1} = d\, \volume_d$.
We make the convention that $\volume_0 = 1$.

Let $y$ be an $\real^d$-valued random variable with distribution
$\alpha$ and let $K$ be a multiscale kernel; that is, a measurable
function $K: \real^d \times \real^d \times (0, \infty) \to \real$
such that $K(x,y,\sigma) \geqslant 0$, for any $x,y \in \real^d$
and $\sigma > 0$.
%---------
\begin{definition} \label{D:ctf}
The {\em multiscale covariance tensor field} (CTF) of $y$ associated
with the kernel $K$ is the one-parameter family of tensor fields,
indexed by $\sigma \in (0, \infty)$, given by
\begin{equation} \label{E:ctf}
\Sigma_\alpha (x, \sigma) := 
\int_{\real^d} (y-x) \otimes (y-x) K(x,y,\sigma)\, \alpha (dy) \,,
\end{equation}
provided that the integral converges for each
$x \in \real^d$ and $\sigma > 0$.
\end{definition}
%-----------------
\begin{remark}
Note that $\Sigma_\alpha$ depends only on the probability measure
$\alpha$, not on $y$. For this reason, we refer to $\Sigma_\alpha$
interchangeably as the multiscale CTF of the random variable $y$
or the probability measure $\alpha$. 
\end{remark}
%-----------------
$\Sigma_\alpha (x, \sigma)$ measures the covariation of $y$
about $x$ with probability mass at $y$ weighted by $K (x,y, \sigma)$.
It is simple to verify that the bilinear form
$\Sigma_\alpha (x, \sigma)$ is symmetric and positive semi-definite.
Note that if $K$ is bounded for each $\sigma>0$, that is,
$\exists M_\sigma > 0$ such that $K(x,y,\sigma) \leq M_\sigma$,
$\forall x,y \in \real^d$, then $\Sigma_\alpha (x, \sigma)$ is
well defined for any random variable $y$ with finite second
moment. In particular if $K \equiv 1$, 
$\Sigma_\alpha (x, \sigma) = \Sigma_\alpha (x)$, $\forall x \in \real^d$.
However, as our primary goal is to study the organization of
data and random variables at scales ranging from local
to global, we consider kernels in $\real^d$ that satisfy additional
decay conditions as they produce a windowing effect. The
kernels are constructed as follows. 

\begin{definition} \label{D:kernel}
Let $d$ be a positive integer and $f \colon [0, \infty) \to \real$
a bounded and measurable function satisfying:
\begin{itemize}
\item[(a)] $f (r) \geqslant 0$, $\forall r \in [0, \infty)$;
\item[(b)]  $M_d = \int_0^\infty r^{\frac{d}{2}-1} f(r) \,dr < \infty$;
\item[(c)] There is $C > 0$ such that $r f(r) \leq C$,
$\forall r \in [0, \infty)$.
\end{itemize}
The multiscale kernel $K \colon \real^d \times \real^d \times (0, \infty) \to \real$
associated with $f$ is defined as
\begin{equation} \label{E:kernel}
K(x,y,\sigma) := \frac{1}{C_d (\sigma)} \,
f \left(\frac{\Vert y-x\Vert^2}{\sigma^2}\right) ,
\end{equation}
where $C_d (\sigma) = \frac{1}{2} \sigma^d M_d \,\area_{d-1}$. 
\end{definition}

Condition (b) in the definition implies that the normalizing
constant $C_d (\sigma)$ is well defined. The normalization
is adopted so that $\int K(x,y,\sigma) \, dy = 1$,
$\forall x \in \real^d$ and $\forall \sigma > 0$.
Condition (c) guarantees that the integral in \eqref{E:ctf}
is convergent for any probability measure $\alpha$. Henceforth,
for convenience, we assume that $\sup f =1$. This is not restrictive
since scaling $f$ does not change the kernel $K$ because of
the normalization.

Whereas we investigate properties of multiscale CTFs in a
more general setting, our examples and
experiments focus on two special kernels:
\begin{itemize}
\item[(i)] The isotropic Gaussian kernel
\begin{equation}
G(x,y, \sigma) = \frac{1}{(2 \pi \sigma^2)^{d/2}} 
\exp \left( - \frac{\|y-x\|^2}{2 \sigma^2} \right) ,
\end{equation}
which is associated with the function $f(x) = e^{-x/2}$;
\item[(ii)] The truncation kernel
\begin{equation}
T(x,y,\sigma) = \frac{1}{\sigma^d \volume_d} \,
\chi \left(\frac{\|y-x\|^2}{\sigma^2} \right)
\end{equation}
associated with the characteristic function
$\chi \colon [0, \infty) \to \real$  of the unit interval $\left[ 0,1\right]$. 
In measuring covariation of random variables about $x$, the
kernel $T$ attributes a uniform weight to mass at points
within the closed ball of radius $\sigma$ centered at $x$
and weight zero to mass elsewhere.
\end{itemize}

%----------------------------

%\begin{align*}
%\frac{1}{C_d(\sigma)}&=\int_{\mathbb{R}^d}f\left(\frac{\Vert x\Vert^2}{\sigma^2}\right)dx\\
%&\text{Using a change of variables }\Vert y\Vert=\frac{\Vert x \Vert}{\sigma}\\
%&=\sigma^d\int_{\mathbb{R}^d}f\left(\Vert y\Vert ^2 \right)dy\\
%&\text{Changing the coordinate system to spherical coordinates }y=(r,v)\text{ where }v\in \partial S^d\\
%&=\sigma^d\int_{\partial S^d}\int_0^{\infty}f(r^2)r^{d-1}drdv=\sigma^d\int_{\partial S^d}dv\int_0^{\infty}r^{d-1}f(r^2)dr\\
%&=\sigma^d\frac{2\pi^{d/2}}{\Gamma(d/2)}\int_0^{\infty}r^{d-1}f(r^2)dr=\sigma^d\frac{\pi^{d/2}}{\Gamma(d/2)}\int_0^{\infty}r^{\frac{d}{2}-1}f(r)dr\\
%\frac{1}{C_d(\sigma)}&=\sigma^d\frac{\pi^{d/2}}{\Gamma(d/2)}M_{\frac{d}{2}-1}
%\end{align*}
%
%As observed in the last equality, $\int_0^{\infty}r^{\frac{d}{2}-1}f(r)dr$ must be finite, hence, the function $f$ must have $\frac{d}{2}-1$ finite moments. Also, if $f$ is a function of compact support, property 3 can be dropped.  It is important to note certain properties of the kernel function that come with its definition:

\begin{remark} \label{R:iso}
The kernel $K$ defined in \eqref{E:kernel} is
homogeneous and isotropic;
that is, for any isometry $\varphi \colon \real^d \to \real^d$,
$K (\varphi (x), \varphi (y), \sigma) = K (x,y,\sigma)$,
$\forall x,y \in \real^d$ and $\sigma > 0$. Moreover,
if we write $\varphi (x) = U x + b$, with $U \in O(d)$ and
$b \in \real^d$, then 
\begin{equation}
U \, \Sigma_{\alpha}(x,\sigma) \, U^T =
\Sigma_{\varphi_\ast (\alpha)}(\varphi(x),\sigma),
\end{equation}
for any $(x, \sigma) \in \real^d \times (0, \infty)$. Here
$O(d)$ is the group of $d \times d$ orthogonal matrices and
$\varphi_\ast (\alpha)$ is the pushforward of $\alpha$ under
$\varphi$. 
\end{remark}

\begin{remark} \label{R:support}
Multiscale covariance tensor fields can be defined for
any positive Borel measure $\alpha$ that satisfies
\begin{equation}
\int_{\real^d} \|z\|^2 f (\|z\|^2) \, \alpha(dz) < \infty \,,
\end{equation}
not just for probability measures. In particular, if $f$ has
compact support, covariance fields
are defined for any locally finite Borel measure
$\alpha$; that is, measures for which every point $p \in \real^d$
has an open neighborhood $U_p$ such that $\alpha (U_p) < \infty$.
\end{remark}

We conclude this section with examples that support
our contention that multiscale covariance tensor fields are
rich in information about the shape of data. 

%\begin{example}[A line in $\real^d$]
%
%For a unit vector $v \in \real^d$, consider the line
%$L_v = \{ t v \,|\, t \in \real\}$. Let $\alpha$ be the singular
%measure supported on $L_v$ induced by arc length. We
%calculate multiscale covariance fields at points $x \in L_v$
%to show that the line may be recovered from
%$\Sigma_\alpha (x, \sigma)$. By Remark \ref{R:iso},
%we may assume that $x = 0$. For the Gaussian kernel
%\begin{equation}
%\Sigma_\alpha (0, \sigma) =
%\frac{1}{(\sqrt{2 \pi})^{d-1} \sigma^{d-3}}\, v \otimes v 
%\end{equation}
%and for the truncation kernel
%\begin{equation}
%\Sigma_\alpha (0, \sigma) = \frac{1}{\sigma^d \volume_d}
%\frac{2 \sigma^3}{3} \, v \otimes v = \frac{2}{3}
%\frac{1}{\sigma^{d-3} \volume_d} \, v \otimes v\,.
%\end{equation}
%In both cases, $\Sigma_\alpha (0, \sigma)$ is a rank-one
%tensor and $L_v$ is the eigenspace associated with
%the non-trivial eigenvalue of $\Sigma_\alpha (0, \sigma)$.
%
%\end{example}

\begin{example} \label{E:dimensionality}
This example shows that the spectrum of multiscale
covariance tensors allow us to estimate the dimensionality of data
in a scale dependent manner. We consider the data points
$y_1, \ldots, y_n$ in $\real^2$, shown in Figure \ref{F:dimension},
and calculate $\Sigma_{\alpha_n}$ centered at one of the data
points for the Gaussian kernel at scales $\sigma = 0.1$ and
$\sigma = 2$. Here $\alpha_n$ denotes the empirical measure
$n^{-1} \sum_{i=1}^n \delta_{y_i}$. The covariance tensors are
depicted as ellipses whose principal axes are in the direction of
the eigenvectors of the covariance matrix and principal radii
are proportional to $\sqrt{\lambda_1}$ and $\sqrt{\lambda_2}$,
where $0 \leq \lambda_1 \leq \lambda_2$ are the eigenvalues of
the covariance. At scale $\sigma = 0.1$,
$\lambda_1/\lambda_2 = 0.908$, showing that the covariance
tensor is nearly isotropic, indicating that the ``dimension'' of the
data is 2. At $\sigma = 2$, the ratio of the eigenvalues is
$0.025$, giving a highly anisotropic covariance tensor, from
which we infer that the dimension is 1.
%-------------------------
\begin{figure}[h!]
\begin{center}
\begin{tabular}{cc}
\begin{tabular}{c}
\includegraphics[width=0.42\linewidth]{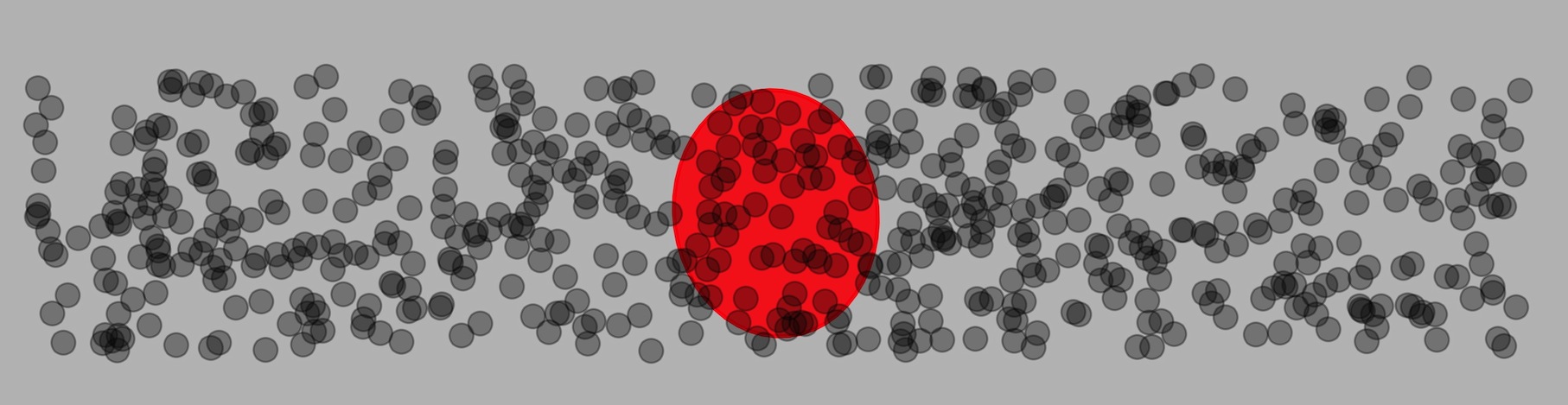} 
\end{tabular}
\qquad  & \qquad
\begin{tabular}{c}
\includegraphics[width=0.4\linewidth]{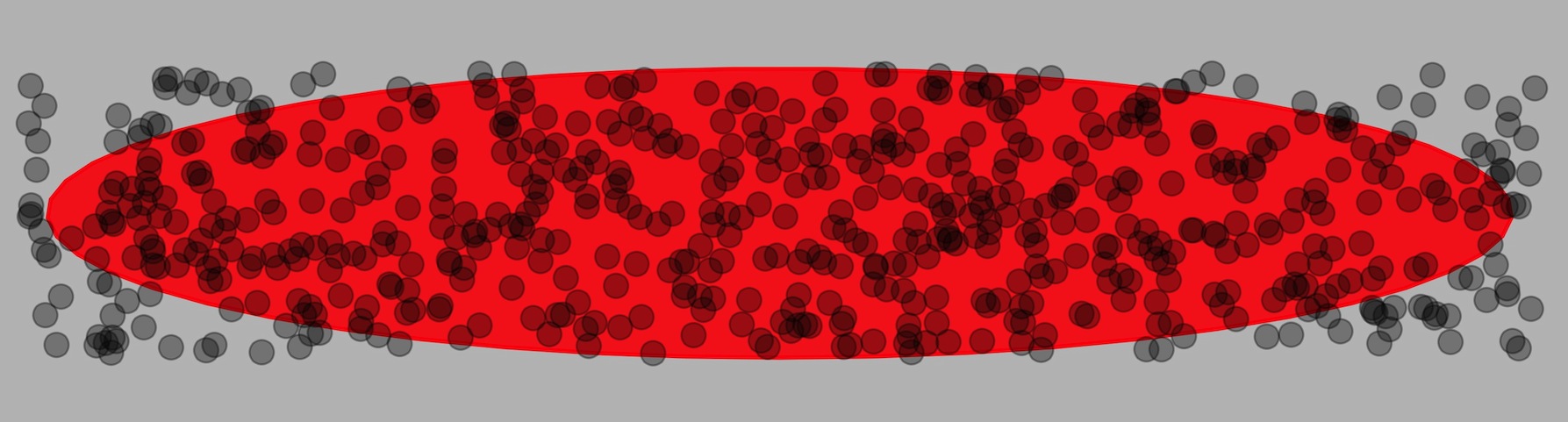}
\end{tabular} \\
$\sigma = 0.1$ \qquad  & \qquad $\sigma = 2$
\end{tabular}
\end{center}
\caption{Estimating data dimensionality at different scales
through multiscale covariance.}
\label{F:dimension}
\end{figure}

\end{example}

\begin{example}[A linear subspace of $\real^d$]
\label{E:subspace}

Let $v_1, \ldots, v_r \in \real^d$, $1 \leq r \leq d$,
be orthonormal vectors and consider the subspace
$H = <v_1, \ldots, v_r >$ that they span. Let $\alpha$
denote the singular measure supported on $H$ induced
by the volume form on $H$. The measure $\alpha$ clearly
is locally finite. We calculate multiscale
covariance fields at points $x \in H$ to show that $H$
may be recovered from $\Sigma_\alpha (x, \sigma)$.
By Remark \ref{R:iso}, we may assume that $x = 0$.
A calculation shows that for the  Gaussian kernel,
\begin{equation}
\Sigma_\alpha (0, \sigma) =
\frac{r}{(\sqrt{2\pi})^{d-r }\sigma^{d-r-2}} \sum_{i=1}^r \,
v_i \otimes v_i \,.
\end{equation}
For the truncation kernel, 
\begin{equation}
\Sigma_\alpha (0, \sigma) = \lambda_r \sum_{i=1}^r \,
v_i \otimes v_i \,,
\end{equation}
where
\begin{equation}
\lambda_r = \frac{1}{\sigma^{d-r-2}}
\frac{\volume_{r-1}}{\volume_d} \int_{-\pi/2}^{\pi/2}
\sin^2 \theta \cos^r \theta \, d\theta\,.
\end{equation}
For $r=1$, this expression simplifies to
$\lambda_1 = 2/(3 \sigma^{d-3} \volume_d)$.
Thus, for both kernels, the orthogonal complement of $H$ is
the null space of $\Sigma_\alpha (0, \sigma)$ and $H$
is the eigenspace associated with the positive
eigenvalue $\lambda_r$.
\end{example}

\begin{example}[Wedge of $n$ segments]
\label{E:wedge}
Consider the wedge (one-point union) $W$ of $n$ segments
$L_1, \ldots, L_n$ in $\real^d$ attached at the origin, as depicted in
Fig.\,\ref{F:wedge}. Each segment $L_i$ is determined by its
length $\ell_i > 0$ and a unit direction vector $v_i$. We assume
that $v_i \ne v_j$, for any $1 \leq i < j \leq n$.
Let $\alpha$ be the singular measure on $\real^d$ that is
supported on $W$ and agrees with the measure induced by
arc length on each segment $L_i$. We consider the multiscale
covariance field of $\alpha$ associated with the truncation
kernel. For $x \in L_i$, $x \ne 0$, as in the case $r=1$ in
Example \ref{E:subspace}, we have that  $\Sigma_\alpha (x, \sigma)
= (2/3 \sigma^{d-3} \volume_d) \, v_i \otimes v_i$ at small enough
scales. Thus, $\Sigma_\alpha (x, \sigma)$ has rank one. However,
at the origin,
\begin{equation}
\Sigma_\alpha (0, \sigma) = \frac{1}{3 \sigma^d \volume_d}
\sum_{i=1}^n \left( \min(\sigma, \ell_i)\right)^3 
v_i \otimes v_i \,.
\end{equation}
for any $\sigma > 0$. Thus, for $\sigma \leq
\min \{\ell_i, 1\leq i \leq n\}$,
\begin{equation}
\Sigma_\alpha (0, \sigma) = \frac{1}{3 \sigma^{d-3} \volume_d}
\sum_{i=1}^n v_i \otimes v_i \,.
\end{equation}
%Hence, the rank of $\Sigma (0, \sigma)$ is the same as the
%dimension of the subspace $<v_1, \ldots, v_n>$ spanned
%by the direction vectors.
%-------------------------
\begin{figure}[h!]
\begin{center}
\begin{tabular}{cc}
\begin{tabular}{c}
\includegraphics[width=0.25\linewidth]{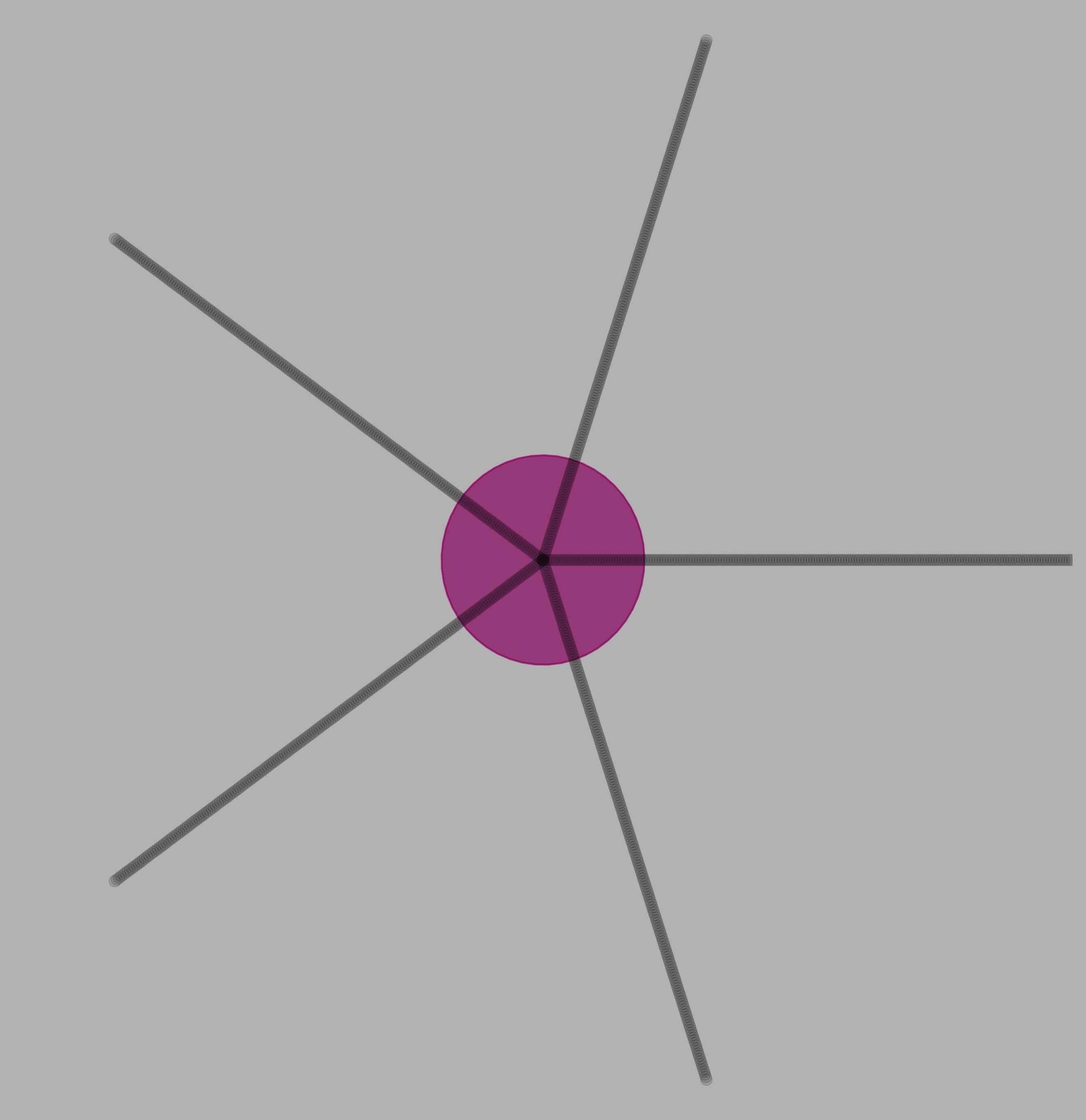} 
\end{tabular}
\qquad \quad & \quad \qquad
\begin{tabular}{c}
\includegraphics[width=0.25\linewidth]{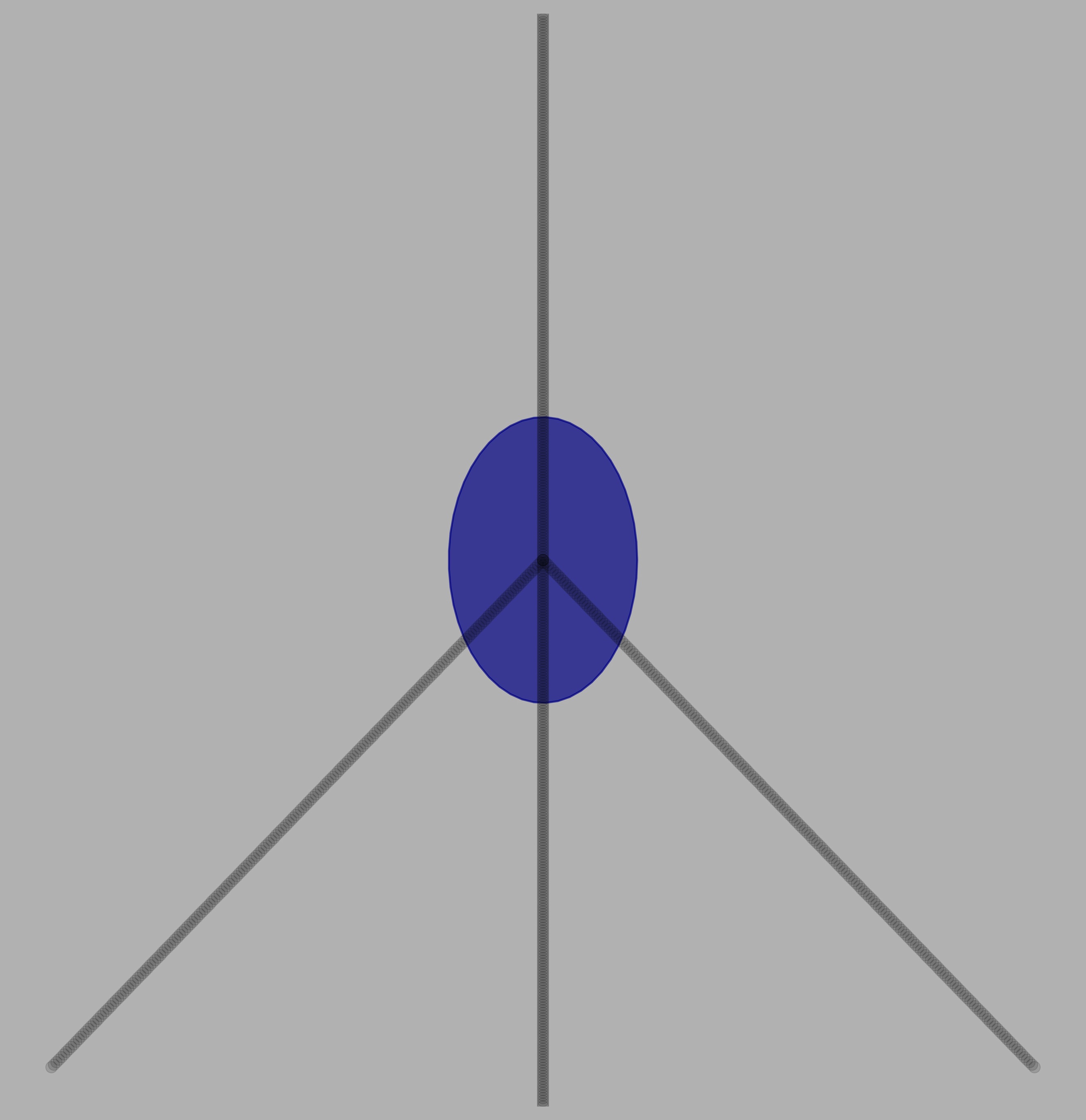}
\end{tabular}
\end{tabular}
\end{center}
\caption{Covariance at the one-point union of line segments.}
\label{F:wedge}
\end{figure}

\end{example}

%-----------------------------

\section{Geometry of Curves and Surfaces} \label{S:geometry}

In this section, we show how multiscale CTFs associated with the
truncation kernel extract precise local geometric information from plane
curves and surfaces in $\real^3$. 
%Theorem \ref{T:clt} ensures that
%these geometric properties can be estimated from sufficient data. 

\subsection{Plane Curves} \label{S:curves}

\begin{example} \label{E:circle}
We begin with the special case of a circle.
Let $C_R \subset \real^2$ be the circle of radius $R$ centered
at the origin in $\real^2$ and $\alpha$ the singular measure
supported on $C_R$ induced by arc length. For any $x \in \real^2$,
we denote $r = \|x\|$. If  $x$ is such that $\left| r - R \right|>\sigma$ then
${\Sigma}_\alpha (x,\sigma) = 0.$ Assume that $x \in \real^2$ and $0 < \sigma < R$
are such that $r \in[R-\sigma,R+\sigma]$. In this case, in the coordinate
system given by the directions $n = x/\|x\|$ and $t = n^\perp$, a calculation
shows that $\Sigma_\alpha (x,\sigma)$ is diagonal with entries
\begin{equation}
\begin{split}
\lambda_n (x,\sigma) &= \frac{1}{\pi \sigma^2} \left[ R{\phi
\left(R^2+2 r^2\right)+R^2 (R \cos \phi - 4r)\sin \phi  }\right] \\
\lambda_t (x,\sigma) &= {\frac{R^3}{\pi \sigma^2} (\phi -\sin \phi  \cos \phi )} \,,
\end{split}
\end{equation}
where $\phi = \arccos\left(\frac{R^2 + r^2-\sigma^2}{2rR}\right)$.
Thus, the normal and tangential vectors, $n$ and $t$, are
eigenvectors with eigenvalues $\lambda_n$ and $\lambda_t$,
respectively. Fig.\,\ref{F:circle} shows the eigenvalues as
functions of $r$, $0.9 \leq r \leq 1.1$, for $\sigma = 0.1$.
%---------------------------------
\begin{figure}[h!]
\begin{center}
\includegraphics[width=0.35\linewidth]{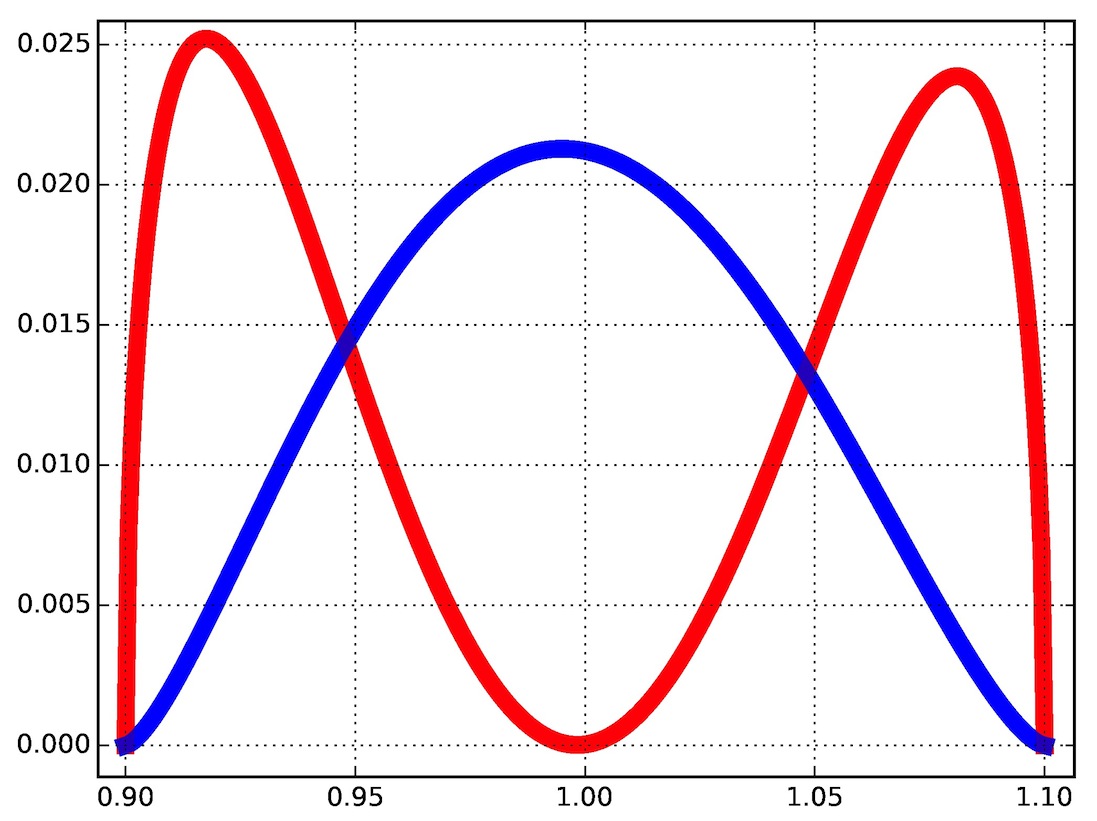} 
\end{center}
\caption{Tangential (blue) and normal (red) eigenvalues
as a function of $r$, $0.9 \leq r \leq 1.1$, at $\sigma = 0.1$,
of the multiscale CTF associated with
the truncation kernel for the singular measure induced
by arc length, supported on the unit circle in $\real^2$.}
\label{F:circle}
\end{figure}

\end{example}
 
 %----------------------------
 
Now we consider a general smooth curve $C \subset \real^2$,
that is, a 1-dimensional, smooth, properly embedded submanifold of
$\real^2$. Let $\alpha$ be the singular measure  on $\real^2$
supported on $C$ and induced by arc length. This measure is locally
finite because the embedding is proper. We calculate the small-scale
covariance at points on $C$ for the truncation kernel and show
that the curvature can be recovered from the eigenvalues of
$\Sigma_\alpha$. Let $x \in C$ be fixed. The arc-length
parametrization of $C$ near $x$ may be written as
\begin{equation}
X(s) = s-\frac{\kappa ^2 s^3}{6}+ O(s^4)
\quad \text{and} \quad
Y(s) = \frac{\kappa  s^2}{2}+\frac{\kappa _s s^3}{6}+O(s^4) \,,
\end{equation}
where $X(s)$ and $Y(s)$ are coordinates along the tangent
and normal to $C$ at $x$, respectively \cite{giblin}. Here, the
curvature $\kappa$ and its derivative $\kappa_s$ are evaluated
at $x$. A calculation yields:
%--------------
\begin{proposition} \label{P:curves}
Let $\sigma>0$ be small. If $C$ is a smooth plane curve
and $x \in C$, then in the coordinates specified above we have
\begin{equation}
{\Sigma}_\alpha (x, \sigma)=\left(
\begin{matrix}
 \frac{2\sigma}{3\pi}-\frac{\kappa^2 \sigma^3}{20 \pi}+O(\sigma^4) \quad & 
 \frac{\kappa_s \sigma^3}{15 \pi}+O (\sigma^4 ) \\
 \frac{\kappa _s \sigma^3}{15 \pi}+O (\sigma^4 ) & \frac{\kappa^2 \sigma^3}{10 \pi}+O (\sigma^4)
\end{matrix}
\right).
\end{equation}
\end{proposition}
%------------------------

Proposition \ref{P:curves} implies that, for $\sigma > 0$ small,
the eigenvalues of $\Sigma_\alpha$ are
\begin{equation}
\lambda_1 = \frac{2\sigma}{3 \pi} - \frac{\kappa^2 \sigma^3}{20 \pi}
+O(\sigma^4) 
\quad \text{and} \quad
\lambda_2 = \frac{\kappa^2 \sigma^3}{10 \pi}+O(\sigma^4) \,,
\end{equation}
so that
\begin{equation}
\text{tr} \, \Sigma_\alpha  (x, \sigma) = \frac{2 \sigma}{3 \pi} +
\frac{\kappa^2 \sigma^3}{20 \pi} + O(\sigma^4) \,.
\end{equation}
Thus, the curvature at $x \in C$ may be recovered, up to a sign, as
\begin{equation}
\kappa = \pm \lim_{\sigma \to 0} \frac{\sqrt{20 \pi}}{\sigma^{3/2}}
\left( \text{tr} \, \Sigma_\alpha  (x, \sigma) - \frac{2 \sigma}{3 \pi} \right)^{1/2}.
\end{equation}

\subsection{Surfaces in $\real^3$}

\begin{example}
Let $S_R$ be the sphere of radius $R$ centered at the origin
in $\real^3$. For $x \in \real^3$, we let $r = \|x\|$. If $x$ is such
that $\left|r - R \right| > \sigma$,
then ${\Sigma_\alpha}(x,\sigma)=0.$ Assume that $x \neq (0,0,0)$
and $\sigma>0$ are such that $r \in [R-\sigma, R + \sigma]$.
In the coordinate system given by the vector
$n = x/ \|x\|$, and any orthonormal basis $\{t_1, t_2\}$
of the orthogonal complement $n^\perp$, a direct calculation
shows that ${\Sigma_\alpha}(x,\sigma)$ is a $3\times 3$ diagonal
matrix with entries
\begin{equation}
\begin{split}
\lambda_{t_1}(x,\sigma) &= \lambda_{t_2}(x,\sigma) = 
\frac{R^4}{\sigma^3} \sin ^4\left(\frac{\phi }{2}\right) (\cos \phi + 2) \\
\lambda_n (x,\sigma) &= \frac{R^2}{2 \sigma^3}(1-\cos \phi)
\left( R^2+R \cos \phi (R \cos \phi +R-3 r) \right) \\
&+ \frac{R^2}{2 \sigma^3}(1-\cos \phi) \left(- 3 R r+3 r^2\right)
\end{split}
\end{equation}
where $\phi = \arccos\left(\frac{R^2+ r^2-\sigma^2}{2R r}\right)$.
In particular, this means that $\lambda_n$ is the eingenvalue
corresponding to the eigenvector $n$ along the normal direction
to the sphere at $x/\|x\|$,
and $\{t_1, t_2\}$ span the eigenspace along the tangent directions
with eigenvalue $\lambda_{t_1} = \lambda_{t_2}$. Fig.\,\ref{F:S2}
shows a plot of the eigenvalues as a function of $r$,
$0.9 \leq r \leq 1,1$, for $\sigma = 0.1$.
%---------------------------------
\begin{figure}[h!]
\begin{center}
\includegraphics[width=0.35\linewidth]{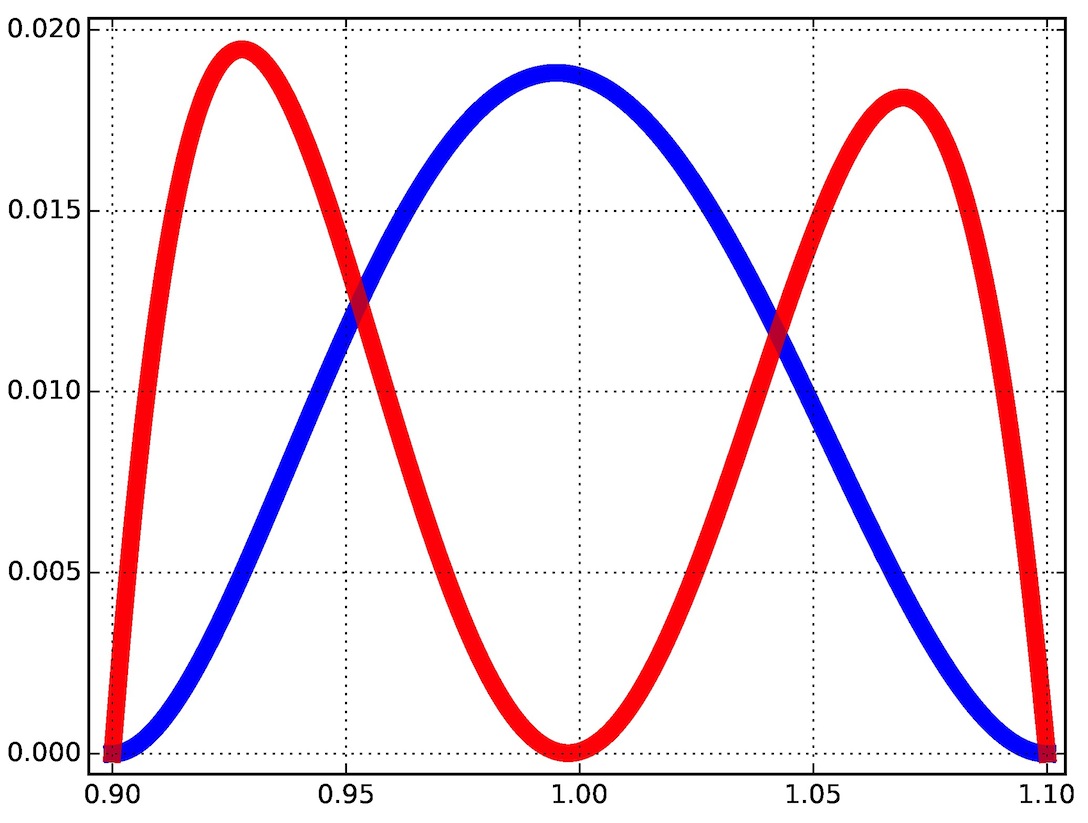} 
\end{center}
\caption{Tangential (blue) and normal eigenvalues (red)
of the multiscale CTF associated with the truncation kernel as
a function of $r$, $0.9 \leq r \leq 1.1$, at $\sigma = 0.1$,
for the singular measure induced by surface area, supported
on the unit sphere in $\real^3$.}
\label{F:S2}
\end{figure}

\end{example}

Now we consider a general smooth compact surface
$S\subset \real^3$. Let $\alpha$ be the singular measure
on $\real^3$ supported on $S$ and induced by the area
measure on $S$. We calculate the
small-scale covariance at points on $S$ for the truncation kernel
and show that the principal curvatures may indeed be recovered
from the spectrum of ${\Sigma}$.
Given a non-umbilic point $p \in S$, one can choose a Cartesian
coordinate system centered at $p$ so that the $x$-axis is along
the direction of maximal curvature at $p$, the $y$-axis is along the
direction of minimal curvature at $p$, and the $z$-axis is along
the normal to $S$ at $p$. 

\begin{proposition} \label{P:surface}
Let $\sigma>0$ be small, $p \in S$ be non-umbilic, and
$\alpha$ be the surface area measure on $S$. In the coordinate
system described above, the covariance tensor for the truncation
kernel is given by
\begin{equation}
\Sigma_\alpha (p, \sigma) = \begin{bmatrix}
A_{t_1} &  O (\sigma^4 ) & O(\sigma^5) \\
O(\sigma^4) &  A_{t_2}& O(\sigma^5) \\
O(\sigma^5) &  O(\sigma^5) & A_n
\end{bmatrix},
\end{equation}
where 

\begin{equation}
\begin{split}
A_{t_1} &= \frac{3\sigma}{16} + \frac{1}{256}
(-3\kappa_1^2 - 6\kappa_1 \kappa_2+\kappa_2^2) \sigma^3
+ O(\sigma^4),\\
A_{t_2} &= \frac{3\sigma}{16} + \frac{1}{256}
(\kappa_1^2 - 6\kappa_1 \kappa_2 - 3\kappa_2^2) \sigma^3
+ O(\sigma^4) ,\\
A_n &= \frac{3 \kappa_1^2 + 2 \kappa_1 \kappa_2 +
3 \kappa_2^2}{128} \sigma^3 + O(\sigma^4),
\end{split}
\end{equation} and
$\kappa_1 > \kappa_2$ are the principal curvatures
of $S$ at $p$.
\end{proposition}

It follows from this result that, for $\sigma>0$ small,
\begin{equation}
\begin{split}
\text{tr} \, \Sigma_\alpha (p, \sigma) &=
\frac{3}{16}\sigma+ \frac{1}{64}(\kappa_1-\kappa_2)^2\sigma^3+O(\sigma^4)
\ \text{and} \\
\text{det} \, \Sigma_\alpha (p, \sigma) &=\big({3\kappa_1^2+2\kappa_1\kappa_2+3\kappa_2^2\big)}\frac{\pi^2}{2048}\sigma^{11} + O(\sigma^{12}) \,.
\end{split}
\end{equation}
As a consequence, $\kappa_1$ and $\kappa_2$ can be recovered from
the spectrum of ${\Sigma}_\alpha(p,\sigma)$ as a function of $\sigma$.
Indeed, from the small scale asymptotics of the trace and determinant
of $\Sigma_\alpha (p, \sigma)$, we can extract the values of
$(\kappa_1-\kappa_2)^2$ and
$3\kappa_1^2+2\kappa_1\kappa_2+3\kappa_2^2$ from which we
can determine the values of $\kappa_1$ and $\kappa_2$.

\begin{proof}[Proof of Proposition \ref{P:surface}]
Using cylindrical coordinates in the chosen reference system,
we can parametrize the patch $S\cap B(p,\sigma)$
as $(\rho\cos\phi,\rho\sin\phi,z(\rho,\phi))$, for $\phi\in[0,2\pi]$, 
$\rho\in[0,\rho_\sigma(\phi)]$, where $\rho_\sigma(\phi)
= \sigma-\frac{1}{8}\big(\kappa_1(\cos\phi)^2
+ \kappa_2(\sin\phi)^2\big)^2\sigma^3+O(\sigma^4)$, and $z(\rho,\phi)
= \frac{\rho^2}{2}\big(\kappa_1(\cos\phi)^2
+ \kappa_2(\sin\phi)^2\big)+O(\sigma^3).$ 
The area element on the patch is given by
\begin{equation}
dA = \left( \rho+\frac{\rho^3}{2}(\kappa_1^2(\cos\phi)^2 +
\kappa_2^2(\sin\phi)^2)+O(\rho^5) \right) d\rho\,d\phi.
\end{equation}
Now we have all the ingredients needed to compute
$\Sigma_\alpha (p,\sigma)$. For example, to calculate the
$(1,1)$-entry, we express $\iint_{S\cap B(0,\sigma)}x^2\,dA$ as
\begin{equation}
\int_0^{2\pi}\int_0^{\rho_\sigma(\phi)}
\left[ \rho ^3 \cos ^2 \phi + \frac{\rho ^5 \cos ^2 \phi}{2}
\left(\kappa _2^2 \sin ^2 \phi +\kappa _1^2 \cos ^2 \phi \right)
+ O (\rho ^6) \right] d\rho\, d\phi\,,
\end{equation}
which after a simple but tedious calculation yields the
desired result. The computation of other
entries of the matrix follows similar steps.
\end{proof}

%-----------------------------------

\section{Stability and Consistency} \label{S:stability}

For each $p \in [1, \infty)$, let $\borel_p (\real^d)$ denote the collection
of all Borel probability measures $\alpha$ on $\real^d$ whose
$p$th moment $M_p (\alpha) = \int \|z\|^p \alpha (dz)$ is finite. We
adopt the notation $m_p (\alpha) = M^{1/p}_p (\alpha)$. For $p=\infty$,
we let $\borel_\infty(\real^d)$ be the collection of all Borel
probability measures on $\real^d$ with bounded support and
$m_\infty(\alpha)= \sup \{\|z\|,\,z\in \mathrm{supp}\,[\alpha]\}$.
By Jensen's inequality, if $1 \leq q \leq p \leq \infty$,
then $\borel_p (\real^d) \subset \borel_q (\real^d)$ and
$m_q (\alpha) \leq m_p (\alpha)$,
for any $\alpha \in \borel_p (\real^d)$.
%--------------------
\begin{definition} \label{D:moments}
For $p \in [1, \infty]$ and $\lambda>0$, we define
$\borel_p^\lambda(\real^d) \subset \borel_p (\real^d)$ as the
subset of all $\alpha \in \borel_p (\real^d)$ such that
$\alpha(A)\leq \lambda \,\mathcal{L}(A)$, for all measurable
sets $A$, where $\mathcal{L}$ stands for Lebesgue measure.
\end{definition}
%-------------------------------

\begin{example}
If $\alpha \in \borel_p (\real^d)$ is absolutely continuous with respect
to the Lebesgue measure with density function
$f \in \mathbb{L}^\infty (\real^d)$ satisfying $\|f\|_\infty \leq \lambda$,  
then $\alpha \in \borel_p^\lambda (\real^d)$.
\end{example}

Let us recall the definition of the $p$-Wasserstein distance
$\dwassp (\alpha,\beta)$ between
$\alpha, \beta \in \borel_p (\real^d)$. Let $\joint (\alpha,\beta)$ be the
collection of all couplings of $\alpha$ and $\beta$; that is, probability
measures $\mu$ on $\real^d\times\real^d$
such that $(\pi_1)_\ast \mu=\alpha$ and $(\pi_2)_\ast \mu = \beta$,
where $\pi_1, \pi_2 \colon \real^d \times \real^d \to \real^d$ denote
projections onto the first and second components, respectively. 
%--------------
\begin{definition}
For $p \in [1, \infty)$, the {\em $p$-Wasserstein distance}
between $\alpha, \beta \in \borel_p (\real^d)$ is given by
\[
\dwassp (\alpha,\beta) := \inf_{\mu \in \joint (\alpha,\beta)}
 \left(\iint\| z_1 - z_2 \|^p \mu(dz_1 \times dz_2) \right)^{1/p},
\]
and the $\infty$-Wasserstein distance between $\alpha, \beta
\in \borel_\infty (\real^d)$ by
\[
\dwassinf (\alpha,\beta) := \inf_{\mu \in \joint (\alpha,\beta)}
\sup\left\{\|z_1-z_2\|,\,(z_1,z_2)\in\mathrm{supp} \, [\mu]\right\}.
\]
\end{definition}

%The following Proposition from \cite{givens} will be useful in some
%arguments below. For any subset $U \subseteq \real^d$ and
%$\delta\geq 0$, let $U^\delta = \{x\in\real^d|\, \inf_{u\in U}
%\|x-u\|\leq \delta\}.$
%%-------------------------
%\begin{proposition} \label{P:prokhorov}
%Let $\alpha$ and $\beta$ be Borel probability measures in
%$\mathcal{P}_\infty (\real^d)$. If $W_\infty(\alpha,\beta)\leq \eta$,
%then $\alpha(U)\leq \beta(U^\eta)$ for all Borel sets $U$.
%\end{proposition}

\begin{remark} \label{R:wasserstein} 
\hfill
\begin{itemize}
\item[(i)] For any $\alpha, \beta \in \borel_p (\real^d)$, $p\in[1,\infty]$,
there exists a coupling that realizes the infimum in the definition
of $W_p (\alpha,\beta)$ (cf. \cite{givens84}).

\item[(ii)] It is a standard result that, for each $p \in [1, \infty)$, $\dwassp$
defines a metric on $\borel_p (\real^d)$ that is compatible with
weak convergence of probability measures \cite{villani03}.

\item[(iii)] If $\varphi \colon \real^d \to \real^d$ is an isometry, then
$\dwassp (\alpha, \beta) = 
\dwassp (\varphi_\ast (\alpha), \varphi_\ast (\beta))$,
for any $\alpha, \beta \in \borel_p (\real^d)$.

\end{itemize}
\end{remark}

\subsection{Smooth Kernels}

\begin{theorem}[Stability for Smooth Kernels] \label{T:stab}
Let $f \colon [0, \infty) \to \real$ be as in Definition \ref{D:kernel}
with multiscale kernel $K$. Suppose that $f$ is differentiable
and there exists a constant $A_1>0$ such that
$r^{3/2} \, |f'(r)| \leq A_1$, $\forall r \geq 0$. Then,
there is a constant $A_f >0$, that depends only on $f$, such
that
\[
\sup_{x\in\real^d} \left\| \Sigma_\alpha (x, \sigma) -
\Sigma_\beta (x, \sigma) \right\| \leq 
\frac{\sigma A_f}{C_d(\sigma)} \,
\dwassone (\alpha,\beta),
\]
for any $\alpha,\beta\in\borel_1 (\real^d)$ and any $\sigma>0$.
Here $\| \cdot \|$ is the norm associated
with the inner product defined in \eqref{E:inner}.
\end{theorem}

Theorem \ref{T:stab} shows that multiscale covariance fields
yield a robust representation of probability measures that make
their geometric properties more readily accessible, as illustrated in
our examples. In Section \ref{S:frechet}, we show that not only
is $\Sigma_\alpha (\cdot, \sigma)$ stable, but all the information
contained in the probability measure $\alpha$ is
fully absorbed into the multiscale CTF associated with the
Gaussian kernel. In fact, $\alpha$ may
be recovered from the multiscale scalar field given by
$V_\alpha (x, \sigma) = \mathrm{tr} \, \Sigma_\alpha (x, \sigma)$,
$x \in \real^d$ and $\sigma > 0$.

The following lemma will be used in the proof of the stability
theorem for smooth kernels. To simplify notation, we define
$K_\sigma \colon \real^d \to \real$ by
\begin{equation} \label{E:center}
K_\sigma (z) = K (z, 0, \sigma) = 
\frac{1}{C_d (\sigma)} \,
f \left(\frac{\Vert z\Vert^2}{\sigma^2}\right)
\end{equation}
and $Q_\sigma \colon \real^d \to \real^d \otimes \real^d$ by
\begin{equation}
Q_\sigma (z) = (z \otimes z) K_\sigma (z) \,.
\end{equation}
%Note that $0 \leq K_\sigma (x) \leq 1/ C_d (\sigma)$, since we
%are assuming that $\sup f = 1$.

\begin{lemma}\label{L:kernel}

Let  $f$ be as in Definition \ref{D:kernel} and suppose
that $f$ is differentiable and there is a constant $A_1 >0$
such that  $r^{3/2} \, |f'(r)| \leq A_1$, $\forall r \geq 0$.
Then, there is a constant $A_f > 0$, that depends only on $f$,
such that
\[
\left\| Q_\sigma(z_1) - Q_\sigma(z_2) \right\| \leq
\frac{A_f  \sigma}{C_d(\sigma)} \|z_1 - z_2\| ,
\]
for any $z_1, z_2 \in \real^d$ and all $\sigma > 0$. 
\end{lemma}

\begin{proof}
Let $z(t) = t z_1 + (1-t) z_2$, $0 \leq t \leq 1$. Then,
\begin{equation} \label{E:estimate1}
\begin{split}
Q_\sigma(z_1) - Q_\sigma(z_2) =
\int_0^1 \frac{d}{dt} Q_\sigma(z(t)) \,dt \,.
\end{split}
\end{equation}
Since
\begin{equation}
\begin{split}
\frac{d}{dt} Q_\sigma(z) &= \frac{(z_1-z_2) \otimes z}{C_d (\sigma)}
f \left( \frac{\|z\|^2}{\sigma^2} \right) +
\frac{z \otimes (z_1-z_2)}{C_d (\sigma)}
f \left( \frac{\|z\|^2}{\sigma^2} \right) \\
&+ (z \otimes z) \frac{2}{\sigma^2 C_d(\sigma)} 
f' \left(\frac{\Vert z\Vert^2}{\sigma^2}\right) \left( z \cdot
(z_1-z_2) \right) \,,
\end{split}
\end{equation}
it follows that
%-------------------
\begin{equation} \label{E:estimate2}
\begin{split}
\left\|  \frac{d}{dt} Q_\sigma(z) \right\| & \leq
\frac{ 2 \|z\|}{C_d (\sigma)} f \left( \frac{\|z\|^2}{\sigma^2} \right) 
\|z_1 - z_2\| \\
&+ \frac{2 \|z\|^3}{\sigma^2 C_d(\sigma)}
\left| f' \left(\frac{\Vert z\Vert^2}{\sigma^2}\right) \right| \|z_1 - z_2\| \\
&= \frac{ 2 \sigma}{C_d (\sigma)} \frac{\|z\|}{\sigma}
f \left( \frac{\|z\|^2}{\sigma^2} \right)  \|z_1 - z_2\| \\
&+ \frac{2 \sigma}{C_d(\sigma)} \frac{\|z\|^3}{\sigma^3}
\left| f' \left(\frac{\Vert z\Vert^2}{\sigma^2}\right) \right|
\|z_1 - z_2\| \,.
\end{split}
\end{equation}
Since $f$ is smooth, condition (c) of Definition \ref{D:kernel}
ensures that there is a constant  $A_2 > 0$ such that
$\sqrt{r} f(r) \leq A_2$, $\forall r \geq 0$. Moreover, by hypothesis
$r^{3/2} |f' (r)| \leq A_1$. Thus, \eqref{E:estimate2} implies that
\begin{equation} \label{E:estimate3}
\left\|  \frac{d}{dt} Q_\sigma(z) \right\| \leq
\frac{\sigma A_f}{C_d(\sigma)} \|z_2 - z_1\|\,,
\end{equation}
where $A_f = 2(A_1 + A_2)$. The lemma follows from
\eqref{E:estimate1} and \eqref{E:estimate3}.
\end{proof}

\begin{proof}[Proof of Theorem \ref{T:stab}]
Without loss of generality, we may assume that $x=0$.
We express the covariance fields as
\begin{equation}
\Sigma_\alpha(x,\sigma)=\int Q_\sigma(z_1)\,\alpha(dz_1)
\quad \text{and} \quad
\Sigma_\beta(x,\sigma) = \int Q_\sigma(z_2)\,\beta(dz_2).
\end{equation}
Given $\eta > 0$ satisfying $\dwassone (\alpha,\beta) < \eta$,
let  $\mu \in \joint (\alpha,\beta)$ be a coupling  such that
\begin{equation} \label{E:wass}
\iint\|z_1 - z_2\| \mu (dz_1 \times dz_2) < \eta \,.
\end{equation}
We may write
\begin{equation} 
\begin{split}
\Sigma_\alpha(x,\sigma) =\int Q_\sigma(z_1)\,\mu (dz_1 \times dz_2)
\ \ \text{and} \ \ 
\Sigma_\beta (x,\sigma) =\int Q_\sigma(z_2)\,\mu (dz_1 \times dz_2) \,.
\end{split}
\end{equation}
Therefore,
\begin{equation} \label{E:cestimate0}
\left\| \Sigma_\alpha (x,\sigma) - \Sigma_\beta (x,\sigma) \right\| 
\leq \iint \|Q_\sigma( z_1) - Q_\sigma( z_2)\| \,\mu(dz_1\times dz_2) \,.\end{equation}
%----------------------
Lemma \ref{L:kernel} and \eqref{E:cestimate0} imply that
\begin{equation} \label{E:cestimate1}
\begin{split}
\left\| \Sigma_\alpha (x,\sigma) - \Sigma_\beta (x,\sigma) \right\| 
&\leq \frac{\sigma A_f}{C_d(\sigma)}
\iint \|z_1- z_2\| \,\mu(dz_1\times dz_2) \\
&\leq \frac{\sigma A_f}{C_d(\sigma)}  \eta \,.
\end{split}
\end{equation}
%-----------------------
Since \eqref{E:cestimate1} holds for any
$\eta > \dwassone (\alpha,\beta)$,
we can conclude that
\begin{equation} 
\left\| \Sigma_\alpha(x,\sigma) - \Sigma_\beta(x,\sigma) \right\|
\leq \frac{\sigma A_f}{C_d(\sigma)}
\dwassone (\alpha, \beta) \,,
\end{equation}
as claimed.
\end{proof}
%-------------------------

In what follows, given random vectors $y_i\in \real^d$,
$i \in \mathbb{N}$, we let $\alpha_n =
\sum_{i=1}^n \delta_{y_i}/n$.

\begin{corollary}[Consistency for Smooth Kernels]
\label{C:consistency}
Let $K$ be a multiscale kernel as in Theorem \ref{T:stab}
and $\sigma > 0$. If $\alpha \in \mathcal{P}_1 (\real^d)$ and
$y_i \in \real^d$, $i \in \mathbb{N}$,
are i.i.d. random variables with distribution $\alpha$, then
\[
\sup_{x \in \real^d}
\left\| \Sigma_{\alpha_n}(x,\sigma)
-\Sigma_\alpha(x,\sigma) \right\|
\xrightarrow{n\uparrow\infty} 0
\]
almost surely.
\end{corollary}
\begin{proof}
Theorem \ref{T:stab} implies that
\begin{equation}
\sup_{x \in \real^d} \left\| \Sigma_{\alpha_n}(x,\sigma) - 
\Sigma_\alpha(x,\sigma) \right\| \leq
\frac{\sigma A_f}{C_d(\sigma)}
\, \dwassone (\alpha_n,\alpha).
\end{equation}
The conclusion follows from the fact that $\dwassone$ metrizes weak
convergence of probability measures in $\borel_1 (\real^d)$
and Varadarajan's Theorem \cite{dudley} about convergence of
empirical measures on Polish spaces that ensures that
$\alpha_n$ converges weakly to $\alpha$ almost surely.
\end{proof}
%---------------------------
Corollary \ref{C:consistency} guarantees the asymptotic consistency
of empirical CTFs. However, in applications, it is important to have
estimates of the rate of convergence, which we derive
from the stability theorem and a result of Fournier and
Guillin \cite[Theorem 1]{fournier14}.

\begin{theorem}[Fournier and Guillin, \cite{fournier14}]
Let $\alpha \in \borel_s (\real^d)$, where $s > 1$. If $y_1, \ldots, y_n$
are i.i.d. random variables with distribution $\alpha$ and $p \in [1,s)$,
then there exists a constant $b> 0$ that depends only on
$p,s$ and $d$ such that
\begin{equation*}
\begin{split}
\mathbb{E} \left[W_p (\alpha,\alpha_n)\right] & \leq \\
b \, m_s^p (\alpha) \, 
&\cdot \begin{cases}
n^{-\frac{s-p}{s}} + n^{-\frac{1}{2}} & \text{if $p>d/2$ and $s\neq 2p$;} \\
n^{-{\frac{s-p}{s}}} + n^{-\frac{1}{2}}\log(1+n) &
\text{if $p=d/2$ and $s\neq 2p$;} \\
n^{-\frac{s-p}{s}} +n^{-\frac{p}{d}} & \mbox{if $p\in[1,d/2)$ and $s\neq d/(d-p)$,}
\end{cases}
\end{split}
\end{equation*}
for any $n \geq 1$.
\end{theorem}
%-----------------------
\begin{corollary} \label{C:rates}
Let $f$ be as in Theorem \ref{T:stab} and $\sigma > 0$.
Suppose that $\alpha \in \borel_3 (\real^d)$ and
$y_i$, $i \in \mathbb{N}$, are i.i.d. random
variables with distribution $\alpha$.  Then, there is
a constant $b> 0$, that depends only on $d$,
such that
\[
\begin{split}
\mathbb{E} &\left[ \, \sup_{x\in\real^d} \left\| \Sigma_{\alpha_n} (x, \sigma)
- \Sigma_\alpha (x, \sigma) \right\| \right] \leq \\
&\frac{\sigma A_f b}{C_d (\sigma)} m_3 (\alpha) \cdot 
\begin{cases}
n^{-\frac{2}{3}} + n^{-\frac{1}{2}} & \text{if $d=1$;}\\
n^{-{\frac{2}{3}}} + n^{-\frac{1}{2}}\log(1+n) & \text{if $d=2$;}\\
n^{-\frac{2}{3}} +n^{-\frac{1}{d}} & \text{if $d\geq 3$.}
\end{cases}
\end{split}
\]
\end{corollary}
%---------------------
\begin{proof}
Since $\borel_3 (\real^d) \subseteq \borel_1 (\real^d)$,
$\alpha$ is also an element of $\borel_1 (\real^d)$. The
conclusion follows by invoking Theorem \ref{T:stab} and the
result of Fournier and Guillin with $p=1$ and $s=3$. The
constant $b$ depends only on $d$ because we are
fixing $p$ and $s$.
\end{proof}

%--------------------
%\begin{example}[A numerical simulation]
%In this example we empirically test the rates of convergence
%obtained in Corollary \ref{C:rates}. We consider the singular
%measure on $\real^2$ supported on the unit circle $S^1$
%induced by arc length and normalized so that $\alpha (S^1) =1$.
%For a positive integer $n$, we consider $n$ i.i.d. samples
%from $\alpha$ and compare $\Sigma_{\alpha_n} (x,\sigma)$
%with the theoretical $\Sigma_\alpha (x, \sigma)$.
%\end{example}

Theorem \ref{T:stab} ensures that multiscale covariance fields
are stable. However, the results do not apply
to some discontinuous kernels of practical interest. Nonetheless,
we prove a stability theorem for the truncation kernel, as well as
pointwise convergence results for more general kernels.
%---------------------------
\subsection{The Truncation Kernel}

We begin our discussion of covariance fields associated
with the truncation kernel with a stability theorem with
respect to the $\infty$-Wasserstein metric. In preparation for
the proof of the theorem, we introduce some notation.
For $0 \leq a< b$, let
\begin{equation} \label{E:annulus}
R_d (a,b) = \{y \in \real^d \, \colon \, a < \|y\| \leq b \} \,.
\end{equation}
and
\begin{equation} \label{E:moment}
s_d (a,b) = \int_{R_d (a,b)} \|y\|^2 \, dy 
\end{equation}
be its radial moment of inertia.
We will use the fact that for any $B \geq b$, the inequality
\begin{equation} \label{E:inertia}
s_d (a,b) \leq (b-a) \, \frac{\area_{d-1}}{d+2} \, \frac{B^{d+2}}{B-a}
\end{equation}
holds. Indeed,
\begin{equation} 
\begin{split}
s_d(a,b) &= \frac{\area_{d-1}}{d+2} \left(b^{d+2} - a^{d+2} \right) =
\frac{\area_{d-1}}{d+2} a^{d+2} \left((b/a)^{d+2}-1\right) \\
&=\frac{\area_{d-1}}{d+2}a^{d+2} \frac{(b-a)}{a}
\left(1+ \left(\frac{b}{a}\right) + \cdots + \left(\frac{b}{a}\right)^{d+1}\right) \\
&\leq \frac{b-a}{B-a} \, \frac{\area_{d-1}}{d+2} a^{d+2}
\frac{B-a}{a} \left(1+ \left(\frac{B}{a}\right) + \cdots +
\left(\frac{B}{a}\right)^{d+1}\right) \\
&\leq (b-a) \, \frac{\area_{d-1}}{(d+2)} \frac{B^{d+2} - a^{d+2}}{B-a}
\leq (b-a) \, \frac{\area_{d-1}}{d+2} \, \frac{B^{d+2}}{B-a} \,.
\end{split}
\end{equation}

Let $\borel_\infty (\real^d)$ and $\borel^\lambda_\infty (\real^d)$
be as in Definition \ref{D:moments}.
\begin{theorem}[Stability for the Truncation Kernel] \label{T:stab-trunc}
Let $\sigma>0$ and $\lambda>0$. Suppose $\Omega \subset \real^d$ is
a compact set and let $c > 0$ satisfy $\mathrm{diam} (\Omega) \leq c$.
There is a constant $A = A(\sigma, d, c) > 0$ such that if
$\alpha \in \borel_\infty^\lambda (\real^d)$ and $\beta \in 
\borel_\infty (\real^d)$ have their supports contained in $\Omega$,
then the multiscale covariance tensor fields of $\alpha$ and
$\beta$ associated with the truncation kernel $T$ satisfy
\begin{equation*}
\sup_{x\in\real^d} 
\|\Sigma_\alpha (x, \sigma) - \Sigma_\beta (x, \sigma)\| \leq
\lambda A \dwassinf (\alpha, \beta) \,.
\end{equation*}
\end{theorem}
%--------------
\begin{proof}
Without loss of generality, we may assume that $x \in \real^d$
is the origin. We abbreviate $\eta = W_\infty (\alpha, \beta)$
and let $\mu \in \Gamma (\beta, \alpha)$ be a coupling that realizes
$\eta$. Clearly, $\eta \leq \text{diam} (\Omega)\leq c$.  

Using the notation
introduced in \eqref{E:annulus}, let $A_{\sigma, \eta} =
R_d (\sigma, \sigma + \eta)$. We write
\begin{equation} \label{E:diff1}
\begin{split}
\Sigma_\alpha (x, \sigma) - \Sigma_\beta (x, \sigma) &= 
\underbrace{
\Sigma_\alpha (x, \sigma) - \frac{(\sigma + \eta)^d}{\sigma^d}\,
\Sigma_\alpha (x, \sigma + \eta) 
}_{T_1} \\
&+ 
\underbrace{
\frac{(\sigma + \eta)^d}{\sigma^d}\, \Sigma_\alpha (x, \sigma + \eta)
- \Sigma_\beta (x, \sigma)
}_{T_2} \,.
\end{split}
\end{equation}
Using the fact that $\alpha \in \borel^\lambda_\infty (\real^d)$,
we bound the norm of $T_1$  as follows:
\begin{equation} \label{E:diff2}
\begin{split}
\left\|T_1 \right\| &=
\frac{1}{\sigma^d \volume_d}
\Big\| \int_{A_{\sigma, \eta}} y \otimes y \, d\alpha (y) \Big\|
\leq \frac{1}{\sigma^d \volume_d} 
\int_{A_{\sigma, \eta}} \| y\|^2 \, d\alpha (y) \\
&\leq \frac{\lambda}{\sigma^d \volume_d}
\int_{A_{\sigma, \eta}} \| y\|^2 \, d y 
= \frac{\lambda}{\sigma^d \volume_d}
s_d (\sigma, \sigma + \eta) \,.
\end{split}
\end{equation}
Using \eqref{E:inertia} with $a = \sigma$, $b = \sigma + \eta$
and $B = \sigma + c$, we have that
\begin{equation} \label{E:bound3}
s_d(\sigma, \sigma + \eta)\leq \eta \,\frac{\area_{d-1}}{d+2}
\frac{(\sigma + c)^{d+2}}{c} = \eta \, \volume_d \,
\frac{d}{d+2} \frac{(\sigma + c)^{d+2}}{c} \,.
\end{equation}
Thus,
\begin{equation} \label{E:diff3}
\left\| T_1 \right\| \leq
\frac{\lambda d}{d+2} \frac{(\sigma + c)^{d+2}}{c \sigma^d}
\, \dwassinf (\alpha, \beta) \,.
\end{equation}
Now we examine $T_2$. Let $\mathbb{I} \colon \real^d \to \real$ the
characteristic function of the closed ball of radius 1 centered
at the origin. Using the coupling $\mu$, we write
\begin{equation} \label{E:diff4}
\begin{split}
T_2  &= \frac{1}{\sigma^d \volume_d} \iint_{\real^d \times \real^d}
\left[ (y_1 \otimes y_1) \,
\mathbb{I} \left(\frac{y_1}{\sigma + \eta}\right)
- (y_2 \otimes y_2) \, \mathbb{I} \left(\frac{y_2}{\sigma}\right) \right]
\mu (dy_1 \times dy_2) \,.
\end{split}
\end{equation}
Note that the integrand in \eqref{E:diff4} vanishes on
$(\real^d \setminus B_{\sigma + \eta}) \times
(\real^d \setminus B_{\sigma})$ and
the integral also vanishes over $\left(\real^d \setminus 
B_{\sigma + \eta} \right) \times B_\sigma$ because this
subdomain is disjoint from $\text{supp} \,[\mu]$. Combining
these remarks with
\begin{equation}
y_1 \otimes y_1 - y_2 \otimes y_2 = (y_1 - y_2) \otimes y_1 +
y_2 \otimes (y_1- y_2) \,,
\end{equation}
we may rewrite \eqref{E:diff4} as
\begin{equation} \label{E:diff5}
\begin{split}
T_2  &= \frac{1}{\sigma^d \volume_d}
\iint_{B_{\sigma + \eta} \times B_\sigma} \left[ (y_1 \otimes y_1) \,
- (y_2 \otimes y_2) \right] \mu (dy_1 \times dy_2) \\
&+ \frac{1}{\sigma^d \volume_d}
\iint_{B_{\sigma + \eta} \times (\real^d \setminus B_\sigma)} 
y_1 \otimes y_1 \, \mu (dy_1 \times dy_2) \\
&= \frac{1}{\sigma^d \volume_d}
\iint_{B_{\sigma + \eta} \times B_\sigma}
(y_1-y_2) \otimes y_1 \, \mu (dy_1 \times dy_2) \\
&+ \frac{1}{\sigma^d \volume_d}
\iint_{B_{\sigma + \eta} \times B_\sigma}
y_2 \otimes (y_1-y_2) \, \mu (dy_1 \times dy_2) \\
&+ \frac{1}{\sigma^d \volume_d}
\iint_{B_{\sigma + \eta} \times A_{\sigma, 2\eta}} 
y_1 \otimes y_1 \, \mu (dy_1 \times dy_2) \,.
\end{split}
\end{equation}
For the last equality, we used again the fact that
\begin{equation} \label{E:support}
(y_1, y_2) \in \text{supp}\,[\mu] \Longrightarrow
\|y_1 - y_2 \|\leq \eta \,.
\end{equation}
From \eqref{E:diff5} and \eqref{E:support}, using the facts
that $\|y_1\| \leq \sigma + \eta$ for $y_1 \in B_{\sigma+\eta}$
and $\|y_2\| \leq \sigma$ for $y_2 \in B_\sigma$,
we may conclude that
\begin{equation} \label{E:diff6}
\begin{split}
\|T_2\| &\leq \frac{\sigma + \eta}{\sigma^d \volume_d}
\iint_{B_{\sigma + \eta} \times B_\sigma}
\|y_1-y_2\|\, \mu (dy_1 \times dy_2) \\
&+ \frac{\sigma}{\sigma^d \volume_d}
\iint_{B_{\sigma + \eta} \times B_\sigma}
\|y_1-y_2\|\, \mu (dy_1 \times dy_2) \\
&+ \frac{1}{\sigma^d \volume_d}
\iint_{B_{\sigma + \eta} \times A_{\sigma, 2\eta}}
\|y_1\|^2 \, \mu (dy_1 \times dy_2) \\
&\leq \frac{\sigma + c}{\sigma^d \volume_d}
\eta \int_{B_{\sigma + \eta}} \alpha (dy_1)
+ \frac{\sigma}{\sigma^d \volume_d}
\eta \int_{B_{\sigma + \eta}} \alpha (dy_1) \\
&+ \frac{1}{\sigma^d \volume_d}
\int_{A_{(\sigma-\eta)^+, 2\eta}} \|y_1\|^2 \, \alpha (dy_1) \,,
\end{split}
\end{equation}
where $(\sigma - \eta)^+ = \text{max} \{\sigma - \eta, 0\}$.
Since $\alpha \in \borel^\lambda_\infty (\real^d)$ and
$\eta \leq c$,
\begin{equation} \label{E:ball}
\int_{B_{\sigma + \eta}} \alpha (dy_1) \leq
\lambda \int_{B_{\sigma + \eta}} dy_1 =
\lambda (\sigma + \eta)^d \nu_d \leq
\lambda (\sigma + c)^d \nu_d
\end{equation}
and
\begin{equation} \label{E:inertia2}
\begin{split}
\int_{A_{\sigma-\eta, 2\eta}} \|y_1\|^2 \, \alpha (dy_1) & \leq
\lambda \int_{A_{(\sigma-\eta)^+, 2\eta}} \|y_1\|^2 \, dy_1 \\
&=\lambda s_d ((\sigma - \eta)^+, \sigma + \eta) \\
&\leq 2 \lambda \eta \frac{\omega_{d-1}}{d+2} 
\frac{(\sigma + c)^{d+2}}{\eta + c} \\
&\leq 2 \lambda \eta \frac{d}{d+2} \volume_d
\frac{(\sigma + c)^{d+2}}{c} \,,
\end{split}
\end{equation}
where we used \eqref{E:inertia} with $a= (\sigma - \eta)^+$,
$b = \sigma + \eta$ and $B = \sigma + c$.
Combining \eqref{E:diff6}, \eqref{E:ball} and \eqref{E:inertia2},
we obtain
\begin{equation} \label{E:diff7}
\|T_2\| \leq \frac{2 \sigma + c}{\sigma^d}
\lambda (\sigma + c)^d \eta 
+ \frac{2 \lambda}{\sigma^d}
\frac{d}{d+2} 
\frac{(\sigma + c)^{d+2}}{c} \eta
\end{equation}
From \eqref{E:diff1}, \eqref{E:diff3} and \eqref{E:diff7},
it follows that
\begin{equation}
\begin{split}
\|\Sigma_\alpha (x, \sigma) - \Sigma_\beta (x, \sigma) \|
&\leq \lambda \frac{ d}{d+2} \frac{(\sigma + c)^{d+2}}{c \sigma^d} 
\dwassinf (\alpha, \beta)  \\
&+ \lambda  \frac{(2 \sigma + c)(\sigma + c)^d}{\sigma^d}
\, \dwassinf (\alpha, \beta)  \\
&+ \lambda \frac{2d}{(d+2)}  \frac{(\sigma + c)^{d+2}}{c \sigma^d }
\, \dwassinf (\alpha, \beta) \\
&= \lambda A(\sigma, d, c) \dwassinf (\alpha, \beta) \,.
\end{split}
\end{equation}
Since $x \in \real^d$ is arbitrary, the claim follows.
\end{proof}

We now derive a consistency result and estimates for
the rate of convergence of empirical approximations to multiscale
covariance fields. The following result is a $W_\infty$-counterpart
to the theorem by Fournier and Guillin stated above.
%--------------------
\begin{theorem}[Garc\'{\i}a-Trillos and Slep\v{c}ev \cite{gts15}]
\label{T:slepcev}
Let $\Omega\subset \real^d$ be a bounded connected open subset with Lipschitz boundary. Let $\alpha$ be a probability measure on $\Omega$ with density $f_\alpha \colon \Omega \rightarrow(0,\infty)$ such that there exists $\lambda \geq 1$ with $\lambda^{-1} \leq f_\alpha(x) \leq \lambda$, for all $x\in \Omega$, and let $y_, \ldots, y_n$ be i.i.d. random variables with distribution $\alpha$. Then, there exist constants $c_1, C_1, C_2 > 0$, depending only on $\Omega$ and $\lambda$, such that for all $n \in \mathbb{N}$ and $p > 1$,
\begin{equation*}\mathbb{P} \big( W_\infty(\alpha,\alpha_n)
\leq (C_1+ C_2 \sqrt{p}) \,r_d(n) \big) \geq 1- c_1 n^{-p} \,,
\end{equation*} where $r_2(n) = \frac{\ln(n)^{3/4}}{n^{1/2}}$ and
$r_d(n) = \frac{\ln(n)^{1/d}}{n^{1/d}}$, for $d\geq 3.$
\end{theorem}
%------------------------

\begin{corollary}[Consistency for the Truncation Kernel]  \label{C:ctf-infty}
Let $\alpha$ be a probabilty measure on $\real^d$ with density $f_\alpha$ and  let $\Omega_\alpha$ be the interior of the support of $\alpha$. Assume that $\Omega_\alpha$ is bounded and connected with Lipschitz boundary $\partial \Omega_\alpha$. Furthermore, assume that there exists $\lambda \geq 1$ such that $\lambda^{-1}\leq f_\alpha(z) \leq \lambda$, for all $z \in \Omega_\alpha$. If $y_i$,
$i \in \mathbb{N}$, are i.i.d. random variables with
distribution $\alpha$, then, for any $p > 1$, there are constants
$C = C (\Omega_\alpha, \lambda, p) > 0$ and
$c_1 = c_1 (\Omega_\alpha, \lambda) >0$ such that  
\begin{equation*}\mathbb{P} \left( \, \sup_{x\in\real^d}
\left\| \Sigma_{\alpha_n} (x, \sigma) - \Sigma_\alpha (x, \sigma) \right\|
\leq C \,r_d(n) \right) \geq 1- c_1 n^{-p} \,.
\end{equation*}
Here, $\alpha_n = \sum_{i=1}^n \delta_{y_i} /n$.
\end{corollary}

\begin{proof}
We use Theorem \ref{T:slepcev} and write
$C' = C_1 + \sqrt{p} \, C_2$.
Theorem \ref{T:stab-trunc} implies that there is a constant
$C'' = C'' (\Omega_\alpha, \lambda) > 0$
such that
\begin{equation}
\sup_{x\in\real^d} \left\| \Sigma_{\alpha_n} (x, \sigma) -
\Sigma_\alpha (x, \sigma) \right\| \leq C'' \,
\dwassinf (\alpha, \alpha_n) \,.
\end{equation}
Thus,
\begin{equation}
\begin{split}
\mathbb{P} \, &\Big(\,\sup_{x\in\real^d}
\left\| \Sigma_{\alpha_n} (x, \sigma) - \Sigma_\alpha (x, \sigma) \right\|
\leq C' C''  r_d (n) \Big) \geq \\ &\geq \mathbb{P}
\big( \dwassinf (\alpha, \alpha_n) \leq C' r_d (n) \big) \geq 1 - c_1 n^{-p} \,.
\end{split}
\end{equation}
The claim follows by setting $C = C' C''$.
\end{proof}
%----------------------
%\begin{corollary}\label{C:w-infty}
%Let $\alpha$ be a probabilty measure with density $f_\alpha$ and 
%$\area_\alpha$ the interior of the support of $\alpha$.
%Assume that $\area_\alpha$ is bounded and connected with
%Lipschitz boundary $\partial \area_\alpha$. Furthermore, assume that
%there exists $\lambda \geq 1$ such that
%$\lambda^{-1}\leq f_\alpha(z)\leq \lambda$, for all $z\in\area_\alpha$.
%Then, there exist constants $A_1, A_2 > 0$, depending
%only on $\lambda,\area_\alpha,$ and $d$, such that  
%\begin{equation*}
%\expect{W_\infty(\alpha,\alpha_n)}
%\leq A_1\,n^{-2} + A_2\, r_d(n) \,,
%\end{equation*}
%for all $n\in\mathbb{N}.$
%\end{corollary}
%\begin{proof}
%The proof follows from: (a) the identity $\mathbb{E}(Z) =
%\int_{0}^\infty\mathbb{P}(Z>t)\,dt$ that is valid for any non-negative
%random variable $Z$; (b) the fact that $W_\infty(\alpha,\alpha_n)
%\leq D =\mathrm{diam}(\area_\alpha)$; and (c) the
%theorem by Garc\'{\i}a-Trillos and Slep\v{c}ev.
%\end{proof}

\begin{corollary}\label{C:trunc-stab} 
Let $\sigma > 0$ and $p > 1$. Under the assumptions of Corollary \ref{C:ctf-infty}, for the truncation kernel, there exist $N=N(\sigma, \Omega_\alpha, \lambda) \in \mathbb{N}$ and a constant $A = A (\sigma, \Omega_\alpha,\lambda) > 0$ such that

\begin{equation*}
\mathbb{E} \left[ \, \sup_{x\in\real^d} \left\| \Sigma_{\alpha_n} (x, \sigma)
- \Sigma_\alpha (x, \sigma) \right\| \right] \leq A r_d(n),
\end{equation*} for all $n\geq N.$
\end{corollary}

\begin{proof}
We apply to $\dwassinf (\alpha, \alpha_n)$ the identity
$\mathbb{E}(Z) = \int_{0}^\infty\mathbb{P}(Z>t)\,dt$
that is valid for any non-negative random variable $Z$ with finite
first moment. Since
$W_\infty(\alpha,\alpha_n) \leq D =\mathrm{diam}(\Omega_\alpha)$,
we get
\begin{equation} \label{E:ewass1}
\expect{\dwassinf (\alpha, \alpha_n)} = \int_0^D
\mathbb{P} \left(\dwassinf (\alpha, \alpha_n) > t \right) \, dt \,.
\end{equation}
Theorem \ref{T:slepcev} implies that
\begin{equation} \label{E:ewass2}
\mathbb{P} \left( W_\infty(\alpha,\alpha_n) >
(C_1+ \sqrt{p} \, C_2) \,r_d(n) \right) \leq n^{-p} \,.
\end{equation}
Let $t_0 = \min \{D, (C_1+ \sqrt{p} \, C_2)\,r_d(n) \}$.
From \eqref{E:ewass1} and \eqref{E:ewass2},
\begin{equation} \label{E:ewass3}
\begin{split}
\hspace{-0.1in}
\expect{\dwassinf (\alpha, \alpha_n)} &= \int_0^{t_0}
\mathbb{P} \left(\dwassinf (\alpha, \alpha_n) > t \right) dt
+ \int_{t_0}^D \mathbb{P}
\left(\dwassinf (\alpha, \alpha_n) > t \right) dt \\
&\leq t_0 + D \,
\mathbb{P} \left(\dwassinf (\alpha, \alpha_n) >
(C_1+ \sqrt{p} \, C_2)\,r_d(n) \right) \\
&\leq (C_1+ \sqrt{p} \, C_2) \, r_d (n) + D n^{-p} \,.
\end{split}
\end{equation}
Fixing $p$, say $p=2$, for $n$ sufficiently large, the dominant term
on this last expression is the one involving $r_d (n)$. Thus,
the claim follows from \eqref{E:ewass3} and
Theorem \ref{T:stab-trunc} applied to $\alpha$ and
$\beta = \alpha_n$.
\end{proof}
%------------------
\begin{remark}
We carry out an experiment to test the convergence
rates obtained in Corollary \ref{C:trunc-stab}. We consider
the probability measure $\alpha$ supported on the unit circle
$\sphere^1 \subset \real^2$ induced by the normalized arc length
element $(2 \pi)^{-1} ds$. In this case, for the truncation kernel,
$\Sigma_\alpha$ was calculated explicitly in Example \ref{E:circle}.
We consider sets of i.i.d.\ samples of size $n$,
$10 \leq n \leq 10^6$. For each $n$, thirty sets of samples
are taken. For each such set, we compute
$\Sigma_{\alpha_n}$ and estimate the ``error'' as
$\max \|\Sigma_{\alpha_n} (x,\sigma) -
\Sigma_\alpha (x, \sigma)\|$, for $\sigma = 0.6$, where the
maximum is taken over gridpoints on a $24 \times 24$ grid
on the square $[-1.5,1.5] \times [-1.5,1.5]$. We let
$\varepsilon_n$ be the average error over all thirty sets of samples.
Figure \ref{F:error} shows a plot (in blue) of $\varepsilon_n$ in
log-log scale. To compare $\varepsilon_n$ with the predicted rates,
we use a least-squares fit, in log-log scale, of the form 
$\varepsilon = C r_2 (n) = C \frac{\ln(n)^{3/4}}{n^{1/2}}$, also
shown in Figure \ref{F:error} (in red). The discrepancy between
the predicted and observed rates suggests that 
Corollary \ref{C:trunc-stab} might not be optimal. A curve of the
form $\varepsilon = C n^{-1/2}$, shown in green, produces a
tighter fit to the data, suggesting that the optimal bound
might be  $O(n^{-1/2})$.
%---------------------------------
\begin{figure}[ht]
\begin{center}
\includegraphics[width=0.5\linewidth]{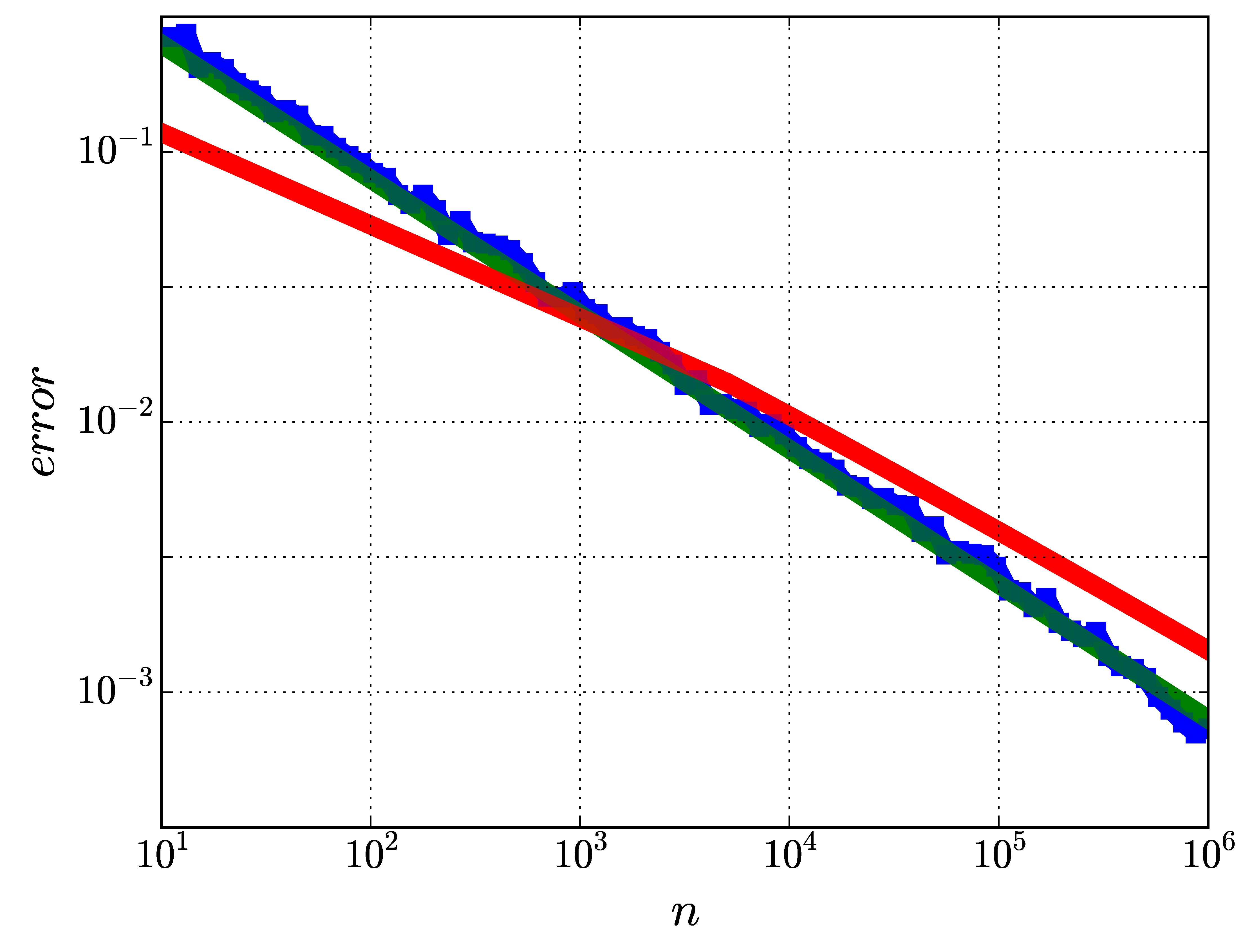} 
\end{center}
\caption{Log-log plots of experimental error rates (in blue) for
empirical covariance fields, rates predicted by Corollary
\ref{C:trunc-stab} (in red), and a least-squares fit of order
$n^{-1/2}$ (in green).}
\label{F:error}
\end{figure}
%-------------
\end{remark}

\subsection{General Kernels}

We conclude the discussion of convergence of empirical
CTFs with a pointwise central limit theorem (CLT) that holds 
for kernels in the full generality of Definition \ref{D:kernel}.
One may think of it as a CLT for each entry of the matrix
$\Sigma_\alpha (x, \sigma)$.
If $e_1, \ldots, e_d$ is an orthonormal basis of $\real^d$, the
$(i,j)$-entry of the covariance matrix in this coordinate system
is given by $\Sigma_\alpha (x, \sigma) (e_i, e_j)$, the bilinear
form $\Sigma_\alpha (x, \sigma)$ evaluated at $(e_i, e_j)$.
In matrix notation, this is the same as
$\inner{e_i}{\Sigma_\alpha (x, \sigma) e_j}$.
More generally, for fixed $u, v, x \in \real^d$ and $\sigma > 0$,
we consider
\begin{equation}
\Sigma_\alpha (x, \sigma) (u,v) =
\int (y-x) \otimes (y-x) (u,v) K(x,y, \sigma) \, \alpha (d y).
\end{equation}
Consider the random variable
\begin{equation} \label{E:randomz}
z_{uv} (y) = (y-x) \otimes (y-x) \, (u,v) K(x,y,\sigma),
\end{equation}
where $y$ has distribution $\alpha$. Clearly, 
\begin{equation} \label{E:zmean}
\expect{z_{uv}} = \Sigma_\alpha (x, \sigma) (u,v) \,.
\end{equation}
%The assumptions in Definition \ref{D:kernel} that $f$ is bounded
%and $r f(r) < C$, $\forall r>0$, guarantee that $z_{uv}$ has finite
%variance $\sigma^2_{uv}$.
%-----------------------------
\begin{theorem}[Central Limit] \label{T:clt}

If $f$ is as in Definition \ref{D:kernel}, then $z_{uv}$
has finite variance $\sigma^2_{uv}$. Moreover, if
$z_i$, $i \in \mathbb{N}$, are i.i.d. random variables with
the same distribution as $z_{uv}$, then
\[
\sqrt{n} \left( \frac{1}{n} \sum_{i=1}^n z_i -
\Sigma_\alpha (x, \sigma) (u,v) \right) \xrightarrow{d}
\mathcal{N} (0, \sigma_{uv}^2) \,,
\]
as $n \to \infty$, where convergence is in distribution and
$\mathcal{N} (0,\sigma_{uv}^2)$ is normally distributed with
mean zero and variance $\sigma_{uv}^2$.
\end{theorem}

\begin{proof}

We show that $z_{uv}$ has finite second moment. From
\eqref{E:randomz} and \eqref{E:norm},
\begin{equation} \label{E:zvariance}
\begin{split}
\int z^2_{uv} \, \alpha (dy) &\leqslant \|u\| \|v\|
\int \|y-x\|^4 K^2 (x,y, \sigma) \, \alpha (dy)  \\
&\leq \frac{\sigma^4  \|u\| \|v\|}{C^2_d (\sigma)} 
 \int \frac{\|y-x\|^4}{\sigma^4}
f^2 \left(\frac{\|y-x\|^2}{\sigma^2} \right) \, \alpha (dy) \\
&\leq \frac{\sigma^4  \|u\| \|v\|}{C^2_d (\sigma)} \, C^2\,.
\end{split}
\end{equation}
The last inequality follows from condition (c) in Definition
\ref{D:kernel} that ensures that $r^2 f^2(r) < C^2$, for any
$r>0$. The theorem now follows from a direct application
of the classical CLT. 
\end{proof}

\begin{remark}
Note that \eqref{E:zvariance} implies that if
$\|u\| = \|v\| =1$, then
\begin{equation}
\sigma^2_{uv} = \int z^2_{uv} \, \alpha (dy)
- \left(\expect{z_{uv}} \right)^2
\leq \frac{C^2 \sigma^4}{C^2_d (\sigma)} \,,
\end{equation}
giving a uniform bound on the variance of $z_{uv}$ over
$x \in \real^d$ and $u,v \in \sphere^{d-1}$.
\end{remark}

%-----------------------------

\section{Multiscale Fr\'{e}chet Functions} \label{S:frechet}

The mean of a random vector $y \in \real^d$ is a simple
and yet oftentimes informative, ``one-element'' summary of
the distribution of $y$. If $y$ has finite second moment and
is distributed according to the probability measure $\alpha$,
then the mean may be characterized more
geometrically as the unique minimizer of the Fr\'{e}chet
function
\begin{equation}
F_\alpha (x) = \expect{\|y-x\|^2} = \int \|y-x\|^2 \, \alpha (dy) \,,
\end{equation}
which measures the spread of $y$ about
$x \in \real^d$.
%Equivalently, $F_\alpha$ may be defined as the
%trace of the covariance tensor field of $y$; that is,
%\begin{equation}
%F_\alpha (x) = \text{tr} \, \Sigma_\alpha (x) \,,
%\end{equation}
%with $\Sigma_\alpha$ as defined in \eqref{E:field}. 
The mean, however, is not as effective
for complex distributions of practical interest such as multimodal
distributions or those supported in
nonlinear subspaces. In this section, we introduce a multiscale
analogue of the Fr\'{e}chet function that is rich in information
about the shape of the distribution of $y$. At each fixed scale,
the local minima of the function may be viewed as localized
analogues of the
mean, as illustrated in examples below. However, instead of
just focusing on the local extrema, we take the view that it is
more informative to investigate the behavior of the
full multiscale Fr\'{e}chet function, as this lets us
uncover more information about the distribution of $y$.

\begin{definition} \label{D:frechet}
Let $f \colon [0, \infty) \to \real$ be as in Definition
\ref{D:kernel} with associated kernel
$K \colon \real^d \times \real^d \times (0, \infty) \to \real$.
The {\em multiscale Fr\'{e}chet function}
$V_\alpha \colon \real^d \times (0, \infty) \to \real$ is
defined as
\[
V_\alpha (x, \sigma) := \int \|y-x\|^2 K(x,y,\sigma) \, \alpha(d y)\,.
\] 
\end{definition}
%---------------------
\begin{proposition} \label{P:trace}
For each $\sigma > 0$, the multiscale Fr\'{e}chet
function satisfies
\[
V_\alpha (x, \sigma) = \mathrm{tr} \, \Sigma_\alpha (x, \sigma) \,.
\]
\end{proposition}
%--------------------
\begin{proof}
Let $\{e_1, \ldots, e_d\} \subset \real^d$ be an orthonormal
basis. Then,
\begin{equation}
\|y-x\|^2 = \sum_{i=1}^d \inner{y-x}{e_i}^2 = 
\sum_{i=1}^d (y-x) \otimes (y -x) (e_i, e_i) \,.
\end{equation}
Hence,
\begin{equation}
\begin{split}
V_\alpha (x, \sigma) &= 
\sum_{i=1}^d \int  (y-x) \otimes (y -x) (e_i, e_i)
K(x,y,\sigma) \, \alpha (dy) \\
&= \sum_{i=1}^d \left( \int  (y-x) \otimes (y -x)
K(x,y,\sigma) \, \alpha (dy) \right) (e_i, e_i) \\
&= \sum_{i=1}^d \Sigma_\alpha (x, \sigma) (e_i, e_i) 
= \text{tr} \, \Sigma_\alpha (x, \sigma) \,,
\end{split}
\end{equation}
as claimed.
\end{proof}
%------------------
\begin{corollary}[Stability] \label{C:stability}
Let $f \colon [0, \infty) \to \real$ be as in Definition \ref{D:kernel}
with multiscale kernel $K$. Suppose that $f$ is differentiable
and there exists a constant $A>0$ such that
$r^{3/2} \, |f'(r)| \leq A$, $\forall r \geq 0$. Then,
there is a constant $A_f >0$, that depends only on $f$, such
that
\[
\sup_{x\in\real^d} \left| V_\alpha (x, \sigma) -
V_\beta (x, \sigma) \right| \leq 
\frac{\sigma d A_f}{C_d(\sigma)} \, \dwassone (\alpha,\beta),,
\]
for any $\alpha,\beta\in\borel_1 (\real^d)$ and any $\sigma>0$.
\end{corollary}
\begin{proof}
The result follows from Proposition \ref{P:trace},
Theorem \ref{T:stab} and the fact that for any $d \times d$
matrix $X$, $| \text{tr} \, X| \leq d \, \|X\|$, where
$\| X \|$ is the Frobenius norm of $X$.
\end{proof}

Similarly, Corollary \ref{C:consistency} and
Proposition \ref{P:trace} yield the following consistency
result for multiscale Fr\'{e}chet functions.
 
\begin{corollary}[Consistency]  
Suppose that $\alpha \in \borel_1 (\real^d)$. Let $y_i \in \real^d$,
$i \in \mathbb{N}$, be i.i.d. random variables with distribution
$\alpha$ and $K$ a multiscale kernel as in Theorem \ref{T:stab}.
Then, for each fixed $\sigma > 0$,
\begin{equation*}
\sup_{x \in \real^d}
\left| V_\alpha(x,\sigma) - V_{\alpha_n}(x,\sigma) \right|
\xrightarrow{n\uparrow\infty} 0
\end{equation*} almost surely.
\end{corollary}

%-----------------------
The following result about convergence of multiscale
Fr\'{e}chet functions are immediate consequences of
Corollary \ref{C:rates} and Corollary \ref{C:stability}.

\begin{corollary} 
Let $f$ be as in Theorem \ref{T:stab} and $\sigma > 0$.
Suppose that $\alpha \in \borel_3 (\real^d)$ and
$y_i$, $i \in \mathbb{N}$, are i.i.d. random
variables with distribution $\alpha$.  Then, there is
a constant $\beta> 0$, that depends only on $d$,
such that
\begin{equation*}
\begin{split}
\expect{\, \sup_{x \in \real^d}\left| V_\alpha(x,\sigma)
- V_{\alpha_n}(x,\sigma) \right| }&\leq \\
\frac{\sigma d A_f \beta}{C_d (\sigma)} m_3 (\alpha) &\cdot 
\begin{cases}
n^{-\frac{2}{3}} + n^{-\frac{1}{2}} & \text{if $d=1$;}\\
n^{-{\frac{2}{3}}} + n^{-\frac{1}{2}}\log(1+n) & \text{if $d=2$;}\\
n^{-\frac{2}{3}} +n^{-\frac{1}{d}} & \text{if $d\geq 3$.}
\end{cases}
\end{split}
\end{equation*}
\end{corollary}
%-------------
\begin{remark}
Analogous stability and consistency results for the truncation kernel
follow from Theorem \ref{T:stab-trunc}, Corollary \ref{C:ctf-infty}
and Corollary \ref{C:trunc-stab}.
\end{remark}
%------------------------
%\begin{corollary} \label{C:vtruncation}
%Let $\alpha$ be a probability measure such that all balls
%$B\subset \mathbb{R}^d$ are continuity sets of $\alpha$.
%Then, for the truncation kernel with window size $\sigma$, 
%\[
%\left| V_{\alpha_n}(x,\sigma) - V_{\alpha}(x,\sigma)
%\right| \xrightarrow{n\uparrow\infty} 0
%\]
%almost surely, for any fixed $x\in \real^d$.
%\end{corollary}
%------------------------
For more general kernels, the following pointwise
central limit theorem holds. For fixed $x \in \real^d$
and $\sigma > 0$, let
\begin{equation}
t (y) = \|y-x\|^2 K(x,y,\sigma),
\end{equation}
whose expected value is
$\expect{t} = V_\alpha (x, \sigma)$. As in
Theorem \ref{T:clt}, the variance of $t$ is finite and
denoted $\sigma^2_t$.
%-----------------------------
\begin{theorem}[Central Limit] \label{T:frechetclt}

Let $f$ be as in Definition \ref{D:kernel}. If $t_i \in \real$,
$i \in \mathbb{N}$, are i.i.d. random variables with the same
distribution as $t$, then
\[
\sqrt{n} \left( \frac{1}{n} \sum_{i=1}^n t_i -
V_\alpha (x, \sigma) \right) \xrightarrow{d}
\mathcal{N} (0, \sigma_t^2) \,,
\]
as $n \to \infty$, where convergence is in distribution and
$\mathcal{N} (0,\sigma_V^2)$ is normally distributed with
mean zero and variance $\sigma_t^2$.
\end{theorem}
%-----------------------
Multiscale Fr\'{e}chet functions
not only give stable representations of probability measures,
but any probability measure $\alpha$ may be fully recovered
from its multiscale Fr\'{e}chet function associated with the
Gaussian kernel, as the following result shows.
%-------------------------------
\begin{proposition} \label{P:fourier}
Let $\sigma > 0$ be fixed.
Any probability measure $\alpha$ is completely determined by the
Fr\'{e}chet function $V_\alpha (\cdot, \sigma)$ associated
with the Gaussian kernel at scale $\sigma$.
\end{proposition}
%------------------------
\begin{proof}

Let $h_\sigma \colon \real^d \to \real$ be given by
\begin{equation}
h_\sigma (x) = \frac{\|x\|^2}{(2 \pi \sigma^2)^{d/2}}
\exp{\left(- \frac{\|x\|^2}{2 \sigma^2}\right)} \,.
\end{equation}
Then, for the Gaussian kernel, we may express the
multiscale Fr\'{e}chet function as the convolution
$V_\alpha (x, \sigma) = (h_\sigma \ast \alpha) (x)$.
Under Fourier transform,  for each fixed $\sigma > 0$, we
obtain
\begin{equation}
\widehat{V}_\alpha (\xi, \sigma) =
\widehat{h}_\sigma (\xi) \, \phi_\alpha (-2 \pi \xi) \,,
\end{equation}
where $\phi_\alpha$ is the characteristic function of
$\alpha$ defined as $\phi_\alpha (\xi) =
\int e^{i \inner{x}{\xi}} \, \alpha (dx)$. Therefore,
\begin{equation} \label{E:phi}
\phi_\alpha (\xi) = \widehat{V}_\alpha
(-\xi / 2 \pi, \sigma) / \, \widehat{h}_\sigma (-\xi / 2 \pi)
\end{equation}
provided that $\widehat{h}_\sigma (-\xi / 2 \pi) \ne 0$.
A calculation shows that
\begin{equation}
\widehat{h}_\sigma (- \xi / 2 \pi) = 
\sigma^2 \left(d - \frac{\sigma^2 \|\xi\|^2}{\pi} \right)
\exp{\left(- \frac{\sigma^2 \|\xi\|^2}{2 \pi} \right)} \,,
\end{equation}
which only vanishes at points $\xi$ on the sphere of radius
$\rho_\sigma = \sqrt{\pi d}/ \sigma$ about the origin.
Thus, \eqref{E:phi} implies that we can recover
$\phi_\alpha (\xi)$ from  $\widehat{V}_\alpha (\cdot, \sigma)$,
if $\|\xi\| \ne \sqrt{\pi d}/ \sigma$. By continuity, we can recover
$\phi_\alpha (\xi)$, for any $\xi$. The claim now follows
from the fact that the characteristic function $\phi_\alpha$
determines $\alpha$ \cite{dudley}.
\end{proof}
%-----------------------
The following examples illustrate how information about
the shape of data can be extracted from multiscale
Fr\'{e}chet functions.
%-----------------------
\begin{example} \label{E:frechet1}
We consider $n=400$ data points distributed into two clusters
of 200 points, each sampled from a Gaussian of variance $0.36$
centered at different points. The data points are plotted in
blue in Fig.\,\ref{F:frechet1}(a), which also shows
the empirical Fr\'{e}chet function $V_n$ at scale $\sigma = 3$. The
local minima of $V_n$ captures what is perceived as the ``centers''
of the two clusters at that scale. However, more information about
the data distribution can be uncovered from $V_n$. For example,
the local minima may be viewed as attractors of the (negative)
gradient field $- \nabla V_n$, indicated by the arrows in the figure.
The stable manifold of each attractor, which comprises points that
move toward the attractors under the associated flow may be
viewed as clusters inferred from the data at that scale. These
clusters are delimited by the repellers of the system, which
correspond to the local maxima of $V_n$. Fig.\,\ref{F:frechet1}(b)
shows how $V_n$ varies across scales, highlighting the bifurcation
of the attractors (in red) and repellers (in green) as $\sigma$ changes.
%-------------------------
\begin{figure}[h!]
\begin{center}
\begin{tabular}{cc}
\includegraphics[width=0.45\linewidth]{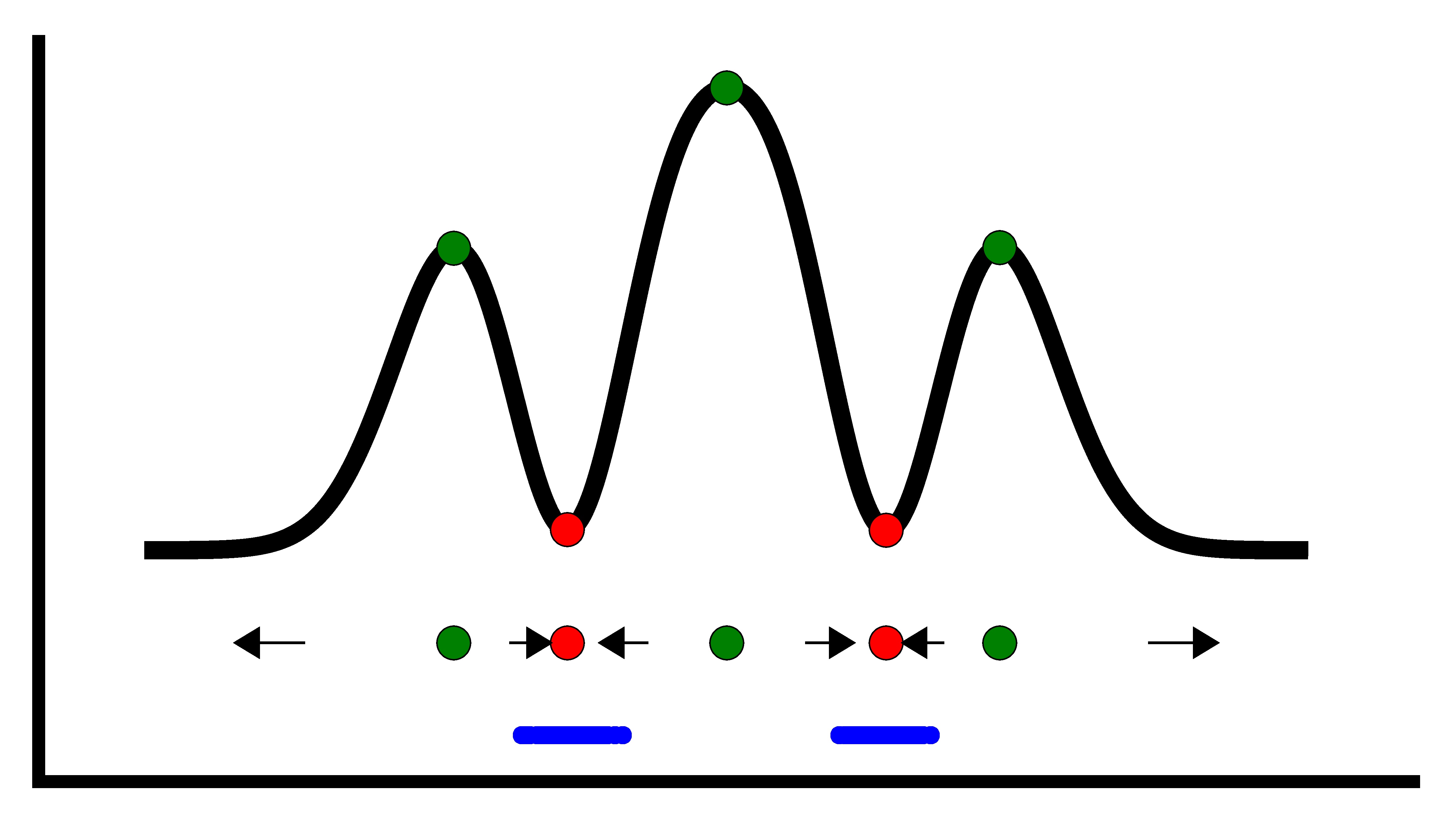} 
\qquad & \qquad
\includegraphics[width=0.3\linewidth]{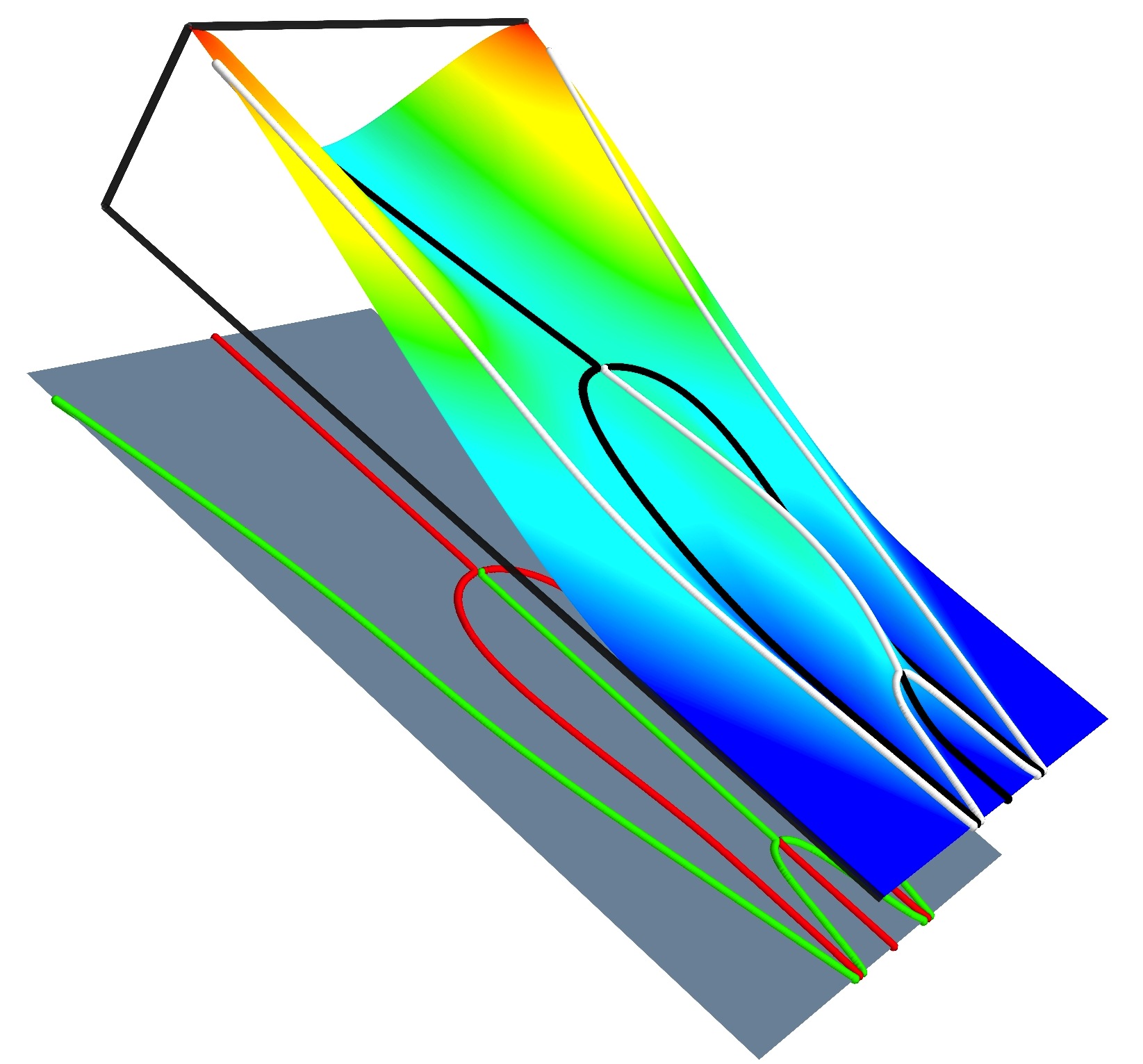} \\
(a) \qquad & \qquad (b)
\end{tabular}
\end{center}
\caption{(a) Fr\'{e}chet function for data on the line
(highlighted in blue) computed with the Gaussian kernel at
scale $\sigma = 3$; (b) Fr\'{e}chet function across scales.}
\label{F:frechet1}
\end{figure}
%--------------------------
In data analysis, such bifurcation diagrams may find 
several applications. For example, if the data
represent the distribution of some phenotypic trait for
two species that have evolved from a single group, the
multiscale Fr\'{e}chet function and the associated
bifurcation diagram let us create an evolutionary model
for the trait from the observed data.
\end{example}
%--------------------------

\begin{example} \label{E:frechet2}
Here we consider the dataset in $\real^2$ shown in panel (a)
of Fig.\,\ref{F:frechet2}. Panels (b)--(h) show the Fr\'{e}chet
function for the Gaussian kernel calculated at increasing scales.
The gradient field $- \nabla V_n$ at scale $\sigma = 2.25$ is
depicted in panel (a) of Fig.\,\ref{F:2dfield} along with the
two attractors $p_1$ and $p_2$, and their stable manifolds
that were estimated numerically. 
%-----------------------
\begin{figure}[h!]
\begin{center}
\begin{tabular}{ccccc}
\begin{tabular}{c}
\includegraphics[width=0.14\linewidth]{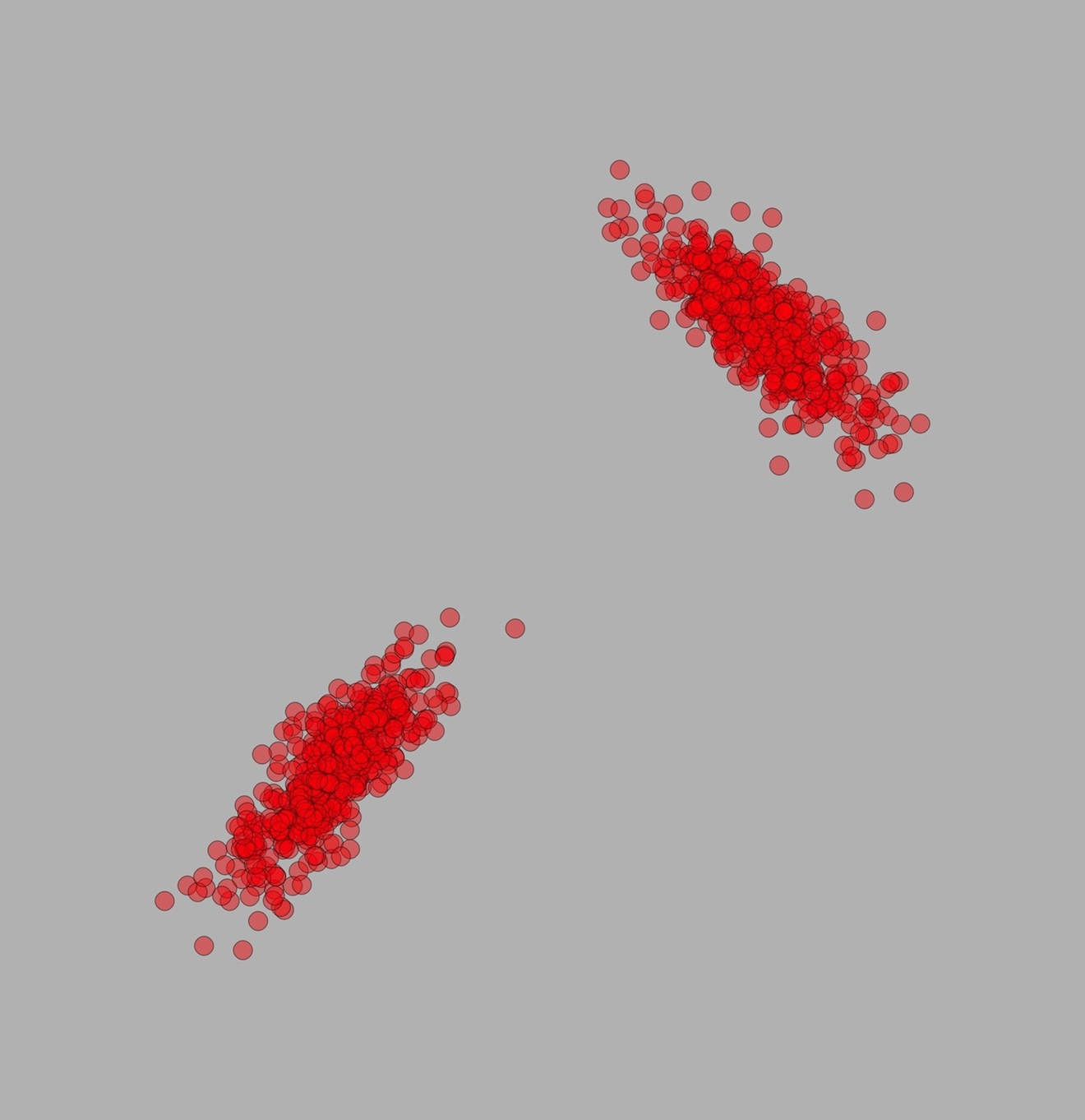}
\end{tabular} &
\begin{tabular}{c}
\includegraphics[width=0.14\linewidth]{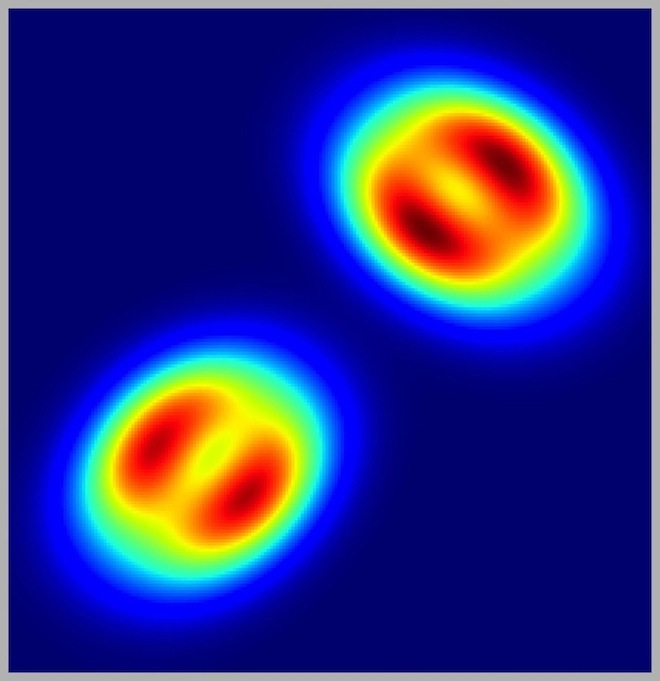} 
\end{tabular} & 
\begin{tabular}{c}
\includegraphics[width=0.14\linewidth]{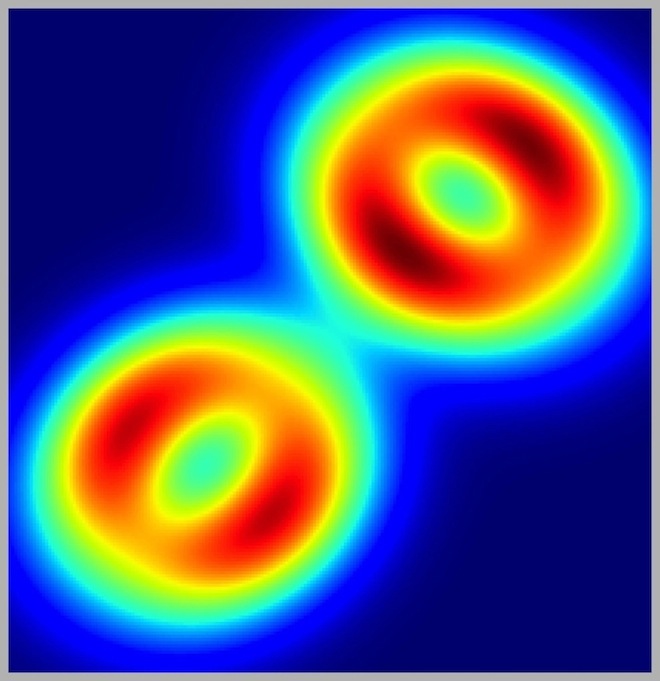} 
\end{tabular} & 
\begin{tabular}{c}
\includegraphics[width=0.14\linewidth]{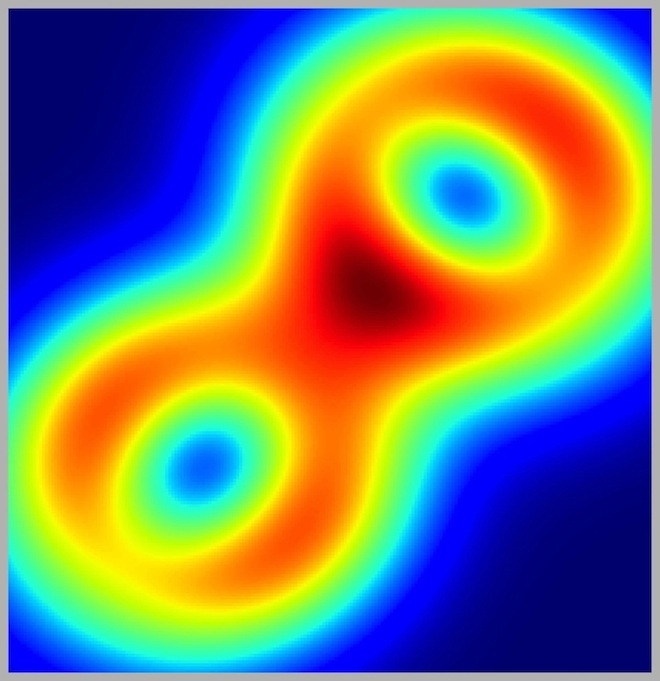} 
\end{tabular} & 
\begin{tabular}{c}
\includegraphics[width=0.14\linewidth]{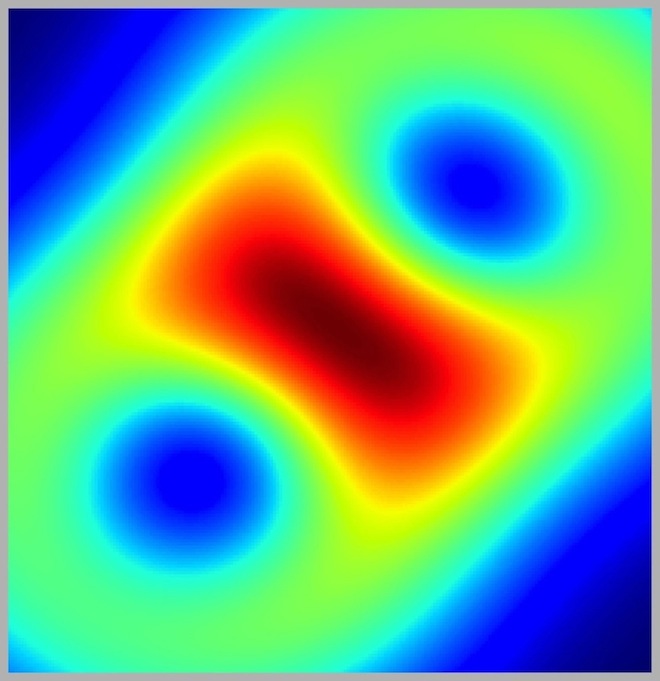} 
\end{tabular} \\
data & $\sigma = $1.0 & $\sigma = 1.5 $ & $\sigma = 2.0$ &
$\sigma = 3.0 $
\vspace{0.1in} \\
\begin{tabular}{c}
\includegraphics[width=0.14\linewidth]{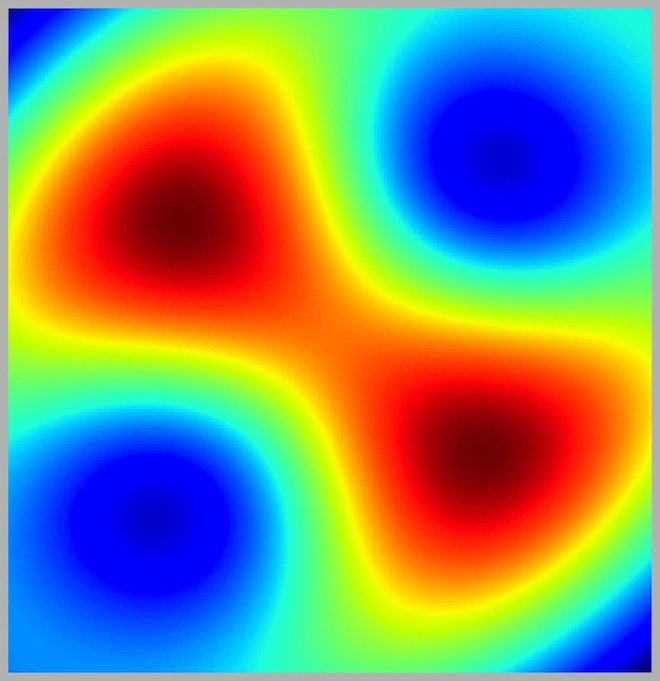} 
\end{tabular} & 
\begin{tabular}{c}
\includegraphics[width=0.14\linewidth]{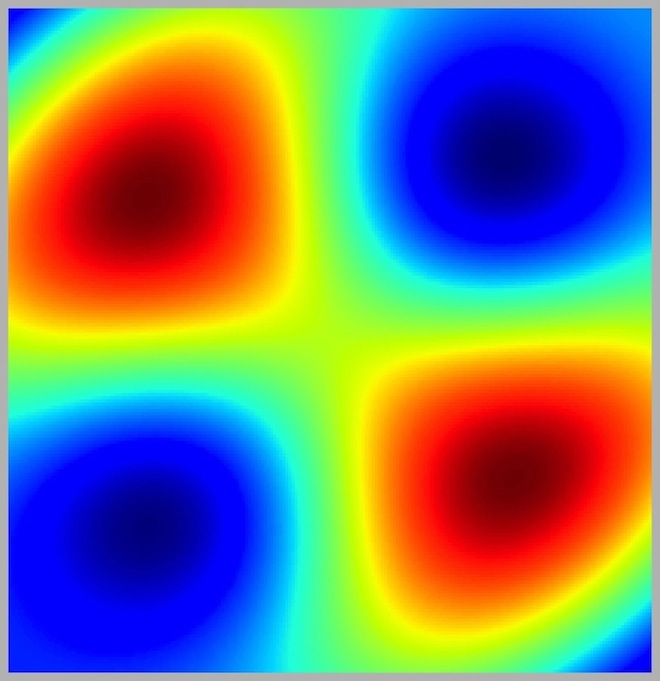} 
\end{tabular} & 
\begin{tabular}{c}
\includegraphics[width=0.14\linewidth]{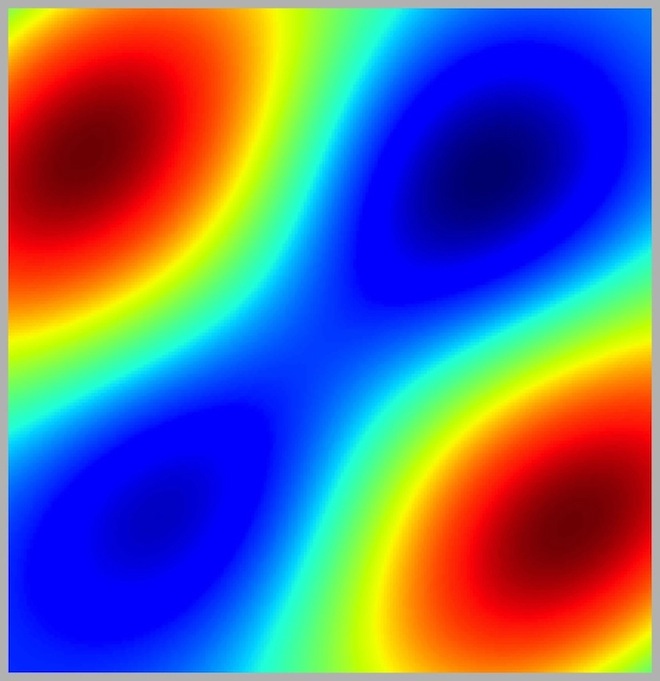} 
\end{tabular} & 
\begin{tabular}{c}
\includegraphics[width=0.14\linewidth]{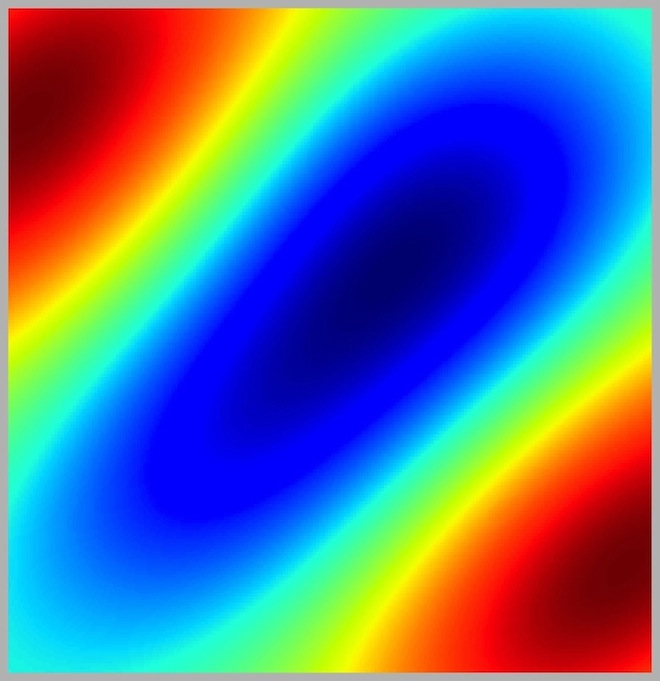} 
\end{tabular} & 
\begin{tabular}{c}
\includegraphics[width=0.14\linewidth]{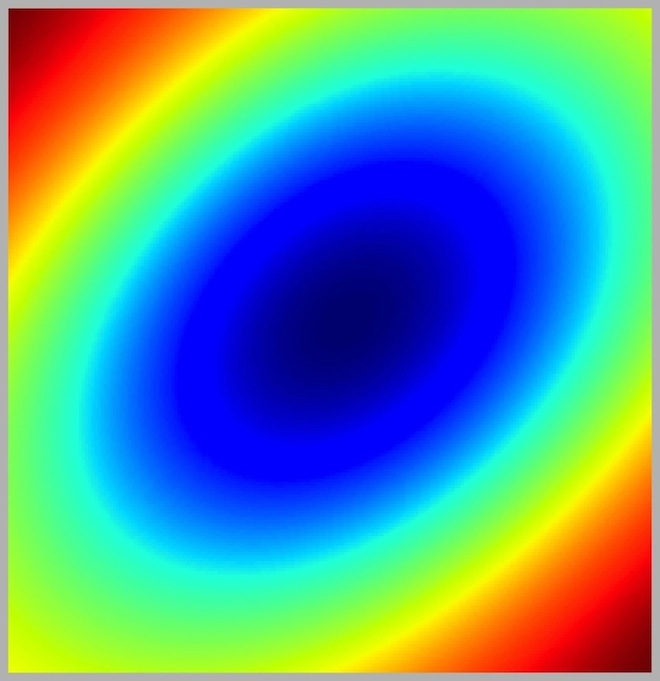} 
\end{tabular} \\
$\sigma = 4.5$ & $\sigma = 5.0 $ & $\sigma = 6.0 $ &
$\sigma = 7.0$ & $\sigma = 9.0$
\end{tabular}
\end{center}
\caption{Heat maps of the multiscale Fr\'{e}chet function
for 2D data at increasing scales computed with the
Gaussian kernel.}
\label{F:frechet2}
\end{figure}
%-----------------------
\begin{figure}[h!]
\begin{center}
\begin{tabular}{cc}
\begin{tabular}{c}
\includegraphics[width=0.32\linewidth]{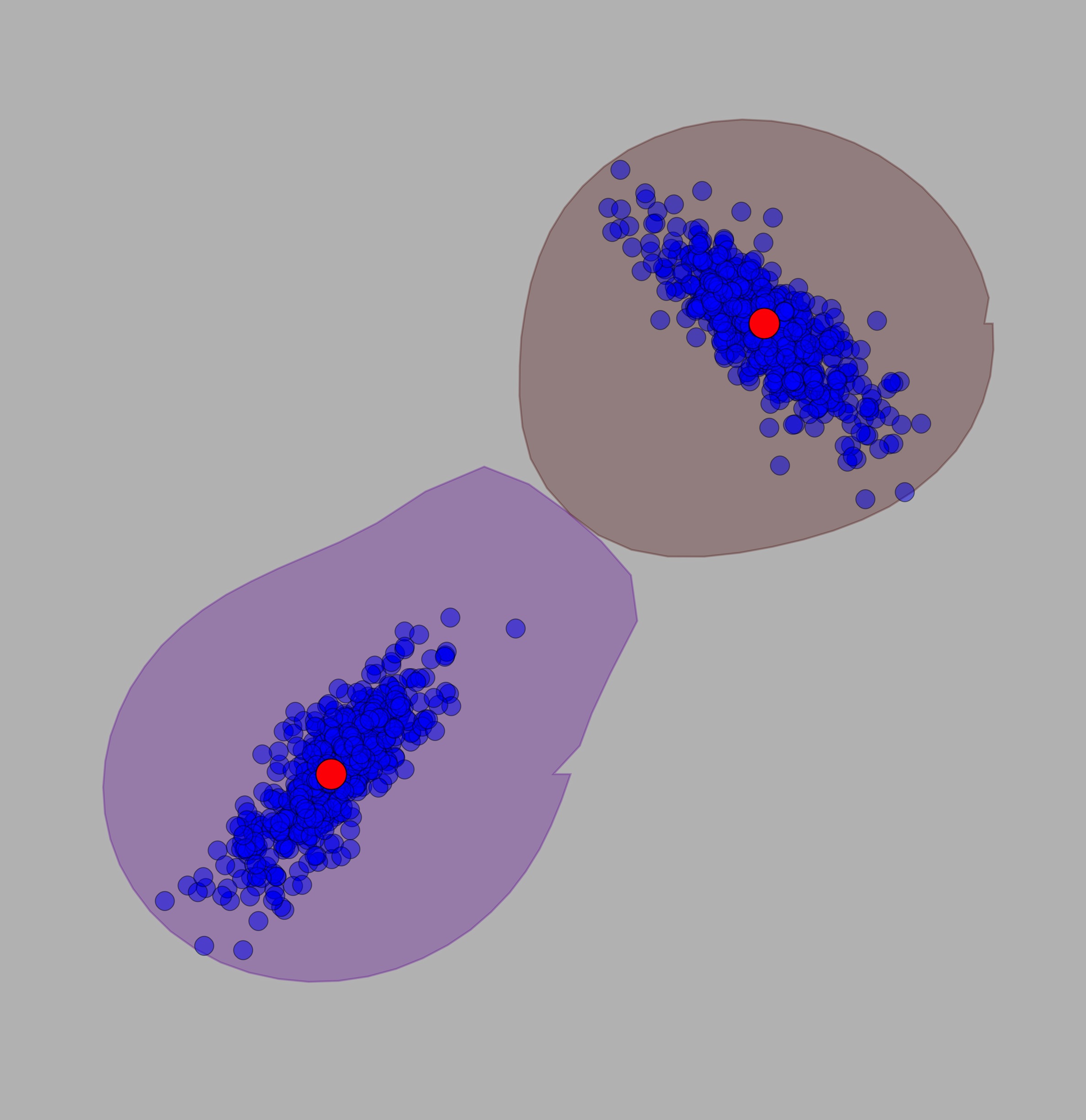}
\end{tabular}
\qquad & \qquad
\begin{tabular}{c}
\includegraphics[width=0.35\linewidth]{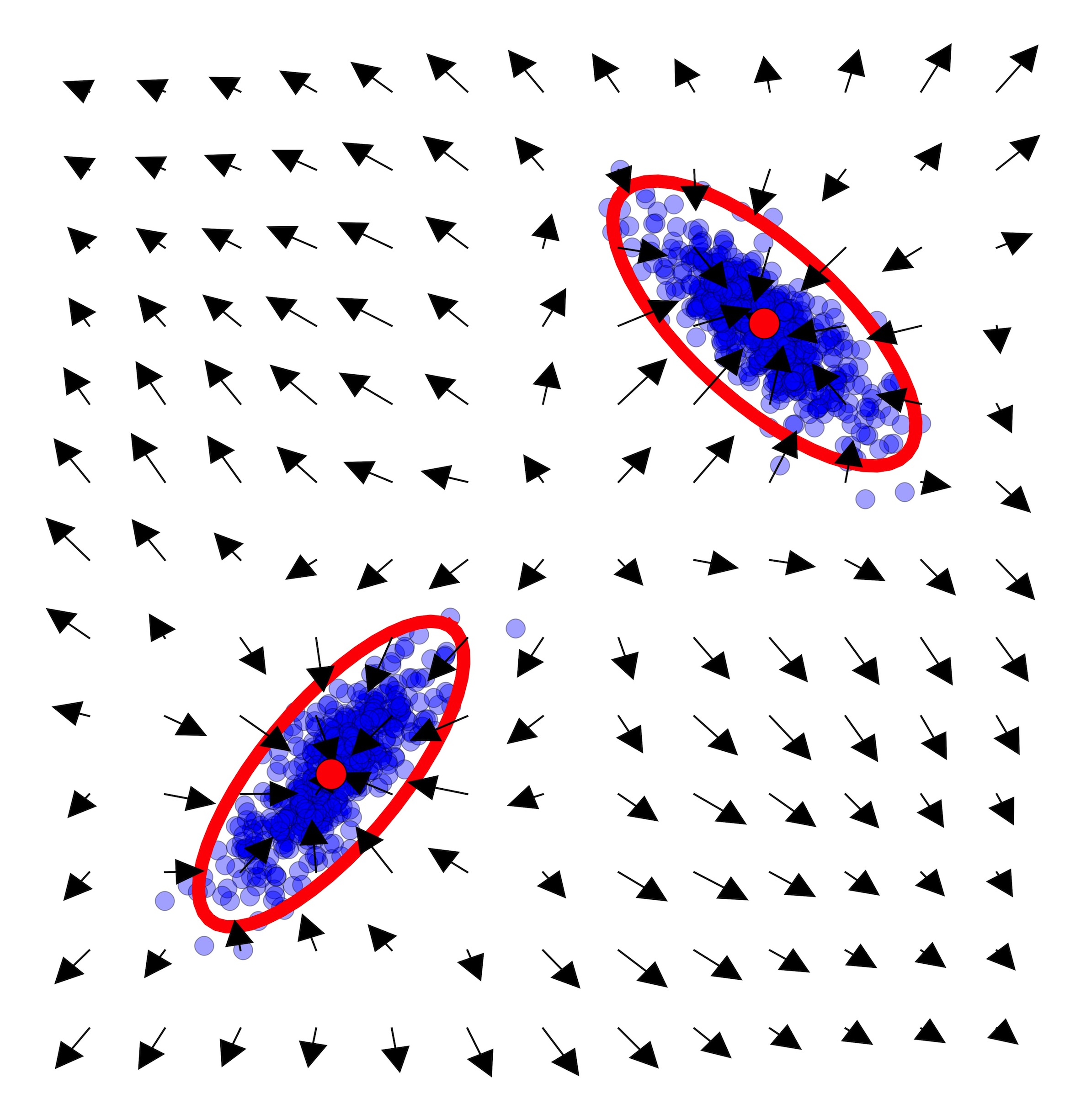}
\end{tabular}
\\
(a) \qquad & \qquad (b) 
\end{tabular}
\end{center}
\caption{(a) 2D data, attractors and their stable manifolds at
a fixed scale ($\sigma=2.25$); (b) gradient vector field and
covariance tensors at the attractors.}
\label{F:2dfield}
\end{figure}
%-----------------------------------
The stable manifolds may be viewed as estimations
at scale $\sigma = 2.25$ of clusters of the underlying probability
measure $\alpha$ from which the data was sampled. Panel (b)
shows the gradient field and the covariance tensors at the attractors
depicted as ellipses with principal radii proportional to the
square root of the eigenvalues of the covariance matrix. This
may be viewed as a localized analogue of principal component
analysis (PCA) that is able to uncover geometry that is not
detectable with standard PCA. Analysis of the spectra
of $\Sigma_n (p_i, \sigma)$, $i =1,2$, suggests that the data
is organized around two one-dimensional clusters, whereas
standard PCA is not sensitive to the local dimensionality because
of the orientation  of the clusters.

\end{example}

These examples are intended as proof-of-concept
illustrations. Topological and other methods will
be explored in forthcoming work for extraction of
structural information from $V_n$.

%-----------------------------------

\section{Hierarchical Manifold Clustering} \label{S:clustering}

Clustering is a central theme in pattern analysis with a rich
history; cf. \cite{jain}. One of the most studied forms of the problem
is that of partitioning a dataset into various subsets if there
is some form of spatial separation of the data into subgroups.
Motivated by problems in such areas as computer vision and
video analysis, cf.\,\cite{vidal07}, there has been a
growing interest in clustering data that are organized as a finite
union of possibly intersecting subspaces that have some
special geometric structure \cite{chen2009foundations,arias2011clustering,arias2013spectral,soltanolkotabi2014robust}. As illustrated in
Fig.\,\ref{F:clusters}, the data may
consist of noisy samples from an arrangement of (affine) linear
subspaces of a Euclidean space such as a collection of lines in
a plane, or an arrangement of lines and planes in $\real^3$. More
generally, the clusters may comprise a finite collection
of possibly non-linear, smooth submanifolds of a Euclidean space
that intersect transversely. Here we propose an approach to
manifold clustering based on CTFs. The basic idea is to use
covariance fields to incorporate directional information at
each data point. Formally, this is achieved via a section of the
tensor bundle $\real^d \times (\real^d \otimes \real^d)$ over
$\real^d$, as follows. Given a probability measure $\alpha$
and a multiscale kernel, let $\Sigma_\alpha (x, \sigma)$ be
the associated CTF. For each $\sigma > 0$, consider the section
$\iota_{\alpha; \sigma} \colon \real^d \to \real^d \times (\real^d
\otimes \real^d)$ given by
$x \mapsto \left( x, \Sigma_\alpha (x, \sigma) \right)$. 
On the total space of the tensor bundle, define the metric
\begin{equation}
\| (x, \Sigma) - (x', \Sigma')\|_\gamma = 
\left( \Vert \Sigma - \Sigma' \Vert^2
+ \gamma^2 \Vert x - x' \Vert^2 \right)^{1/2} \,,
\end{equation}
where $x, x' \in \real^d$, $\Sigma, \Sigma'
\in \real^d  \otimes \real^d$, and $\gamma \geq 0$
is a parameter that balances the contributions of the spatial
and tensor components. Note that  $\|\cdot\|_0$ only defines
a pseudo-metric since $\|\cdot\|_0$ disregards ``horizontal''
distances.

For any subset $X \subseteq \real^d$, we denote
by $\mathbb{X}_{\alpha; \gamma, \sigma}$ the metric
space $(X, d_{\alpha; \gamma, \sigma})$, where  
\begin{equation} \label{E:metric}
d_{\alpha; \gamma, \sigma} (x, x') = 
\left\| \iota_{\alpha; \sigma} (x) -
\iota_{\alpha; \sigma} (x') \right\|_\gamma \,.
\end{equation}
For a dataset $A = \{a_1, \ldots, a_n\} \subset \real^d$,
the proposed clustering method is based on the single-linkage
method \cite{sibson73} applied to the finite metric space
$\mathbb{A}_{\alpha_n; \gamma, \sigma}$ associated with
the empirical measure $\alpha_n = n^{-1} \sum_{i=1}^n \delta_{a_i}$.
Equivalently, clustering is based
on the $n \times n$ affinity matrix $D$ whose $(i,j)$-entry is
\begin{equation} \label{E:affinity}
d_{ij} = d_{\alpha_n; \gamma,\sigma} (a_i, a_j) \,.
\end{equation}
Recall that single linkage
on a finite metric space $\mathbb{A}=(A, d_A)$ starts from $n$ clusters,
each a singleton $\{a_i\}$, $1 \leqslant i \leqslant n$,
sequentially merging the closest clusters until all data points
coalesce into a single cluster. Closeness of two clusters, say
$A_1, A_2 \subset A$, is measured by the inter-cluster distance
\begin{equation}
d_{sl}(A_1, A_2)=\min_{a \in A_1 , a' \in A_2} d_A (a, a') \,,
\label{E:distance}
\end{equation}
We choose single linkage because it yields stable dendrograms,
as expounded below, under assumptions on the
probability measure from which the data is sampled that are not
very restrictive. Combined with our stability and consistency
results for covariance fields, this guarantees that the 
manifold clustering method is stable at all stages. 

\subsection{Dendrogram Stability}

We denote a metric space by $\mathbb{X} = (X,d_X)$.
An ultrametric space is a pair $(X,u_X)$, where $u_X \colon
X\times X \to \mathbb{R}^+$ is a metric on $X$ that satisfies
the \emph{strong triangle inequality}
\begin{equation}
u_X(x,x')\leq \max \left\{ u_X(x,x''),u_X(x'',x') \right\} \,,
\end{equation}
for all $x,x',x''\in X$. Any such function $u_X$  is called
an \emph{ultrametric} on $X$. 

As proved in \cite{memoli10}, dendrograms over a finite set
$X$ are in structure-preserving, bijective correspondence with
ultrametrics on $X$. In this formulation, a hierarchical clustering method
can be regarded as a map $\mathcal{H}:\mathcal{M}\rightarrow\mathcal{U}$
from finite metric spaces into finite ultrametric spaces. Henceforward,
$\mathcal{H}$ will denote the map given by single linkage hierarchical
clustering. It is known \cite{memoli10} that if $\mathbb{X} = (X,d_X)
\in \mathcal{M}$, then $\mathcal{H}(\mathbb{X})=(X,u_X)$ is given by 
\begin{equation}
u_X(x,x') = \min_{x=x_0,\ldots,x_r=x'}\max_{i}d_X(x_i,x_{i+1}) \,.
\end{equation}
The minimum above is taken over all finite ordered sequences
$x_0, x_1,\ldots,x_r$ of points in $X$ such that $x_0=x$ and $x_r=x'$.
%------------------------------------------
If $x, x' \in X$, then $u_X (x,x')$ may be interpreted as the
dendrogram level at which the clusters containing $x$ and $x'$
first merge. This is known as the {\em cophenetic\/}
distance between $x$ and $x'$.

The main goal of this section is to formulate and prove
stability of the map $\mathcal{M} \ni
\mathbb{X}\mapsto\mathcal{H}(\mathbb{X})\in \mathcal{U}$.
The question of stability of single linkage clustering
can be approached using ideas related to the Gromov-Hausdorff
distance \cite{burago}, as follows. A correspondence $R$ between
two sets $X$ and $Y$ is a subset of $X\times Y$ such that
$\pi_1(R)=X$ and $\pi_2(R)=Y$, where $\pi_1$ and $\pi_2$
denote projections onto the first and second factors. Given
$X$ and $Y$, we denote by $\mathcal{R}(X,Y)$ the set of
all correspondences between $X$ and $Y$.
%----------------------------
\begin{definition} Let $\mathbb{X}$ and $\mathbb{Y}$ be
compact metric spaces.
\begin{itemize}
\item[(i)]
The {\em distortion} of a correspondence $R$ between
$\mathbb{X}$ and $\mathbb{Y}$ is defined by 
\[
\text{dis}\, (R;\mathbb{X},\mathbb{Y}) := 
\max_{(x,y),(x',y')\in R} \left| d_X(x,x)-d_Y(y,y') \right|.
\]

\item[(ii)] The {\em Gromov-Hausdorff} distance between
$\mathbb{X}$ and $\mathbb{Y}$ is given by 
\[
d_{GH}(\mathbb{X},\mathbb{Y}):=\frac{1}{2}\inf_{R}
\text{dis}\,(R; \mathbb{X},\mathbb{Y}),
\]
where the infimum is taken over all correspondences between
$\mathbb{X}$ and $\mathbb{Y}$.
\end{itemize}
\end{definition}
%-----------------------------
The following stability result is a generalization of
\cite[Proposition 26]{memoli10}.
%----------------------------
\begin{proposition} \label{P:dendro}
For any $\mathbb{X}, \mathbb{Y} \in \mathcal{M}$ and
any correspondence $R \in \mathcal{R}(X,Y)$, 
\[
\mathrm{dis}(R;\mathcal{H}(\mathbb{X}),\mathcal{H}
(\mathbb{Y}))\leq \mathrm{dis}(R;\mathbb{X},\mathbb{Y}) \,.
\]
As a consequence,
$d_{GH}(\mathcal{H}(\mathbb{X}),\mathcal{H}(\mathbb{Y}))
\leq d_{GH}(\mathbb{X},\mathbb{Y})$.
\end{proposition}
%-----------------------------
\begin{remark}
The claim of the proposition may be written, equivalently,
as follows. If $u_X$ and $u_Y$ denote the ultrametrics
produced by single linkage hierarchical clustering on
$\mathbb{X}$ and $\mathbb{Y}$, then
\begin{equation}\label{E:stab-gh}
|u_X(x,x')-u_Y(y,y')|\leq \max_{(x,y),(x',y')\in
R}|d_X(x,x')-d_Y(y,y')| \,,
\end{equation}
for any correspondence $R$ between $X$ and $Y$ and
all $(x,y),(x',y')\in R$. 
\end{remark}
%----------------------------
\begin{proof}[Proof of Proposition \ref{P:dendro}]
We prove (\ref{E:stab-gh}). Given a correspondence
$R \in \mathcal{R}(X,Y)$ and $(x,y),(x',y')\in R$, let
$x=x_0,x_1,\ldots,x_n=x'$ in $X$ be such that
$\max_i d_X(x_i,x_{i+1}) = u_X(x,x')$. Let $y_0=y$,
$y_n=y'$ and choose $y_1,\ldots,y_{n-1}\in Y$ such that
$(x_i,y_{i})\in R$ for all $i=1,\ldots,n-1$. This is possible since
any correspondence $R$ satisfies $\pi_1(R)=X$.
Notice that 
\begin{equation} \label{E:ubound}
\begin{split}
u_Y(y,y') &\leq \max_{i} d_Y(y_{i},y_{i+1}) \\
&\leq \max_{i} \left( d_X(x_{i},x_{i+1}) + |d_X(x_i,x_{i+1}) 
- d_Y(y_i,y_{i+1})| \right)\\
&\leq \max_i d_X(x_i,x_{i+1}) + \max_{(x,y),(x',y')\in R}
| d_X(x,x')-d_Y(y,y')| \\
&=u_X(x,x') + \max_{(x,y),(x',y')\in R} | d_X(x,x')-d_Y(y,y')|.
\end{split}
\end{equation}
The claim follows since \eqref{E:ubound} also holds
if we reverse the roles of $X$ and $Y$.
\end{proof}

\begin{lemma}\label{lemma:sigma-ab} 
Let $\alpha,\beta\in\mathcal{P}_\infty(\real^d)$ and $\sigma>0$.
If a kernel satisfies the conditions of Lemma \ref{L:kernel}, then 
\begin{equation*}
\sup_{(a,b) \in R_\mu}
\left\|\Sigma_\alpha(a,\sigma) - \Sigma_\beta(b,\sigma) \right\|
\leq \frac{2 A_f \sigma}{C_d(\sigma)} \,
\sup_{(a,b) \in R_\mu} \|a-b\|\,,
\end{equation*}
for any coupling $\mu\in\Gamma(\alpha,\beta)$, where
$R_\mu := \text{supp} \, [\mu]$ and $A_f>0$ is as in 
Lemma \ref{L:kernel}.
\end{lemma}
%------------------------
\begin{proof}
Set $\zeta = \sup_{(a,b) \in \text{supp} \, [\mu]}\|a-b\|$.
Let $\mu\in\Gamma(\alpha,\beta)$ and $(y,y'),(a,b) \in
\text{supp}\, [\mu]$. In the notation of Lemma \ref{L:kernel},
setting $z_1 = y-a$ and $z_2 = y' - b$, we have
\begin{equation}
\|Q_\sigma(y-a) - Q_\sigma(y'-b)\| 
\leq \frac{A_f \sigma}{C_d(\sigma)} \left(\|y-y'\| + \|a-b\|\right)
\leq \frac{2 A_f \sigma}{C_d(\sigma)} \zeta \,,
\end{equation}
where in the last inequality we used $\|y-y'\| \leq \zeta$ and
$\|a-b\| \leq \zeta$. Since
\begin{equation}
\left\|\Sigma_\alpha(a,\sigma) - \Sigma_\beta(b,\sigma) \right\|
\leq \iint \|Q_\sigma(y-a) - Q_\sigma(y'-b)\| \,\mu (dy\times dy') \,,
\end{equation}
the lemma follows.
\end{proof}

%Below, for $\sigma>0$ and $\gamma\geq 0$, if $A = \mathrm{supp}[\alpha]$, then we consider the following metric on $A$:
%$$d_{\alpha;\sigma,\gamma}(a,a'):=\big\|\Sigma_\alpha(a,\sigma)-\Sigma_\alpha(a',\sigma)\big\|+\gamma\,\|a-a'\|$$
%for $a,a'\in A$.

\begin{lemma}[Lemma 2.2 of \cite{dghlp}] \label{L:R-mu}
Let $\alpha, \beta \in \borel_\infty (\real^d)$.
Then, for any coupling $\mu \in \Gamma(\alpha,\beta)$,
$R_\mu = \text{supp}\, [\mu]$ gives a correspondence between
$A=\text{supp} \, [\alpha]$ and $B=\text{supp} \, [\beta]$.
\end{lemma}

\begin{theorem}\label{T:stab-metric}
Let $\alpha, \beta \in \borel_\infty(\real^d)$, $A = \text{supp} \, [\alpha]$, $B=\text{supp}\,[\beta]$, $\sigma > 0$ and $\gamma \geq 0$. Then, for any kernel satisfying the conditions of Lemma \ref{L:kernel}, 
\begin{equation*}
d_{GH} \left((A,d_{\alpha;\sigma,\gamma}),(B,d_{\beta;\sigma,\gamma})
\right) \leq \left(\frac{2 A_f \sigma}{C_d(\sigma)}+\gamma\right) 
\dwassinf (\alpha,\beta)\,,
\end{equation*}
with $A_f>0$ as in Lemma \ref{L:kernel}.
\end{theorem}
%-------------------------
\begin{proof}
Let $\mu\in\Gamma(\alpha,\beta)$ be a coupling that realizes
$\dwassinf (\alpha, \beta)$. By Lemma \ref{L:R-mu},
$R_\mu = \text{supp}\, [\mu]$ is a correspondence between
$A$ and $B$. Thus,
\begin{equation}
\begin{split}
d_{GH} \big(\mathbb{A}_{\alpha;\sigma,\gamma},\,
&\mathbb{B}_{\beta;\sigma,\gamma} \big) \leq 
\frac{1}{2} \, \text{dis} (R_\mu;A,B) \\
&=\frac{1}{2} \sup_{(a,b),(a',b') \in R_\mu} 
\left| d_{\alpha;\sigma,\gamma}(a,a') - 
d_{\beta;\sigma,\gamma} (b, b') \right| \\
&\leq \frac{1}{2}\sup_{(a,b),(a',b')\in R_\mu}
\Big( \left\| \Sigma_\alpha(a,\sigma) - \Sigma_\beta (b,\sigma) \right\|  + \\
&+ \big\|\Sigma_\alpha(a',\sigma) - \Sigma_\beta (b',\sigma) \big\| +
\gamma\|a-b\| + \gamma \|a'-b'\| \Big) \\
&\leq \sup_{(a,b )\in R_\mu} \left\|\Sigma_\alpha(a,\sigma) -
\Sigma_\beta (b,\sigma) \right\| +\gamma \sup_{(a,b)\in R_\mu}\|a-b\| \\
&\leq \left(\frac{2 A_f \sigma}{C_d(\sigma)}+\gamma\right)
\sup_{(a,b)\in R_\mu}\|a-b\|,
\end{split}
\end{equation}
where the last step follows from Lemma \ref{lemma:sigma-ab}.
The conclusion follows since $\dwassinf (\alpha,\beta) =
\sup_{(a,b)\in R_\mu}\|a-b\|.$
\end{proof}

Combining Proposition \ref{P:dendro} and
Theorem \ref{T:stab-metric}, we obtain: 

\begin{corollary}[Stability of Hierarchical Manifold Clustering] 
Let $\alpha, \beta \in \borel_\infty (\real^d)$ be probability
measures with finite support, $A = \text{supp} \, [\alpha]$,
$B=\text{supp}\,[\beta]$, $\sigma > 0$ and $\gamma \geq 0$. Then,

\begin{equation*}
d_{GH} \left(\mathcal{H}(\mathbb{A}_{\alpha; \gamma, \sigma}),
\mathcal{H}(\mathbb{B}_{\alpha; \gamma, \sigma}) \right) \leq
\left(\frac{2 A_f \sigma}{C_d(\sigma)}+\gamma\right)
W_\infty(\alpha,\beta).\end{equation*}
\end{corollary}

\subsection{Comments About Consistency of Hierarchical Manifold Clustering}
A question that our paper leaves open is whether, under a reasonable generative model for the sampling from a collection of intersecting manifolds one may be able to prove that the empirical dendrogram converges in probability to a dendrogram that represents the spatial organization of the underlying manifolds. In the simpler context of clustering $n$ i.i.d. samples $\{x_1,\ldots,x_n\}$ from a compact metric space $(X,d_X)$ endowed with a Borel probability measure $\alpha_X$, it has been established in \cite{memoli10} that single linkage hierarchical clustering converges to a dendrogram whose hierarchical structure depends on the support of $\alpha_X$ in a precise way. 

In the case of flat clustering, i.e. when the goal is to obtain a single partition of the dataset, consistency results for some multi-manifold clustering methods based on local PCA are given in \cite{arias2011clustering,arias2011spectral,arias2013spectral}.

% In the present setting we expect the limiting dendrogram to depend on (1) some quantitative measure of the incidence angle of the pairwise intersection of the manifolds, and (2) on a measure of thinness 

%------------------------

\subsection{Examples and Applications}

Let $X = \{x_1, \ldots, x_n\}$ be a dataset in $\real^d$.
For $\gamma, \sigma > 0$, we apply the single linkage method to
the metric space $\mathbb{X}_{\alpha_n; \gamma,\sigma} =
(X, d_{\alpha_n; \gamma,\sigma})$, where $\alpha_n$ is
the empirical measure associated with $X$ and
$d_{\alpha_n; \gamma,\sigma}$ is the distance defined
in \eqref{E:metric}.  The ultrametric associated with
$\mathcal{H} (\mathbb{X}_{\alpha_n; \gamma,\sigma})$
is abbreviated $u_{\alpha_n; \gamma,\sigma}$.

In this setting, analyzing informative dendrogram cutoff levels
often is an important task, which can be approached in
different ways, depending on the nature of the problem.
For example, a cutoff level $h$ may be based on a
pre-assigned number of clusters, be learned from
training data, or be more exploratory. We give examples
that illustrate all three viewpoints.
%----------------------------
\begin{example}[Lines and planes]

In this experiment we consider the unlabeled point cloud in
Fig.\,\ref{F:clusters}(a) that represents an arrangement of two
parallel planes and two lines that intersect the planes transversely.
Each plane contains $225$ points on a uniform grid and each
line contains $30$ equally spaced points. Cutting the dendrogram
at four clusters, our method finds the four affine
linear subspaces accurately with the Gaussian kernel at
$\sigma = 0.6$. In this case, it is important to choose
$\gamma \ne 0$ since the spatial component
of \eqref{E:affinity} is needed to discriminate the parallel planes.
In Fig.\,\ref{F:clusters}(a), the points are colored according to
cluster membership. In this case, the covariance tensors at
data points on the planes that are away 'from the cluster intersections
have two dominating eigenvalues, whereas for points on the lines
they have only one such eigenvalue. Thus, an
analysis of the spectrum of the covariance tensors at data
points let us infer the dimension of each cluster. 
\end{example}
%---------------------
\begin{example}[A Line Arrangement] \label{S:3lines}
In this example, the point-cloud data represents
three intersecting lines in $\real^2$, as shown in Fig.\,\ref{F:3lines}(a).
Each line segment is sampled at $200$ equally spaced points.
Since the slopes of the lines are different, we expect the covariance
matrices to be able to cluster the points without the aid of additional
spatial information. Thus, we set $\gamma = 0$ in \eqref{E:affinity} and $\sigma = 0.4$. 
The number of clusters was set to six to test the ability of the
algorithm to detect not only the lines, but
also the three intersections. Fig.\,\ref{F:3lines}(b) shows the
single-linkage dendrogram, highlighting each of the six clusters.
The data points are colored according to cluster
membership.
%-------------------
\begin{figure}[h!]
\begin{center}
\begin{tabular}{cccc}
\includegraphics[width=0.22\linewidth]{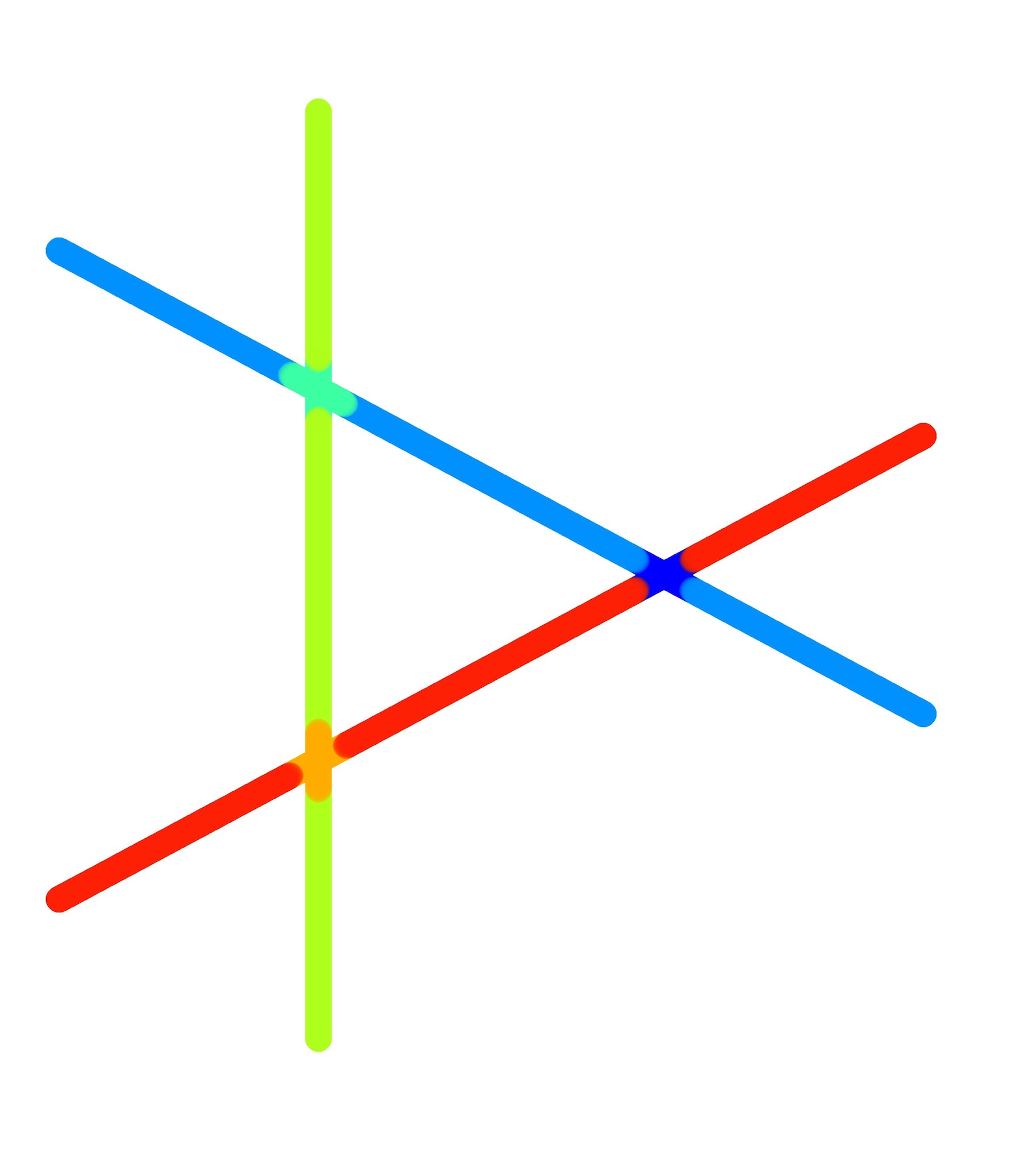} &
\includegraphics[width=0.2\linewidth]{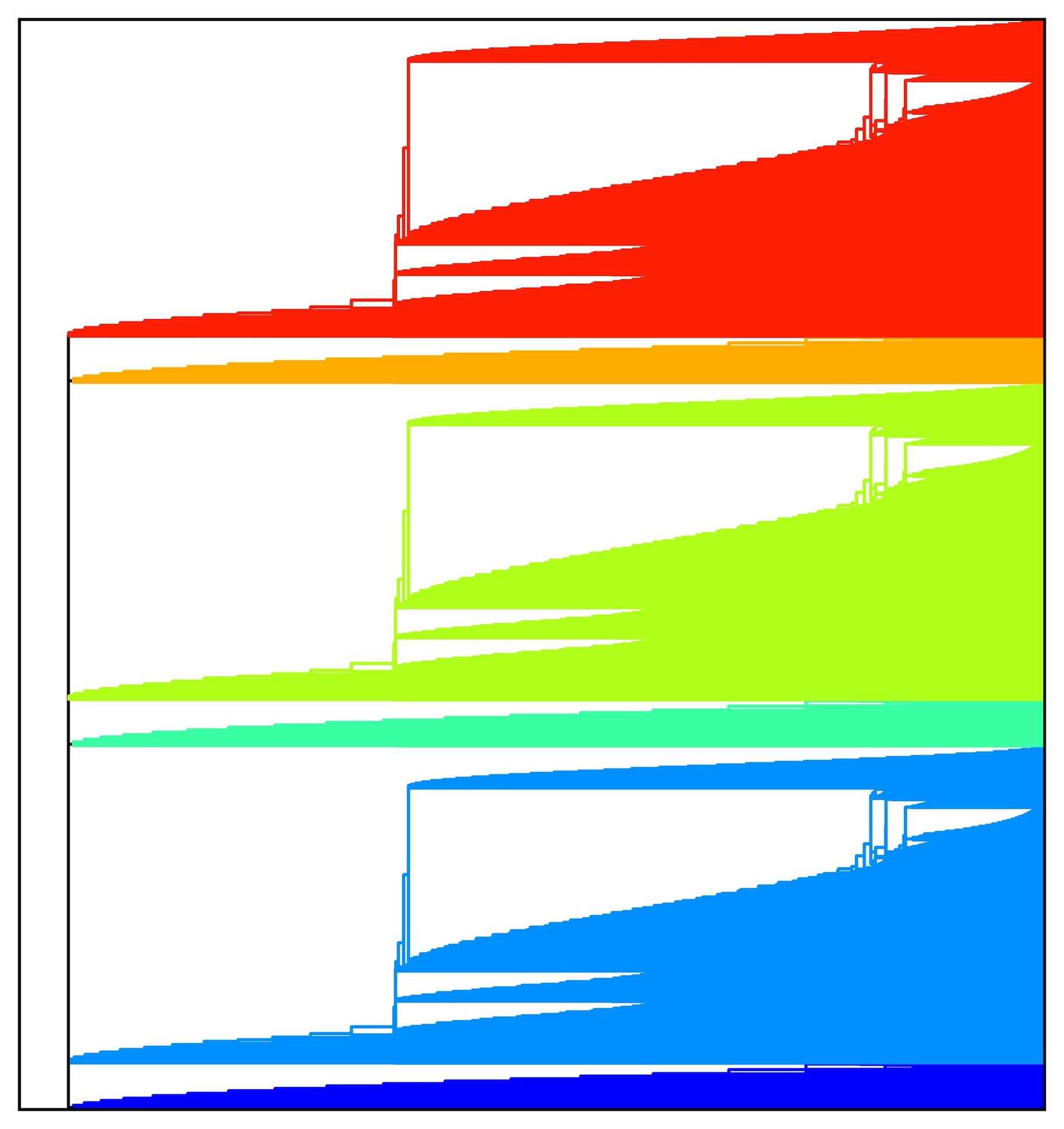} 
\quad & \quad
\includegraphics[width=0.22\linewidth]{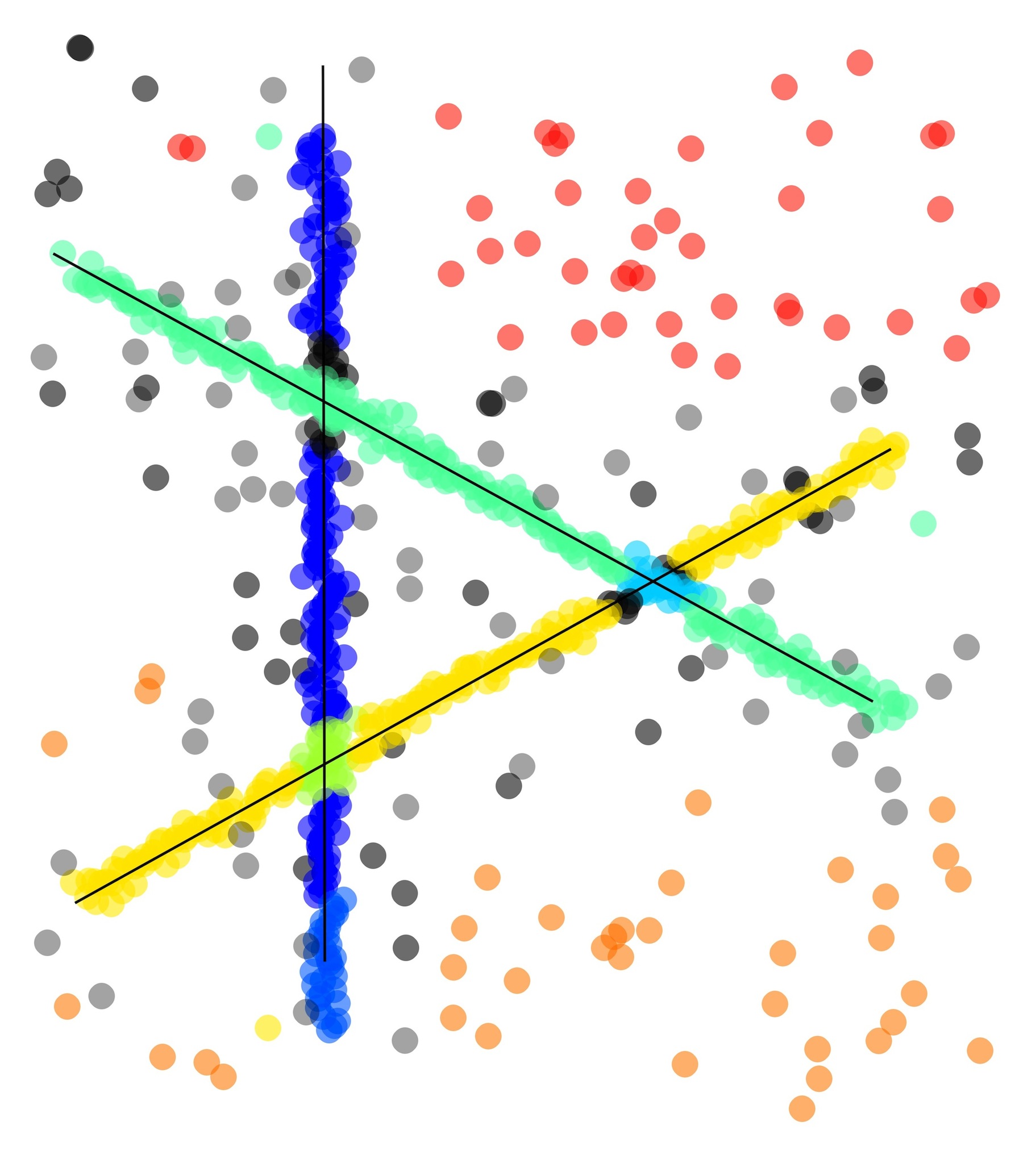} &
\includegraphics[width=0.24\linewidth]{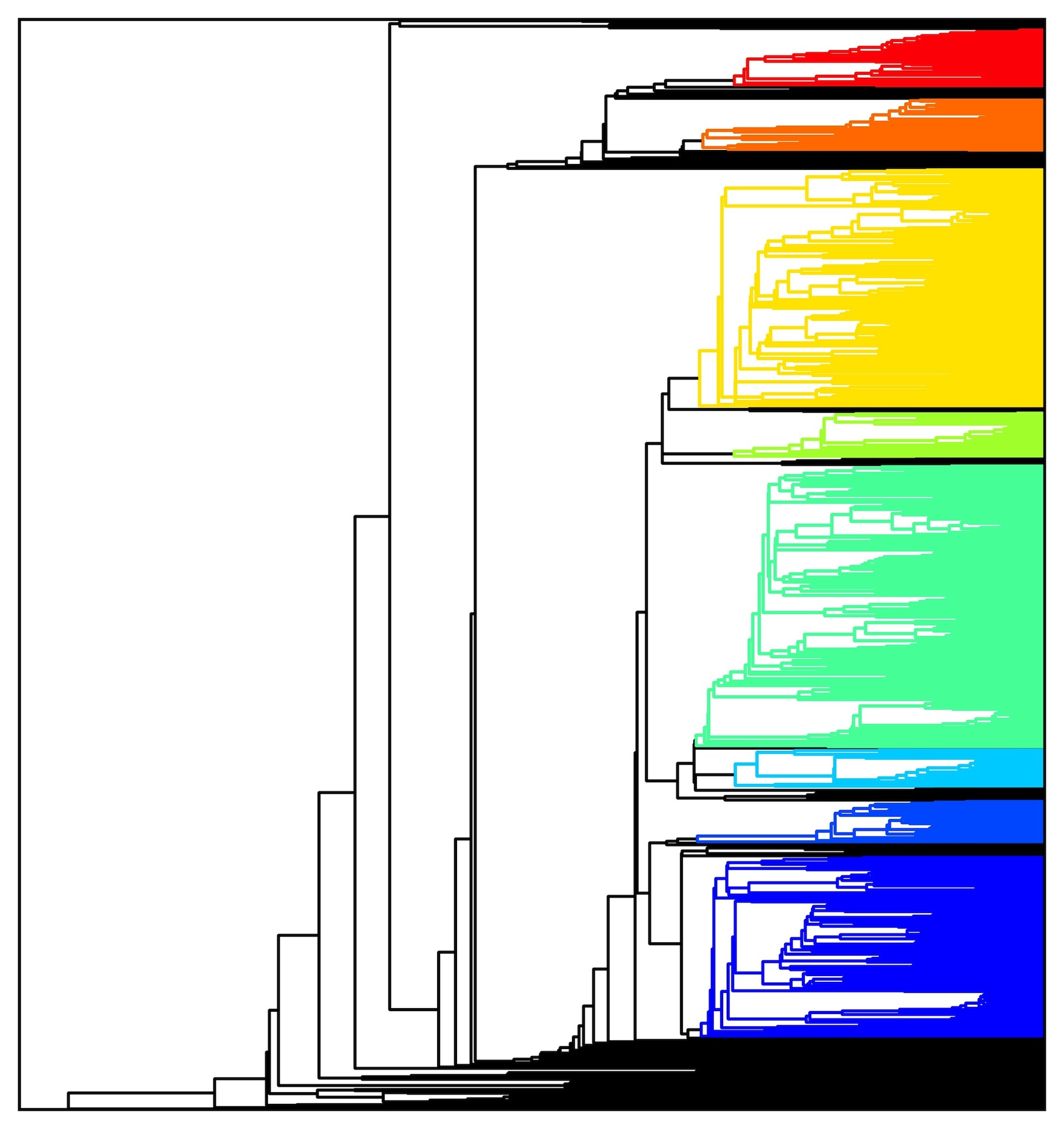} \\
(a) & (b) \quad & \quad (c) & (d)
\end{tabular}
\end{center}
\caption{(a) an arrangement of three lines and (b) clustering
dendrogram; (c) noisy lines with outliers and (d) clustering
dendrogram.}
\label{F:3lines}
\end{figure}
%--------------------
As expected, well delineated clusters are detected away from the
intersection points because the covariance matrices are highly anisotropic
with principal axes that align well with the corresponding line
segments. Although the covariance matrices are not as anisotropic
near the intersection points, there are enough differences in their
behavior near the three intersection loci for the algorithm to be able to
place them into different clusters.
\end{example}
%----------------------
The next two examples are of a more exploratory nature in
that dendrogram cutoff was chosen through experimentation with
the data.
%-----------------------
\begin{example}[Noisy Lines with Outliers] \label{E:noisylines}

This is a noisy version of Example \ref{S:3lines}, as shown
in Fig.\,\ref{F:3lines}(c). As before, each line is represented by
200 points, but we have added Gaussian noise of width $0.015$
to the data, as well as 180 outliers sampled from the uniform distribution
on a rectangle containing the lines. Because of the nature of the data, the
number of clusters was set to $m=80$ so that the three main clusters
did not get merged because of the outliers. The figure also shows
a line fitted to each of the three largest clusters using principal component
analysis. The method was able to sharply recover the three lines, even in
the presence of noise and outliers. The majority of the 80 clusters
are singletons of outliers and these are colored black in the figure.
We remark that the choice of $\sigma = 0.51$ is crucial when dealing with
data contaminated by noise. In this case, it was also important to set
$\gamma \ne 0$ to better cope with noise.
\end{example}

%------------------------
\begin{example}[Floor cracks]
We apply the clustering method to segmentation of two images of concrete
floor cracks. Panel (a) of Fig.\,\ref{F:cracks} shows the original images,
whereas panel (b) shows binary images obtained from an edge detection
algorithm. We cluster the foreground pixels of the binary
images. As in Example \ref{E:noisylines}, it is important to allow a fairly
large number of clusters so that the clusters that detect the main
cracks do not get merged because of the noisy pixels.
Panels (c) and (d) show the outputs (not to scale) of the clustering
algorithm. 
%-------------------------
\begin{figure}[h!]
\begin{center}
\begin{tabular}{cccc}
\begin{tabular}{c}
\includegraphics[width=0.17\linewidth]{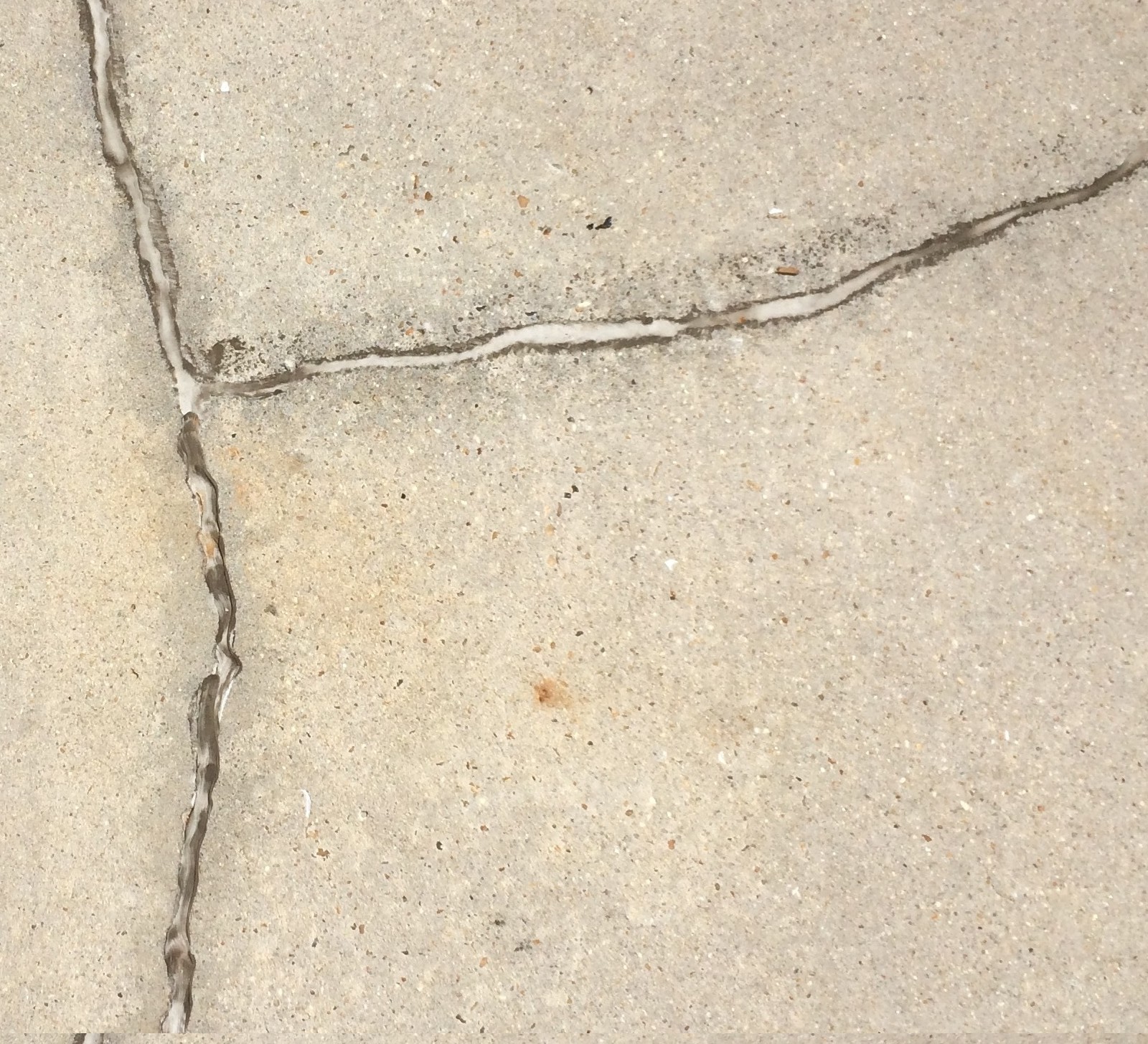} \\
\includegraphics[width=0.17\linewidth]{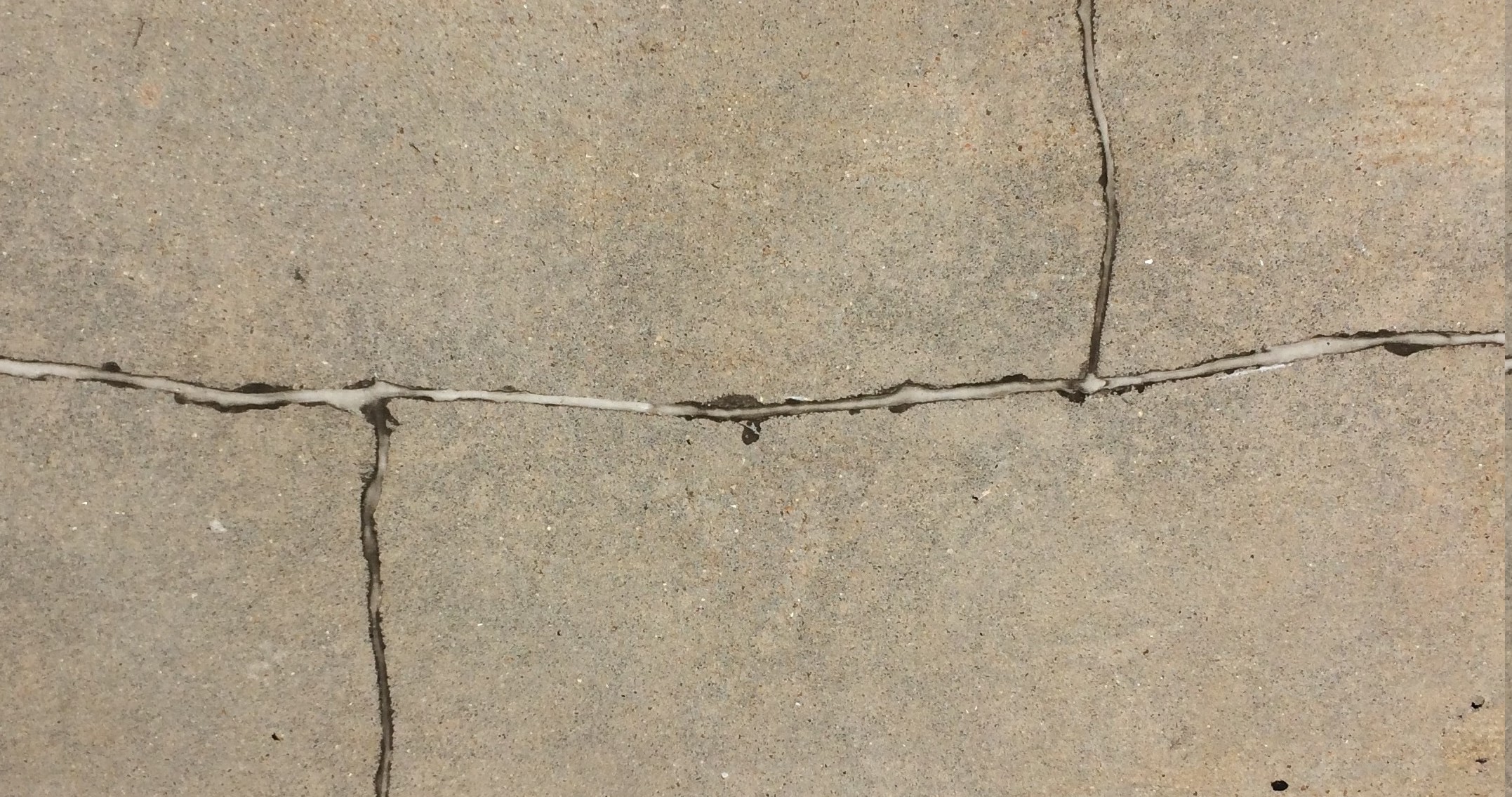}
\end{tabular}
&
\begin{tabular}{c}
\includegraphics[width=0.17\linewidth]{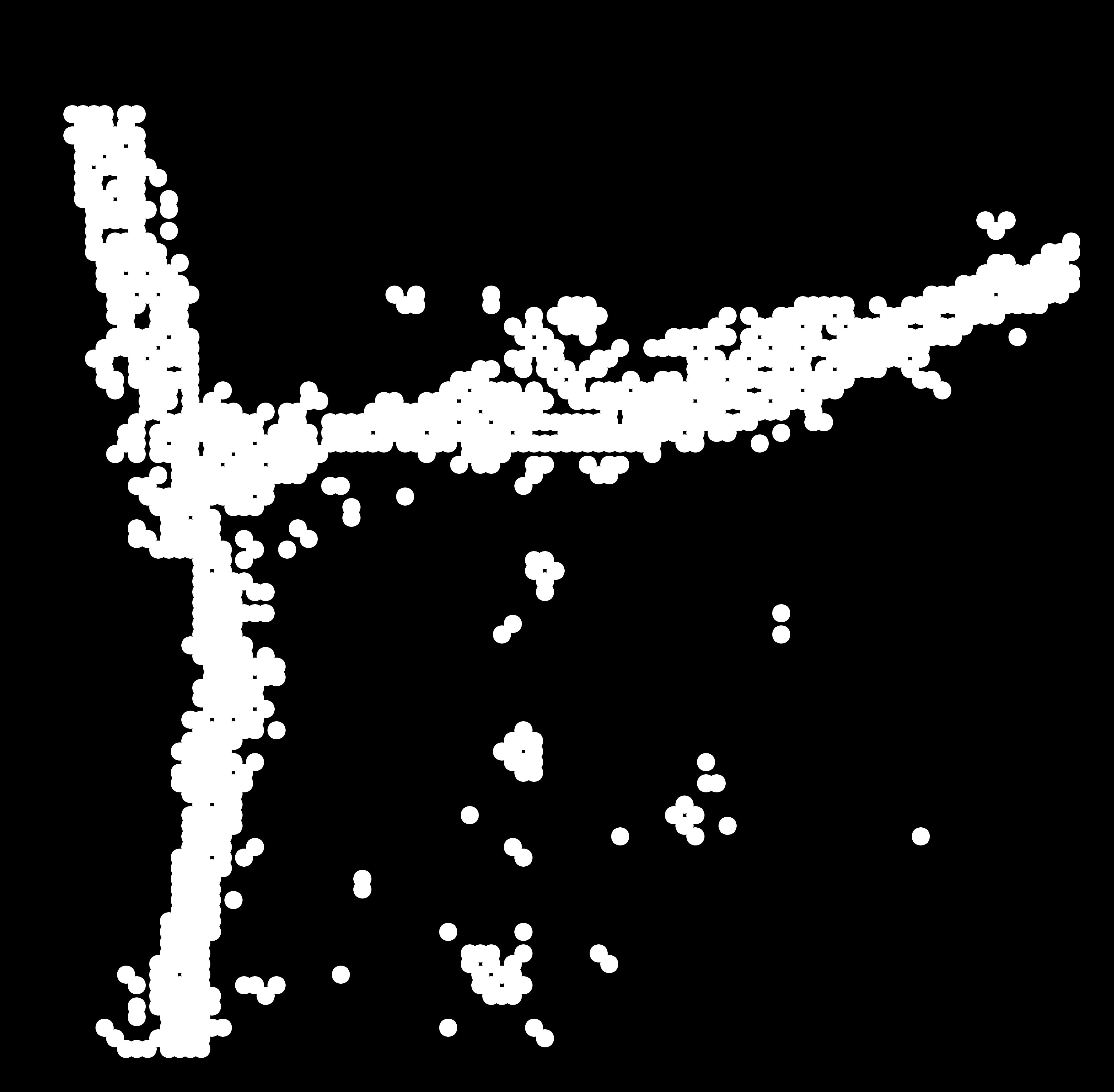} \\
\includegraphics[width=0.17\linewidth]{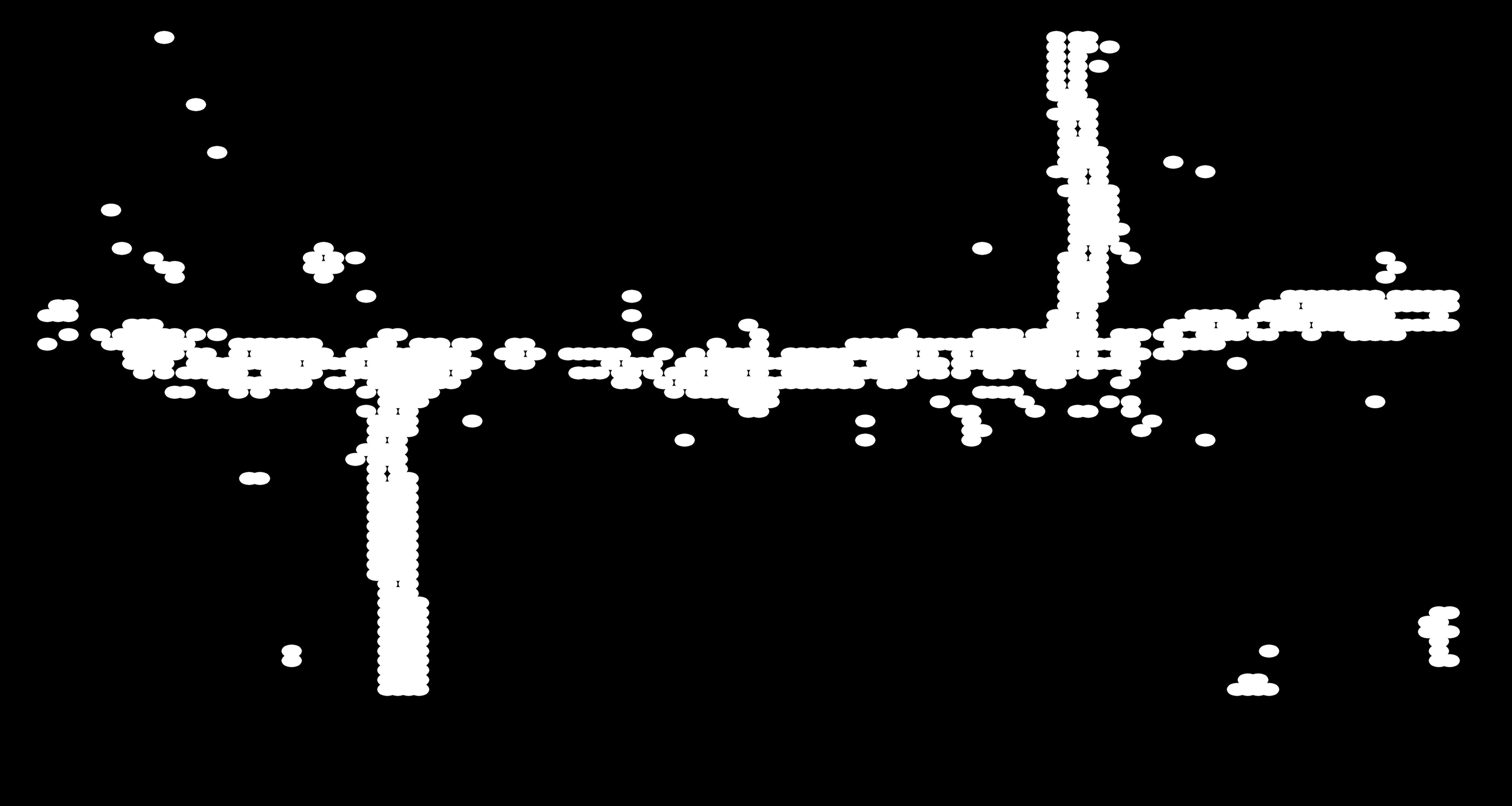}
\end{tabular}
\quad & \quad
\begin{tabular}{c}
\includegraphics[width=0.2\linewidth]{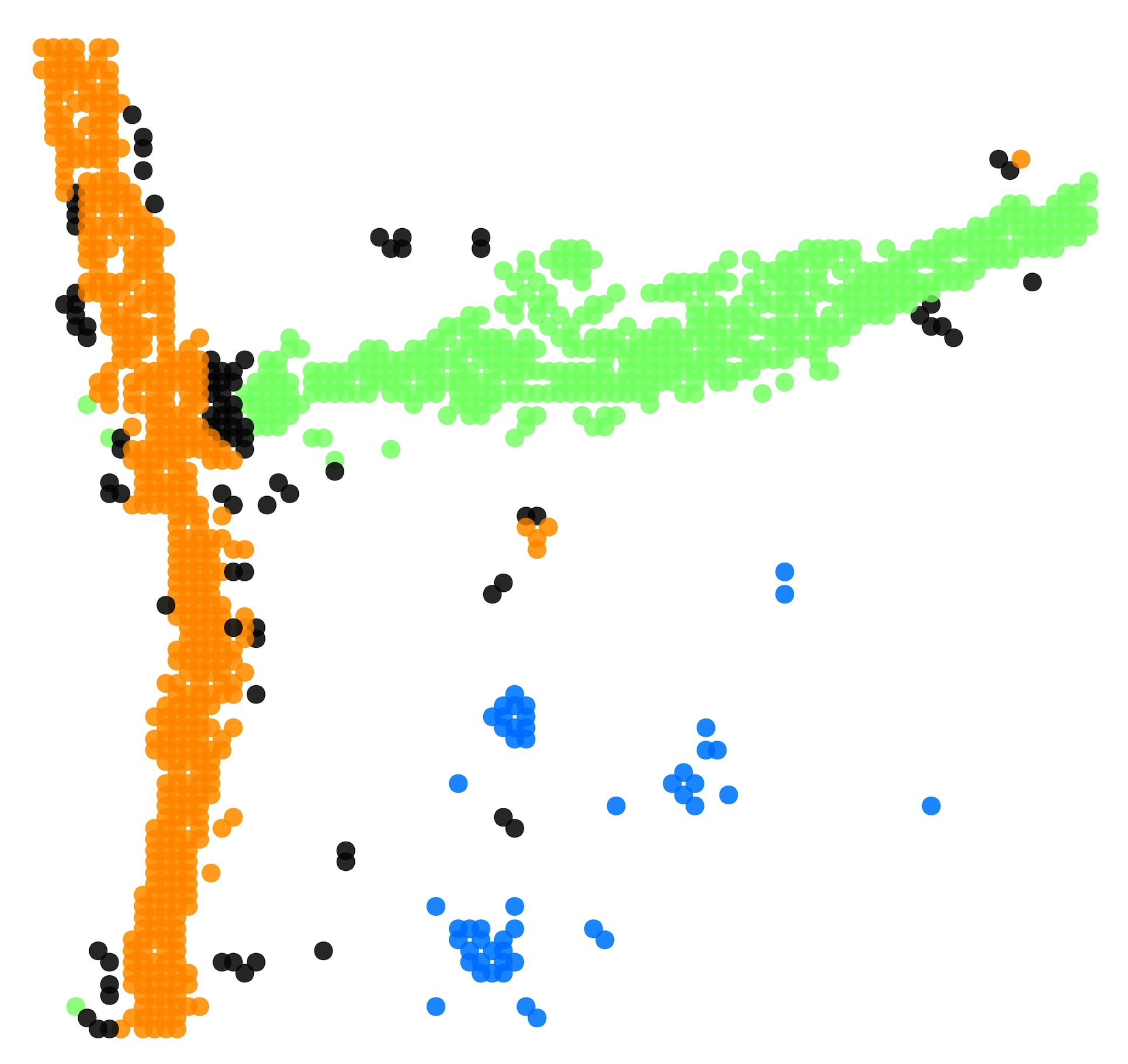}
\end{tabular}
&
\begin{tabular}{c}
\includegraphics[width=0.2\linewidth]{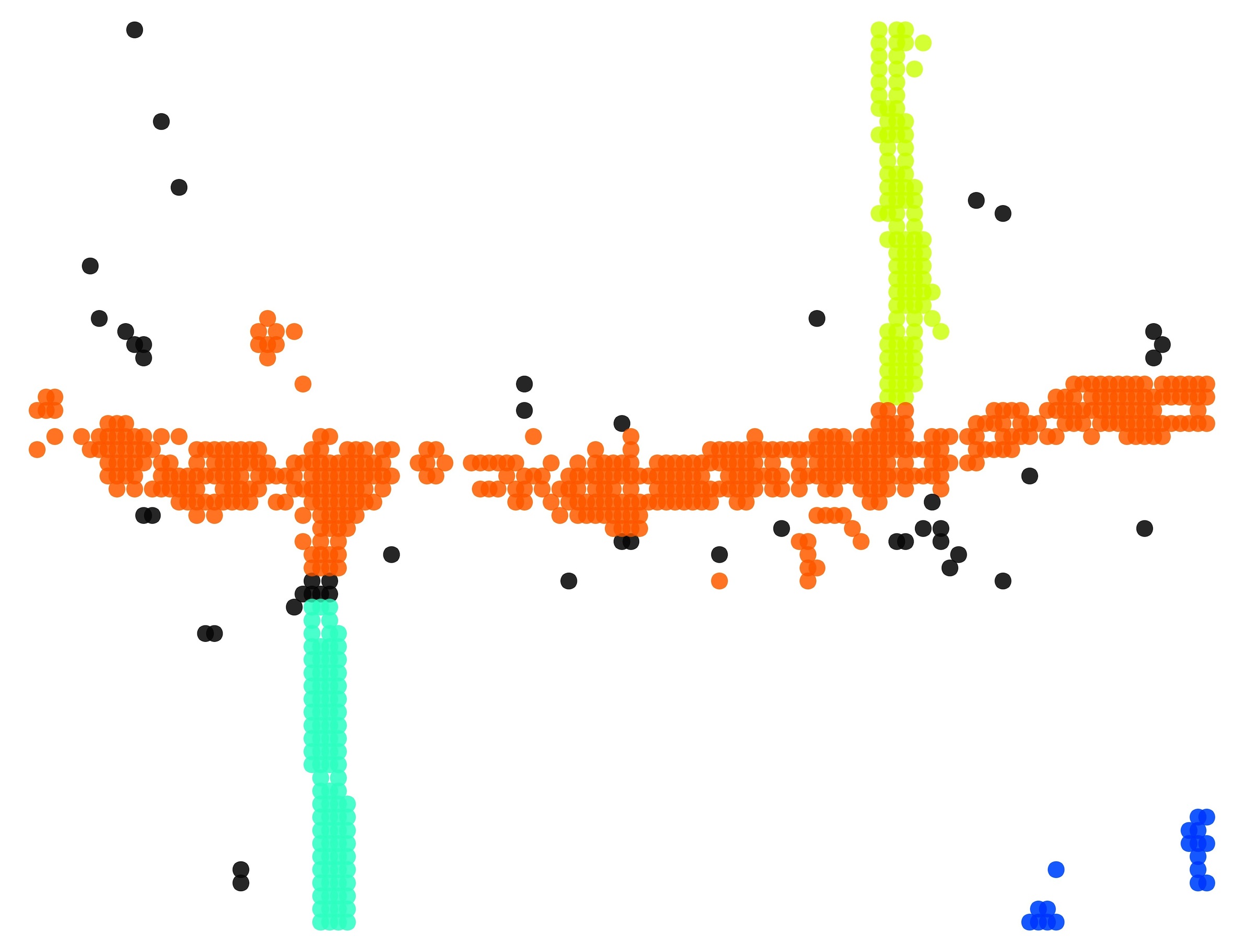}
\end{tabular} \\
(a) & (b) \quad & \quad (c) & (d)
\end{tabular}
\end{center}
\caption{(a) original and (b) processed images of floor cracks;
(c) and (d) show clustering based on CTFs.}
\label{F:cracks}
\end{figure}
\end{example}
%----------------------------------------
To further test the ability of the method to cluster intersecting
manifolds, we experimented with synthetic data comprising
multiple arrangements such as the intersecting lines in
Fig.\,\ref{F:clusters}(b). 
%----------------------------
\begin{example}
We consider three syntehtic datasets of point clouds
representing random arrangements of: (i) three line segments
in $\real^2$; (ii) four curves in $\real^2$ that are either
line segments or arcs of parabolas; and (iii) three
patches of planes in $\real^3$. Each of these datasets
contains a total of 250 point clouds, 50 used for training
the algorithm and 200 test samples. The points in each
point cloud are labeled to allow quantification of the
accuracy of the output of the algorithm.
Fig.\,\ref{F:arrangements} shows a few samples from
each of these datasets.
%------------------------------
\begin{figure}[h!]
\begin{center}
   \includegraphics[width=0.15\linewidth]{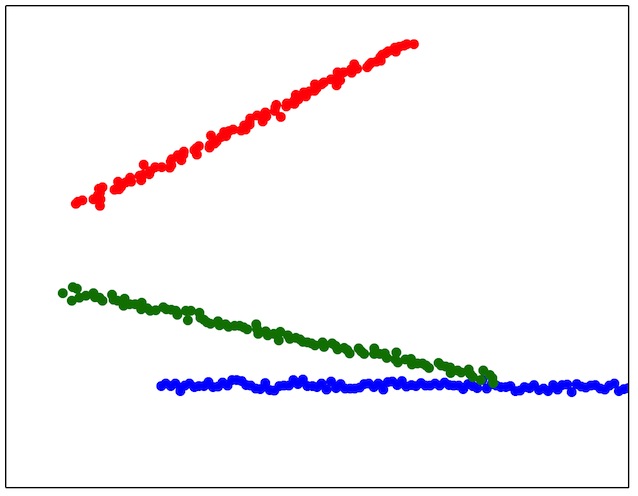}
   \includegraphics[width=0.15\linewidth]{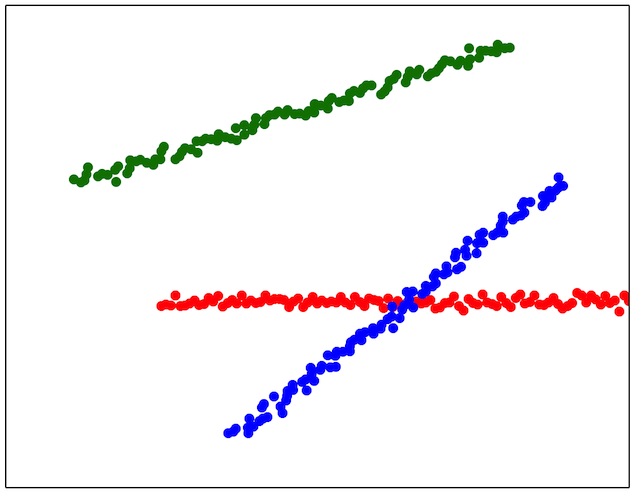}
   \includegraphics[width=0.15\linewidth]{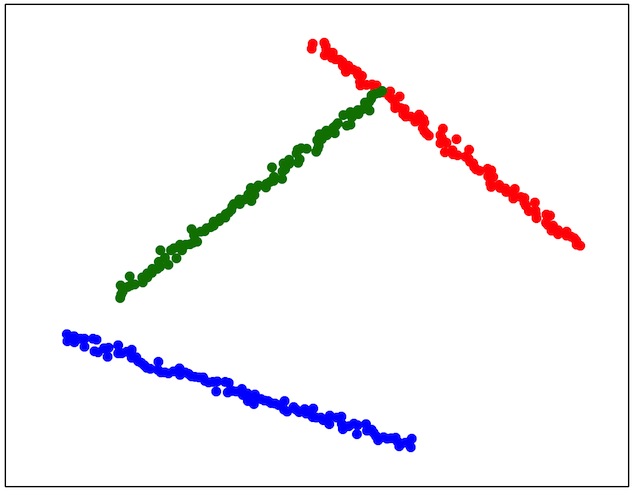}
   \includegraphics[width=0.15\linewidth]{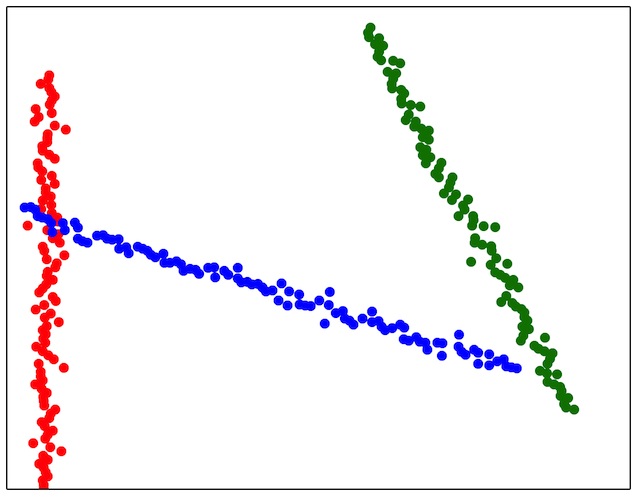}
   \includegraphics[width=0.15\linewidth]{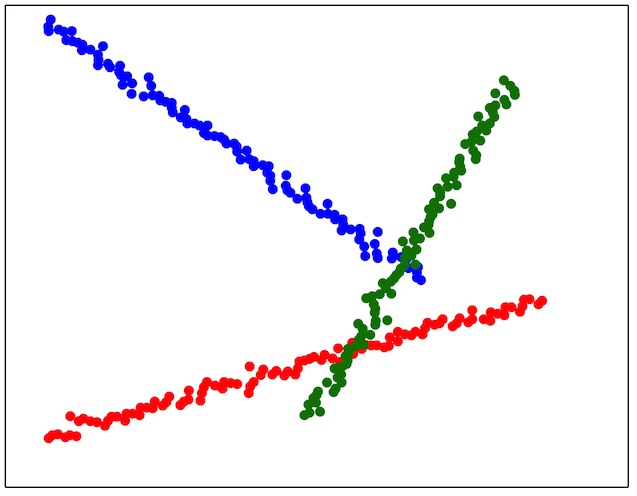}
   \includegraphics[width=0.15\linewidth]{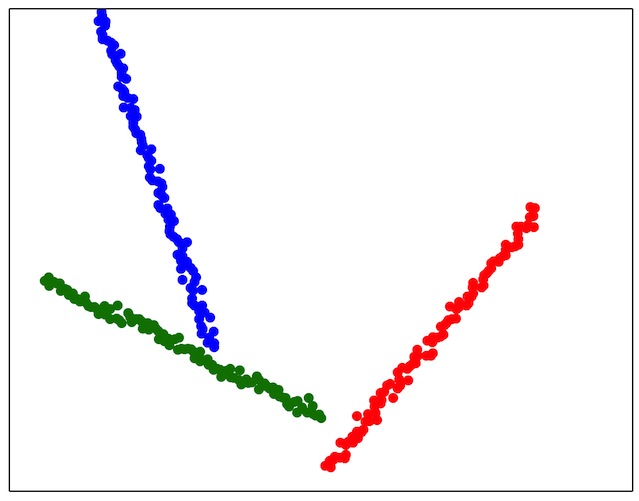}
\vspace{0.1in}\\
   \includegraphics[width=0.15\linewidth]{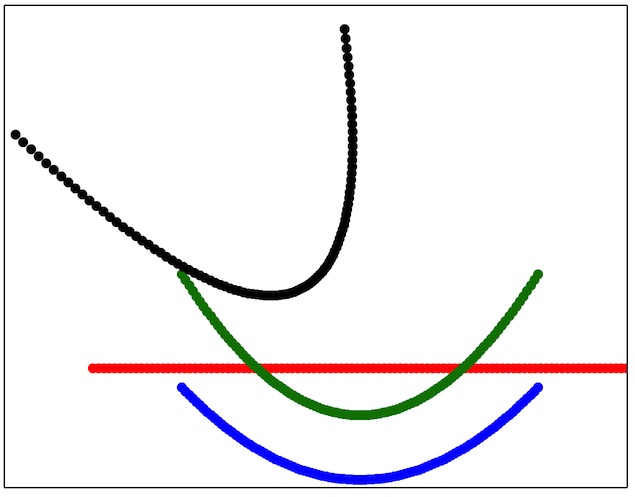}
   \includegraphics[width=0.15\linewidth]{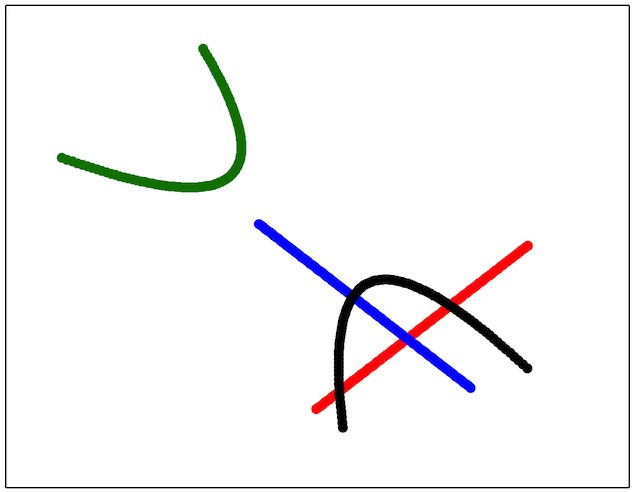}
   \includegraphics[width=0.15\linewidth]{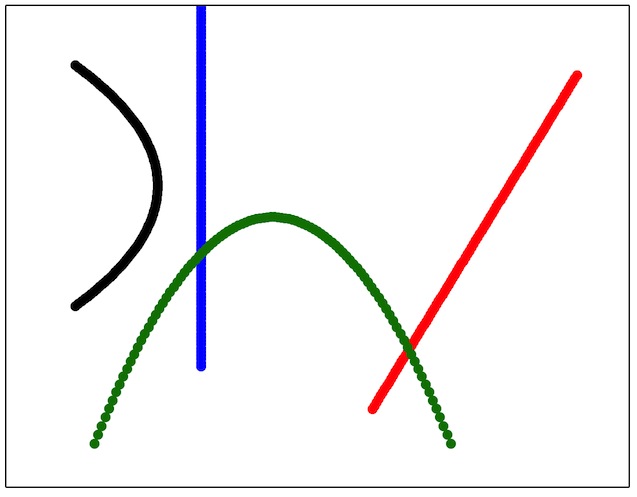}
   \includegraphics[width=0.15\linewidth]{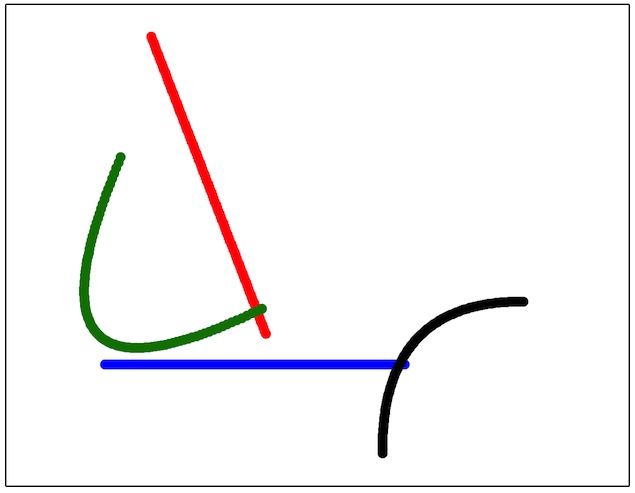}
   \includegraphics[width=0.15\linewidth]{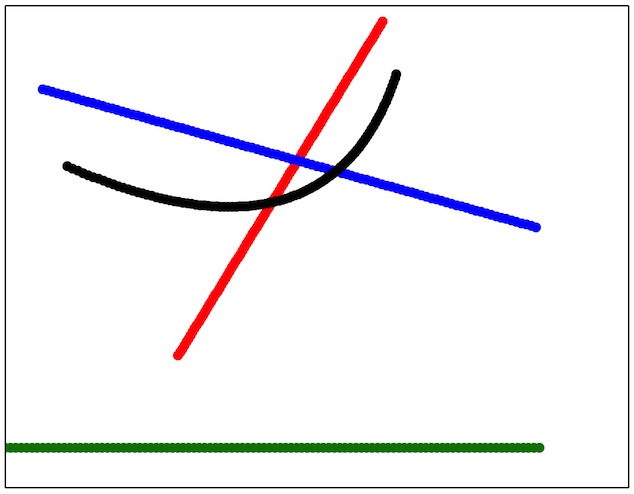}
   \includegraphics[width=0.15\linewidth]{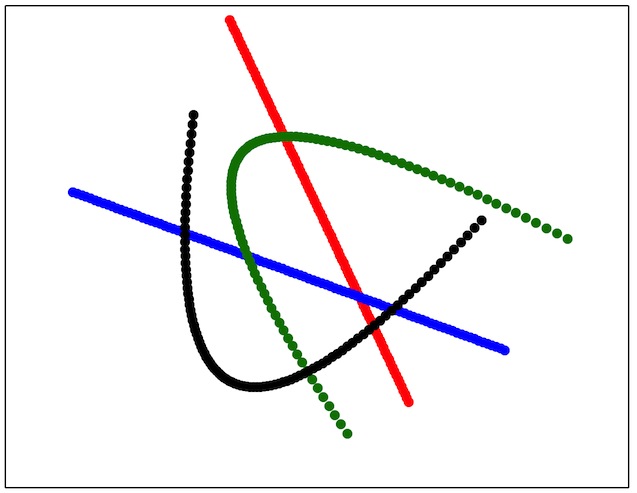}
\vspace{0.1in}\\
   \includegraphics[width=0.18\linewidth]{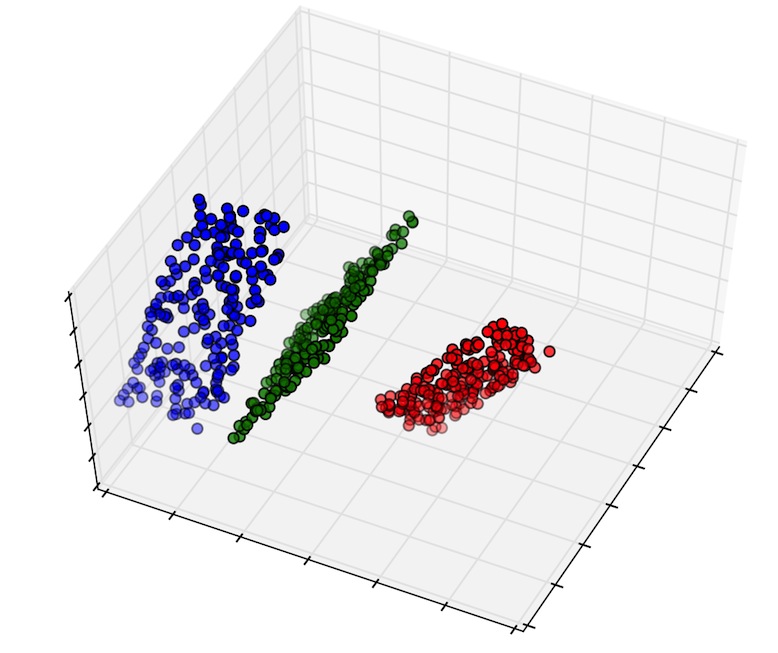}
   \includegraphics[width=0.18\linewidth]{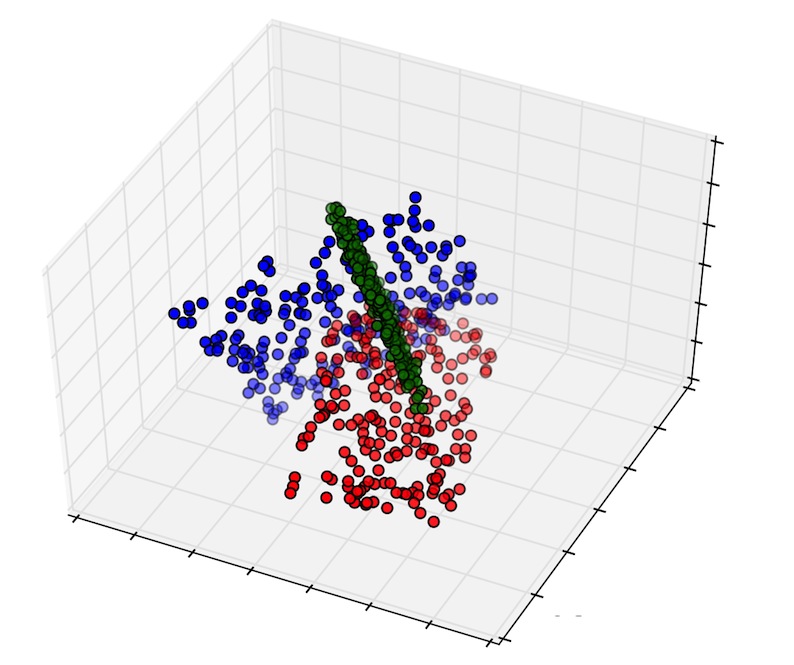}
   \includegraphics[width=0.18\linewidth]{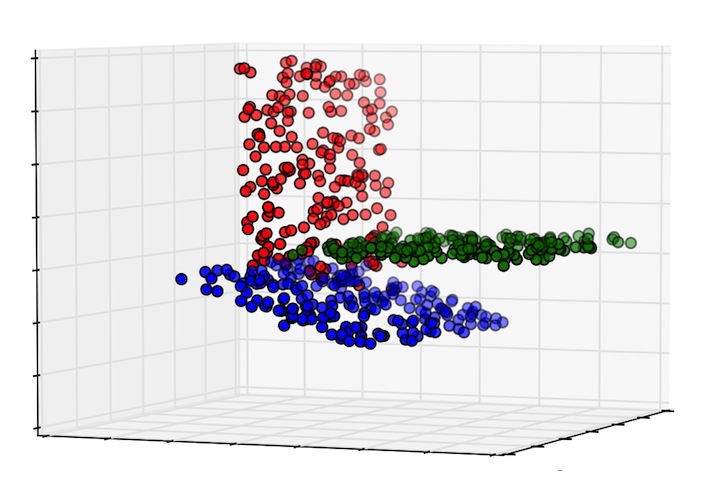}
   \includegraphics[width=0.18\linewidth]{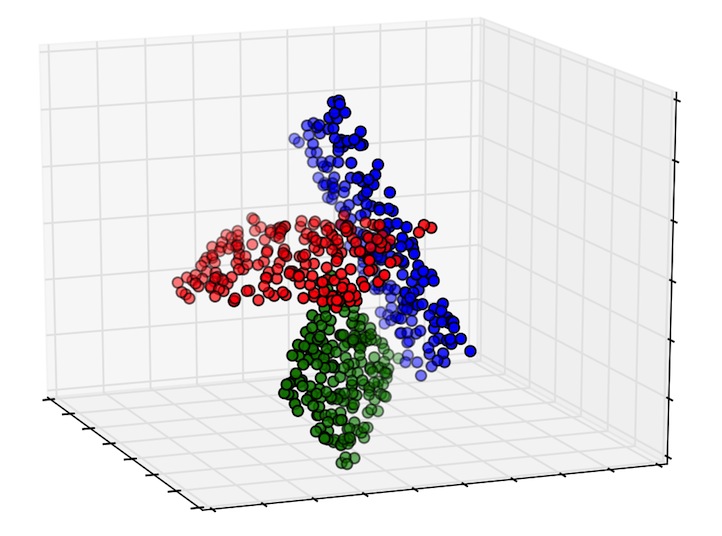}
   \includegraphics[width=0.18\linewidth]{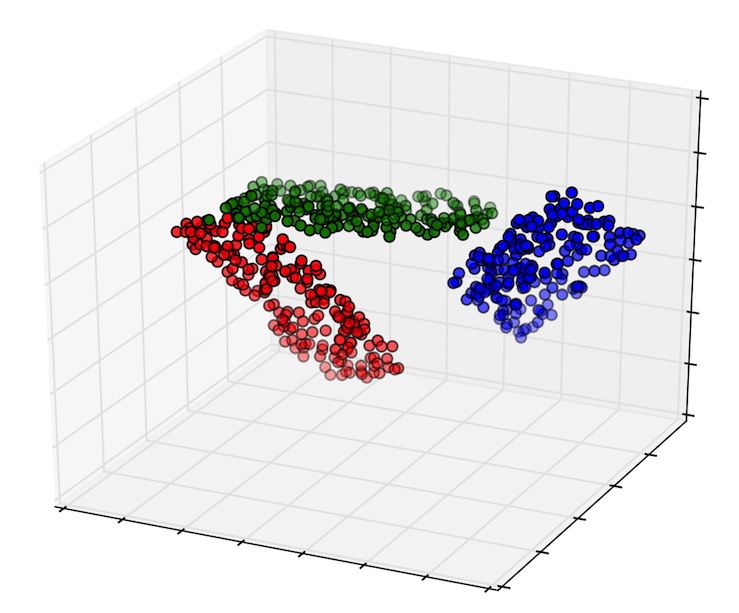}
\end{center}
\caption{Random arrangements of line segments (row 1),
segments of lines and parabolas (row 2), and plane
patches (row 3).}
\label{F:arrangements}
\end{figure}
%-------------------------------
Parameter values that optimize classification performance
are learned from the training samples. Note, however,
that even though the number $k$ of clusters is known, specifying a
height $h$ that yields precisely $k$ clusters may yield
undesirable results. For example, important clusters representing
different components of an arrangement of manifolds may
get merged due to the presence of outliers or the behavior near
the intersections. Thus, it is often
preferable to choose a lower cutoff level before this phenomenon
occurs at the expense of getting a larger number $\ell$ of clusters.
In such situations, we select the largest $k$ clusters and
assign each point in the remaining $(\ell - k)$ clusters
to the closest of the top $k$ clusters. Experiments indicate
that a good baseline for the cutoff level $h$ is the mean
cophenetic distance, which for a point cloud $\{x_1, \ldots,
x_n\} \subset \real^d$ is given by
\begin{equation}
h_0 = \frac{2}{n (n-1)} \sum_{i<j}
u_{\alpha_n; \gamma,\sigma} (x_i, x_j) \,.
\end{equation}
In the learning phase, we
typically search for $h$ in a neighborhood of $h_0$
whose width is determined by the variance of the
distribution of the cophenetic distances.

With the learned parameter values, the algorithm performs
well in all three cases. For each point cloud
we count the number of misclassified points and
calculate the average error (AE) and the mean error (ME) rates
over all test samples obtaining:
\begin{itemize}
\item[(i)] arrangements of lines: 9.59\% (AE) and 4.17\% (ME);
\item[(ii)] lines and parabolas: 9.93\% (AE) and 3.38\% (ME);
\item[(iii)] arrangements of planes: 7.00\% (AE) and 2.42\% (ME).  
\end{itemize}
As expected, a closer inspection of the results reveals that
most of the errors occur at points near the intersections
of the clusters, where the covariance tensors are not as
informative for clustering purposes.

\end{example}

%-----------------------------

\section{Concluding Remarks} \label{S:summary}

We introduced the notion of multiscale covariance tensor
fields associated with Euclidean random variables and developed
a framework for the systematic study of the shape of data
using localized covariance tensors. We investigated foundational
questions such as stability and consistency of multiscale
CTFs, provided illustrations of how CTFs let us uncover
geometry underlying data, and applied the methods to manifold
clustering. We also introduced multiscale Fr\'{e}chet functions,
which are scalar fields derived from CTFs that fully capture
the distributions of random vectors. Multiscale Fr\'{e}chet
functions are particularly well suited for extension
of the methods of this paper to non-Euclidean random variables,
a problem that is receiving ever increasing attention in
data science. In this setting, the goal is to devise methods that
can cope with random variables taking values in spaces such
as Riemannian manifolds and more general metric spaces.
Unless restrictive assumptions are imposed on the sample
space and the distributions, CTFs may be difficult to define
in this nonlinear realm. In contrast, the Fr\'{e}chet function
formulation can be easily extended to metric spaces supporting
a diffusion kernel \cite{coifman06}. In forthcoming work, we will
investigate theoretical and computational aspects of such extensions,
including the accessibility of information residing
in multiscale Fr\'{e}chet functions, a problem that poses
computational challenges even in the case of high-dimensional
Euclidean random variables.

In this paper, we only considered radial basis kernels;
however, many results extend easily to more general kernels.
We emphasized the multiscale formulation largely because
of the questions that motivated this work. Nevertheless, the
majority of the results apply to kernels that are not 
scale dependent.

Covariance tensor fields also suggest ways of formalizing
the notion of shape of Euclidean data and probability measures.
For example, for a distribution $\alpha$ with the property
that  the covariance tensor field $\Sigma_\alpha (\cdot, \sigma)$
associated with a smooth kernel  (such as the Gaussian kernel) is
non-singular for every $x \in \real^d$,
$\Sigma_\alpha^{-1}(\cdot, \sigma)$ defines a metric tensor
with close ties to $\alpha$. This poses the problem
of uncovering relationships between Riemannian metrics
derived from CTFs, such as $\Sigma_\alpha^{-1}(\cdot, \sigma)$,
and the shape of $\alpha$. 

In a different direction, for a fixed point $x\in \real^d$, an interesting problem
is that of capturing the values of $\sigma$ for which $\Sigma_\alpha(x,\sigma)$
exhibits a ``jump'' in behavior.  This study, in the context of images, gives
rise to notions of {\em local scales}. Knowledge of local scales for each
point $x$ leads to criteria for selecting important, salient points in the
spirit of SIFT \cite{lowe2004,mmm13}. The concept of local scales arose
first in the context of images \cite{jones-le} and was later extended to
probabilty distributions \cite{le-scales}. The notions of local scales in
\cite{jones-le,le-scales} were isotropic. Thus, future developments
related to characterizing shape using CTFs are suggested by the
possibility of  constructing notions of local scales on general shapes
\cite{le-scales,mmm13} which ---by exploiting the tensorial nature
of $\Sigma_\alpha(x,\sigma)$--- become sensitive to direction.

%\facundo{As another possiblity for development, for a fixed point $x\in X$, it seems of interest to try and capture the values of $\sigma$ for which $\Sigma_\alpha(x,\sigma)$ exhibits a ``jump'' in behavior.  This study, in the context of images gives rise to notions of ``local scales'' of a probability distributions. Knowledge of local scales for each point $x$ permit then selecting ``important'' points in the spirit of SIFT \cite{lowe2004,mmm13}. The concept of local scales arose first in the context of images in \cite{jones-le} and was later extended to probabilty distributions in \cite{le-scales}. The notions of local scales in \cite{jones-le,le-scales} were isotropic. Thus, future developments related to characterizing shape using CTFs are suggested by the possibility of  constructing notions of local scales on general shapes \cite{le-scales,mmm13} which ---by exploiting the tensorial nature of $\Sigma_\alpha(x,\sigma)$--- become sensitive to direction.}

%----------------------------

\section*{Data Accessibility}
The synthetic data used in the manifold clustering experiments
is available at \url{https://bitbucket.org/diegodiaz-math/ctf-files/}.  

%----------------------------

\section*{Acknowledgements}
This research was supported in part by NSF grants IIS-1422400,
DMS-1418007 and DBI-1262351, and by the Erwin Schr\"{o}dinger Institute
in Vienna. We thank Dejan Slep\v{c}ev for a discussion about
the results of \cite{gts15}.



\bibliographystyle{elsarticle-num} 
\bibliography{ctf}


%\begin{thebibliography}{00}


%\bibitem{}

%\end{thebibliography}

\end{document}